\documentclass{article}

\usepackage{microtype}
\usepackage{graphicx}
\usepackage{subfigure}
\usepackage{booktabs} %

\usepackage[accepted]{icml2021}

\usepackage[utf8]{inputenc} %
\usepackage[T1]{fontenc}    %
\usepackage[colorlinks]{hyperref}       %
\usepackage{url}            %
\usepackage{booktabs}       %
\usepackage{amsfonts}       %
\usepackage{nicefrac}       %
\usepackage{microtype}      %

\usepackage{mathtools}
\usepackage{amssymb}
\usepackage{xcolor}
\usepackage{wrapfig}
\usepackage{amsthm}
\usepackage{graphicx}
\usepackage{amsmath}
\newcommand\numberthis{\addtocounter{equation}{1}\tag{\theequation}}
\usepackage{bm}

\usepackage{blindtext}
\usepackage{multirow}

\usepackage{soul}

\usepackage{graphicx}
\usepackage[space]{grffile}

\usepackage[nameinlink,capitalize]{cleveref}
\crefname{section}{\S}{\S}
\Crefname{section}{\S}{\S}
\crefname{appendix}{App.}{Apps.}
\Crefname{appendix}{App.}{Apps.}
\crefname{theorem}{Thm.}{Thms.}
\Crefname{theorem}{Thm.}{Thms.}
\crefname{proposition}{Prop.}{Props.}
\Crefname{proposition}{Prop.}{Props.}

\newcommand{\myparatightest}[1]{\noindent\textbf{{#1.}}~}

\newtheorem{theorem}{Theorem}
\newtheorem{corollary}{Corollary}
\newtheorem{proposition}{Proposition}
\newtheorem{lemma}{Lemma}

\newcommand{\Pb}{\mathbb{P}}

\newcommand{\Eb}{\mathbb{E}}
\newcommand{\Rb}{\mathbb{R}}
\newcommand{\Zb}{\mathbb{Z}}
\newcommand{\bra}[1]{\left( #1 \right)}
\newcommand{\brb}[1]{\left[ #1 \right]}
\newcommand{\brc}[1]{\left\{ #1 \right\}}

\newcommand{\brn}[1]{\left\lVert #1 \right\rVert}
\newcommand{\brinfinity}[1]{\left\lVert #1 \right\rVert_\infty}

\newcommand{\brf}[1]{\brn{ #1 }_{\text{F}}}
\newcommand{\brsp}[1]{\brn{ #1 }_{\text{sp}}}
\newcommand{\brspw}[1]{\brn{ \brmiyato{#1} }_{\text{sp}}}
\newcommand{\cin}{c_{in}}
\newcommand{\cout}{c_{out}}
\newcommand{\brspwshape}[1]{\brsp{ {#1}^{\cout \times \bra{ \cin k_w  k_h }}}}
\newcommand{\spwshape}[1]{ {#1}^{\cout \times \bra{ \cin k_w  k_h }}}
\newcommand{\brspwshapet}[1]{\brsp{ {#1}^{\cin \times \bra{ \cout k_w  k_h }}}}
\newcommand{\wwshape}[1]{{ {#1}^{\cout \times \bra{ \cin k_w  k_h }}}}
\newcommand{\wwshapet}[1]{{ {#1}^{\cin \times \bra{ \cout k_w  k_h }}}}
\newcommand{\brconvsp}[1]{\brn{ \tilde{#1} }_{\text{sp}}}
\newcommand{\brconv}[1]{ \tilde{#1} }
\newcommand{\brmiyatosp}[1]{\brn{ \hat{#1} }_{\text{sp}}}
\newcommand{\brmiyato}[1]{ \hat{#1} }
\newcommand{\brlip}[1]{\brn{ #1 }_{\text{Lip}}}

\newcommand{\Nc}{\mathcal{N}}

\newcommand{\var}{\text{Var}}
\newcommand{\inner}[2]{\langle #1, #2 \rangle}
\newcommand{\T}{\text{T}}
\newcommand{\diag}[1]{\text{diag}\bra{ #1 }}

\newcommand{\deq}{\overset{d}{=}}
\newcommand{\hyg}[1]{{}_{2}F_{1}\bra{  #1 }}
\newcommand{\poc}[2]{ \bra{ #1 }_{ #2 } }

\newcommand{\red}[1]{\textcolor{red}{#1}}

\newcommand{\nameshort}{BSN}

\newcommand{\namebssn}{Bidirectional Scaled Spectral Normalization}
\newcommand{\namebssnshort}{BSSN}

\newcommand{\namebsn}{Bidirectional Spectral Normalization}
\newcommand{\namebsnshort}{BSN}

\newcommand{\namessn}{Scaled Spectral Normalization}
\newcommand{\namessnshort}{SSN}

\newcommand{\cifar}{CIFAR10}
\newcommand{\stl}{STL10}
\newcommand{\celeba}{CelebA}
\newcommand{\ilsvrc}{ILSVRC2012}
\newcommand{\imagenet}{ImageNet}
\newcommand{\mnist}{MNIST}

\newcommand{\snw}{SN\textsubscript{w}}
\newcommand{\snconv}{SN\textsubscript{Conv}}

\newcommand{\uniform}{\text{Uniform}}

\newcommand{\trainingcurve}[6]{
    \begin{figure}[t]
        \centering
        \begin{minipage}{.45\textwidth}
            \centering
            \includegraphics[width=\linewidth]{{figure/sample_quality/#2/0.0,sn-double_curve/trainingcurve_d_lr-#4,g_lr-#3,n_dis-#5,inception_score_mean_NeurIPS2020}.pdf}
            \caption{Inception score in #1. The results are averaged over 5 random seeds. The hyper-parameters are: $\alpha_g=#3$, $\alpha_d=#4$, $n_{dis}=#5$.}
            \label{fig:#6-g#3-d#4-ndis#5-inception}
        \end{minipage}%
        ~~
        \begin{minipage}{.45\textwidth}
            \centering
            \includegraphics[width=\linewidth]{{figure/sample_quality/#2/0.0,sn-double_curve/trainingcurve_d_lr-#4,g_lr-#3,n_dis-#5,fid_NeurIPS2020}.pdf}
            \caption{FID in #1. The results are averaged over 5 random seeds. The hyper-parameters are: $\alpha_g=#3$, $\alpha_d=#4$, $n_{dis}=#5$.}
            \label{fig:#6-g#3-d#4-ndis#5-fid}
        \end{minipage}%
    \end{figure}
}

\newcommand{\trainingcurveinception}[6]{
	\begin{figure}[ht]
		\centering
		\includegraphics[width=0.7\linewidth]{{figure/sample_quality/#2/0.0,sn-double_curve/trainingcurve_d_lr-#4,g_lr-#3,n_dis-#5,inception_score_mean_NeurIPS2020}.pdf}
		\vspace{-0.3cm}
		\caption{Inception score in #1. The results are averaged over 5 random seeds, with  $\alpha_g=#3$, $\alpha_d=#4$, $n_{dis}=#5$.}
		\label{fig:#6-g#3-d#4-ndis#5-inception}
		\vspace{-0.3cm}
	\end{figure}
}

\newcommand{\trainingcurvefid}[6]{
    \begin{figure}[t]
            \centering
            \includegraphics[width=0.45\linewidth]{{figure/sample_quality/#2/0.0,sn-double_curve/trainingcurve_d_lr-#4,g_lr-#3,n_dis-#5,fid_NeurIPS2020}.pdf}
            \caption{FID in #1. The results are averaged over 5 random seeds. The hyper-parameters are: $\alpha_g=#3$, $\alpha_d=#4$, $n_{dis}=#5$.}
            \label{fig:#6-g#3-d#4-ndis#5-fid}
    \end{figure}
}

\icmltitlerunning{Why Spectral Normalization Stabilizes GANs: Analysis and Improvements}
\begin{document}

\twocolumn[
\icmltitle{Why Spectral Normalization Stabilizes GANs: Analysis and Improvements}

\icmlsetsymbol{equal}{*}

\begin{icmlauthorlist}
\icmlauthor{Zinan Lin}{cmu}
\icmlauthor{Vyas Sekar}{cmu}
\icmlauthor{Giulia Fanti}{cmu}
\end{icmlauthorlist}

\icmlaffiliation{cmu}{Department of Electrical and Computer Engineering, Carnegie Mellon University, Pittsburgh, PA 15213}

\icmlcorrespondingauthor{Zinan Lin}{zinanl@andrew.cmu.edu}
\icmlcorrespondingauthor{Vyas Sekar}{vsekar@andrew.cmu.edu}
\icmlcorrespondingauthor{Giulia Fanti}{gfanti@andrew.cmu.edu}

\icmlkeywords{}

\vskip 0.3in
]

\printAffiliationsAndNotice{}  %

\begin{abstract}
Spectral normalization (SN) \cite{miyato2018spectral} is a widely-used technique for improving the stability and sample quality  of Generative Adversarial Networks (GANs).
However, there is currently limited understanding of why SN is effective. 
In this work, we show that SN controls two important failure modes of GAN training: exploding and vanishing gradients.
Our proofs illustrate a (perhaps unintentional) connection with the successful LeCun initialization \cite{LeCun1998}.
This connection helps to explain why the most popular implementation of  SN for GANs \cite{miyato2018spectral} requires no hyper-parameter tuning, whereas stricter implementations of SN \cite{gouk2018regularisation,farnia2018generalizable} have poor empirical performance out-of-the-box.
Unlike LeCun initialization which only controls gradient vanishing at the beginning of training, SN preserves this property throughout training.
Building on this theoretical understanding, we propose a new spectral normalization technique: \namebssn{} (\namebssnshort{}), which incorporates insights from later improvements to LeCun initialization: Xavier initialization \cite{glorot2010understanding} and Kaiming initialization \cite{he2015delving}. 
Theoretically, we show that \namebssnshort{} gives better gradient control than SN. 
Empirically, we demonstrate  that it outperforms SN in sample quality and training stability on several benchmark datasets.%
\end{abstract}

\section{Introduction}
Generative adversarial networks (GANs) are state-of-the-art deep generative models, perhaps best known for their ability to produce high-resolution, photorealistic images \cite{goodfellow2014generative}.
The objective of GANs is to produce random samples from a target data distribution, given only access to an initial set of training samples.
This is achieved by learning two functions: a generator $G$, which maps random input noise to a generated sample, and a discriminator $D$, which tries to classify input samples as either real (i.e., from the training dataset) or fake (i.e., produced by the generator).
In practice, these functions are implemented by deep neural networks (DNNs), 
and the competing generator and discriminator are trained in an alternating process known as \emph{adversarial training}. 
Theoretically, given enough data and model capacity, GANs  converge to the true underlying data distribution \cite{goodfellow2014generative}. 

Although GANs have been very successful in improving the sample quality of data-driven generative models \cite{karras2017progressive,brock2018large}, their adversarial training also contributes to instability. %
That is, small hyper-parameter changes and even randomness in the optimization can cause training to fail.
Many approaches have been proposed for improving the stability of GANs, including different architectures \cite{radford2015unsupervised,karras2017progressive,brock2018large}, loss functions \cite{arjovsky2017principled,arjovsky2017wasserstein,gulrajani2017improved,wei2018improving}, and various types of regularizations/normalizations \cite{miyato2018spectral,brock2016neural,salimans2016weight}.
One of the most successful proposals to date is  called \emph{spectral normalization} (SN) \cite{miyato2018spectral,gouk2018regularisation,farnia2018generalizable}.
SN forces each layer of the generator to have unit spectral norm during training. 
This has the effect of controlling the Lipschitz constant of the discriminator, which is empirically observed to improve the stability of GAN training \cite{miyato2018spectral}. 

Despite the successful applications of SN \cite{brock2018large,lin2019infogan,zhang2018self,jolicoeur2018relativistic,yu2019free,miyato2018cgans,lee2018stochastic}, 
to date, it remains unclear precisely why this specific normalization is so effective. 

In this paper, we show that SN controls two important failure modes of GAN training: exploding gradients and vanishing gradients. 
These problems are well-known to cause instability in GANs  \cite{arjovsky2017wasserstein,brock2018large},  
leading either to bad local minima or stalled training prior to convergence. 
We make three primary contributions:

\noindent \textit{(1) Analysis of why SN avoids exploding gradients (\cref{sec:explosion}). } 
Poorly-chosen architectures and hyper-parameters, as well as randomness during training, can amplify the effects of large gradients on training instability, ultimately leading to generalization error in the  learned discriminator. 
We theoretically prove that SN imposes an upper bound on gradients during GAN training, mitigating these effects. 

\noindent \textit{(2) Analysis of why SN avoids vanishing gradients (\cref{sec:vanishing}). } 
Small gradients during training are known to cause GANs (and other DNNs) to converge to bad models \cite{LeCun1998,arjovsky2017wasserstein}. 
The well-known LeCun initialization, first proposed over two decades ago, mitigates this effect by carefully choosing the variance of the initial weights \cite{LeCun1998}.
We prove theoretically that SN controls the variance of weights in a way that closely parallels  LeCun initialization.
Whereas LeCun initialization only controls the gradient vanishing problem at the beginning of training, we show empirically that SN preserves this property throughout training. 
Our analysis also explains why a strict implementation of SN \cite{farnia2018generalizable} has poor out-of-the-box performance on GANs and requires additional %
tuning to avoid the vanishing gradient problem, whereas the implementation of  SN in  \cite{miyato2018spectral} requires no %
tuning. 

\noindent \textit{(3) Improving SN with the above theoretical insights (\cref{sec:approach}).} Given this new understanding of the connections between SN and LeCun initialization, we propose  \namebssn{} (\namebssnshort{}), a new normalization technique that combines two key insights (\cref{fig:summary}): 
(a) It introduces a novel bidirectional spectral normalization inspired by \emph{Xavier initialization}, which improved on LeCun initialization by controlling not only the variances of internal outputs, but also the variance of backpropagated gradients  \cite{glorot2010understanding}. 
We theoretically prove that \namebssnshort{} mimics Xavier initialization to give better gradient control than SN. 
(b) \namebssnshort{} introduces a new scaling of weights inspired by \emph{Kaiming initialization}, a newer initialization technique that has better performance in practice \cite{he2015delving}. 
We show that \namebssnshort{} achieve better sample quality and training stability than SN on several benchmark datasets, including \cifar{}, \stl{}, \celeba{}, and \imagenet{}.%

\begin{figure}[t]
	\centering
	\includegraphics[width=1\linewidth]{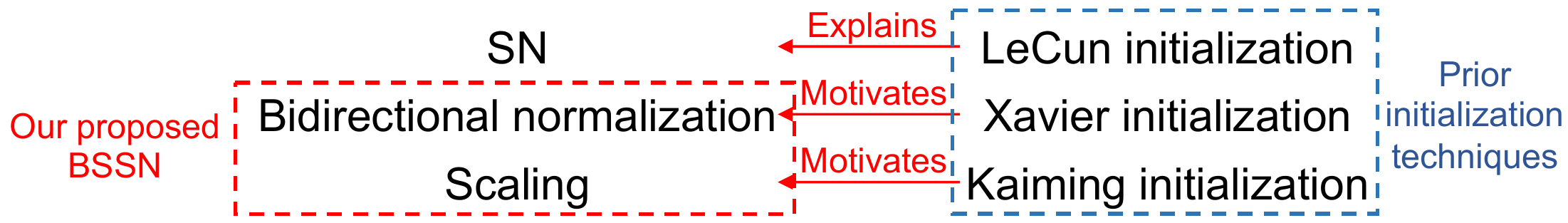}
	\vspace{-0.2cm}
	\caption{The interesting connections we find between spectral normalizations and prior initialization techniques: (1) The insights from LeCun initialization \cite{LeCun1998} can help explain why SN avoids exploding gradients; (2) Motivated from newer initialization techniques \cite{glorot2010understanding,he2015delving}, we proposed \namebssnshort{} to further improve SN.}
	\label{fig:summary}
	\vspace{-0.5cm}
\end{figure}

Note that better gradient control should \emph{not} be the only reason behind the success of SN (more discussions in \cref{sec:discussion}). However, our theoretical results do show a connection between gradient control, initialization techniques, and spectral normalization. Empirical results of the two improvements we propose demonstrate the practical value of this new theoretical understanding.

\section{Background and Preliminaries}
\label{sec:model}

The instability of GANs is believed to be predominantly caused by poor discriminator learning \cite{arjovsky2017principled,salimans2016improved}. 
We therefore focus in this work on the discriminator, and the effects of SN on discriminator learning.
We adopt the same model as \cite{miyato2018spectral}. Consider a discriminator with $L$ internal layers:
\begin{gather}
\scalebox{0.95}{
$ D_\theta (x) = a_L \circ l_{w_L} \circ a_{L-1} \circ l_{w_{L-1}} \circ \ldots \circ a_1 \circ l_{w_1} (x) \label{eq:d} $}
\end{gather}
where $x$ denotes the input to the discriminator and $\theta=\brc{w_1,w_2,...,w_L }$ the weights; $a_i \bra{i=1,...,L-1}$ is the activation function in the $i$-th layer, which is usually element-wise ReLU or leaky ReLU in GANs \cite{goodfellow2014generative}.
 $a_L$ is the activation function for the last layer, which is sigmoid for the vanilla GAN \cite{goodfellow2014generative} and identity for WGAN-GP \cite{gulrajani2017improved}; $l_{w_i}$ is the linear transformation in $i$-th layer, which is usually fully-connected or a convolutional neural network \cite{goodfellow2014generative,radford2015unsupervised}.
Like prior work on the theoretical analysis of (spectral) normalization \cite{miyato2018spectral,farnia2018generalizable,santurkar2018does}, we do not model bias terms.

\myparatightest{Lipschitz regularization and spectral normalization}
Prior work has shown that regularizing the Lipschitz constant of the discriminator $\brlip{D_\theta}$ improves the stability of GANs \cite{arjovsky2017wasserstein, gulrajani2017improved,wei2018improving}. For example, WGAN-GP \cite{gulrajani2017improved} adds a gradient penalty $\bra{\brn{\nabla D_\theta(\tilde{x})} -1}^2$ to the loss function, where $\tilde{x}=\alpha x + (1-\alpha) G(z)$ and $\alpha\sim \uniform\bra{0,1}$ to ensure that the Lipschitz constant of the discriminator is bounded by 1.  

Spectral normalization (SN) takes a different approach. 
For fully connected layers (i.e., $l_{w_i}(x)=w_ix$), it regularizes the weights $w_i$ to ensure that spectral norm $\brsp{w_i}=1$ for all $i\in[1, L]$, where the spectral norm $\brsp{w_i}$ is defined as the largest singular value of $w_i$.
This bounds the Lipschitz constant of the discriminator since $\brlip{D_\theta}\leq \prod_{i=1}^L \brlip{l_{w_i}} \cdot  \prod_{i=1}^{L} \brlip{a_i} \leq \prod_{i=1}^{L} \brsp{w_i} \cdot  \prod_{i=1}^{L} \brlip{a_i} \leq 1$, as $\brlip{l_{w_i}} \leq \brsp{w_i}$ and $\brlip{a_i} \leq 1$ for networks with (leaky) ReLU as activation functions for the internal layers and identity/sigmoid as the activation function for the last layer \cite{miyato2018spectral}. 
Prior work has theoretically  connected the generalization gap of neural networks to the product of the spectral norms of the layers \cite{bartlett2017spectrally,neyshabur2017pac}. 
These  insights led to multiple implementations of spectral normalization %
\cite{farnia2018generalizable,gouk2018regularisation,yoshida2017spectral,miyato2018spectral}, with the implementation of \cite{miyato2018spectral} achieving particular success on GANs.
SN can be viewed as a special case of more general techniques for enhancing stability of neural network training by controlling the spectrum of the network's input-output Jacobian \cite{pennington2017resurrecting}, e.g., through techniques like Jacobian clamping \cite{odena2018generator}, which constrains the values of the maximum and minimum singular values in the generator during training.

In practice, spectral normalization \cite{farnia2018generalizable, miyato2018spectral} is implemented by dividing the weight matrix $w_i$ by its spectral norm: $\frac{w_i}{u_i^Tw_iv_i}$, where $u_i$ and $v_i$ are the left/right singular vectors of $w_i$ corresponding to its largest singular value.
As observed by Gouk et al. \cite{gouk2018regularisation}, there are two approaches in the SN literature for instantiating the matrix $w_i$ for convolutional neural networks (CNNs).
In a CNN, since convolution is a linear operation, convolutional layers can equivalently be written as a multiplication by an expanded weight matrix $\tilde w_i$ that is derived from the raw weights $w_i$. %
Hence in principle, spectral normalization should normalize each convolutional layer by  $\brsp{\tilde w_i}$ \cite{gouk2018regularisation,farnia2018generalizable}. 
We call this canonical normalization  \snconv{} as it controls the spectral norm of the convolution layer.

However, the spectral normalization 
that is known to outperform other regularization techniques and improves training stability for GANs \cite{miyato2018spectral}, 
which we call \snw{}, does not implement SN in a strict sense.
Instead, it
uses $\brspwshape{w_i}$; that is, it first reshapes the convolution kernel $w_i\in \Rb^{c_{out}  c_{in} k_w  k_h}$ into a matrix $\brmiyato{w_i}$ of shape $c_{out}\times \bra{c_{in} k_w k_h}$, and then normalizes with the spectral norm $\brmiyatosp{w_i}$, where $c_{in}$ is the number of input channels, $c_{out}$ is the number of output channels, $k_w$ is the kernel width, and $k_h$ is the kernel height. 
Miyato et al. showed that their implementation implicitly penalizes $w_i$ from being too sensitive in one specific direction \cite{miyato2018spectral}. 
However, this does not explain why \snw{} is more stable than other Lipschitz regularization techniques,
and as observed in \cite{gouk2018regularisation}, it is unclear how \snw{} relates to \snconv{}.
Despite this, \snw{} has empirically been immensely successful in stabilizing the training of GANs \cite{brock2018large,lin2019infogan,zhang2018self,jolicoeur2018relativistic,yu2019free,miyato2018cgans,lee2018stochastic}. 
Even more puzzling, we show in \cref{sec:vanishing} that the canonical approach \snconv{}  has comparatively poor out-of-the-box performance when training GANs. 

Hence, two questions arise: (1) Why is SN so successful at stabilizing the training of GANs? (2) Why is \snw{} proposed by \cite{miyato2018spectral} so much more effective than the canonical \snconv{}?

In this work, we show that both questions are related to two well-known phenomena: vanishing and exploding  gradients.
These terms describe a problem in which gradients either grow or shrink rapidly during training \cite{bengio1994learning,pascanu2012understanding,pascanu2013difficulty,bernstein2020distance}, and they are known to be closely related to the instability of GANs \cite{arjovsky2017principled,brock2018large}. 
We provide an example to illustrate how vanishing or exploding gradients cause training instability in GANs in \cref{app:exp-vani}.

\section{Exploding Gradients}
\label{sec:explosion}

In this section, we show that spectral normalization prevents gradient explosion by bounding the gradients of the discriminator.
Moreover, we show that the common choice to normalize all layers equally achieves the tightest upper bound for a restricted class of discriminators.
We use $\theta\in \Rb^d$ to denote a vector containing all elements in $\brc{w_1,...,w_L}$.
In the following analysis, we assume linear transformations are fully-connected layers $l_{w_i}(x) = w_ix$ as in \cite{miyato2018spectral}, though the same analysis can be applied  to convolutional layers.
Following prior work on the theoretical analysis of (spectral) normalization \cite{miyato2018spectral,farnia2018generalizable,santurkar2018does}, we assume no bias in the network (i.e., \cref{eq:d}) for simplicity.

To highlight the effects of the spectral norm of each layer on the gradient and simplify the exposition, 
we will compute gradients with respect to $w_i'=\frac{w_i}{u_i^Tw_iv_i}$ in the following discussion. 
In reality, gradients are computed with respect to $w_i$; we defer this discussion to \cref{app:gradient}, where we show the relevant extension.

\myparatightest{How SN controls exploding gradients}
The following proposition shows that under this simplifying assumption, spectral normalization controls the magnitudes of the gradients of the discriminator with respect to $\theta$.
Notice that simply controlling the Lipschitz constant of the discriminator (e.g., as in WGAN \cite{arjovsky2017principled}) does not imply this property; it instead ensures small (sub)gradients with respect to the input, $x$.

\begin{proposition}[Upper bound of gradient's Frobenius norm for spectral normalization]\label{thm:gradient-upperbound-sn}
	If $\brsp{w_i} \leq 1$ for all $i\in [1, L]$, then we have
	$\brf{\nabla_{w_t} D_{\theta}(x)}  \leq \brn{x} \prod_{i=1}^{L}\brlip{a_i},$
	and the norm of the overall gradient can be bounded by
	$ \brf{\nabla_{\theta} D_{\theta}(x)} \leq \sqrt{L}\brn{x} \prod_{i=1}^{L}\brlip{a_i} \;.$
\end{proposition}

\emph{(Proof in \cref{app:proof-gradient-upperbound}).}
	Note that under the assumption that internal activation functions are ReLU or leaky ReLU, if the activation function for the last layer is identity (e.g., for WGAN-GP \cite{gulrajani2017improved}), the above bounds can be simplified to
	$\brf{\nabla_{w_t} D_{\theta}(x)}  \leq \brn{x} \text{    and    }  \brn{\nabla_{\theta} D_{\theta}(x)} \leq \sqrt{L}\brn{x}$,
	and if the activation for the last layer is sigmoid (e.g., for vanilla GAN 
\begin{figure}[t]
	\centering
	\includegraphics[width=0.7\linewidth]{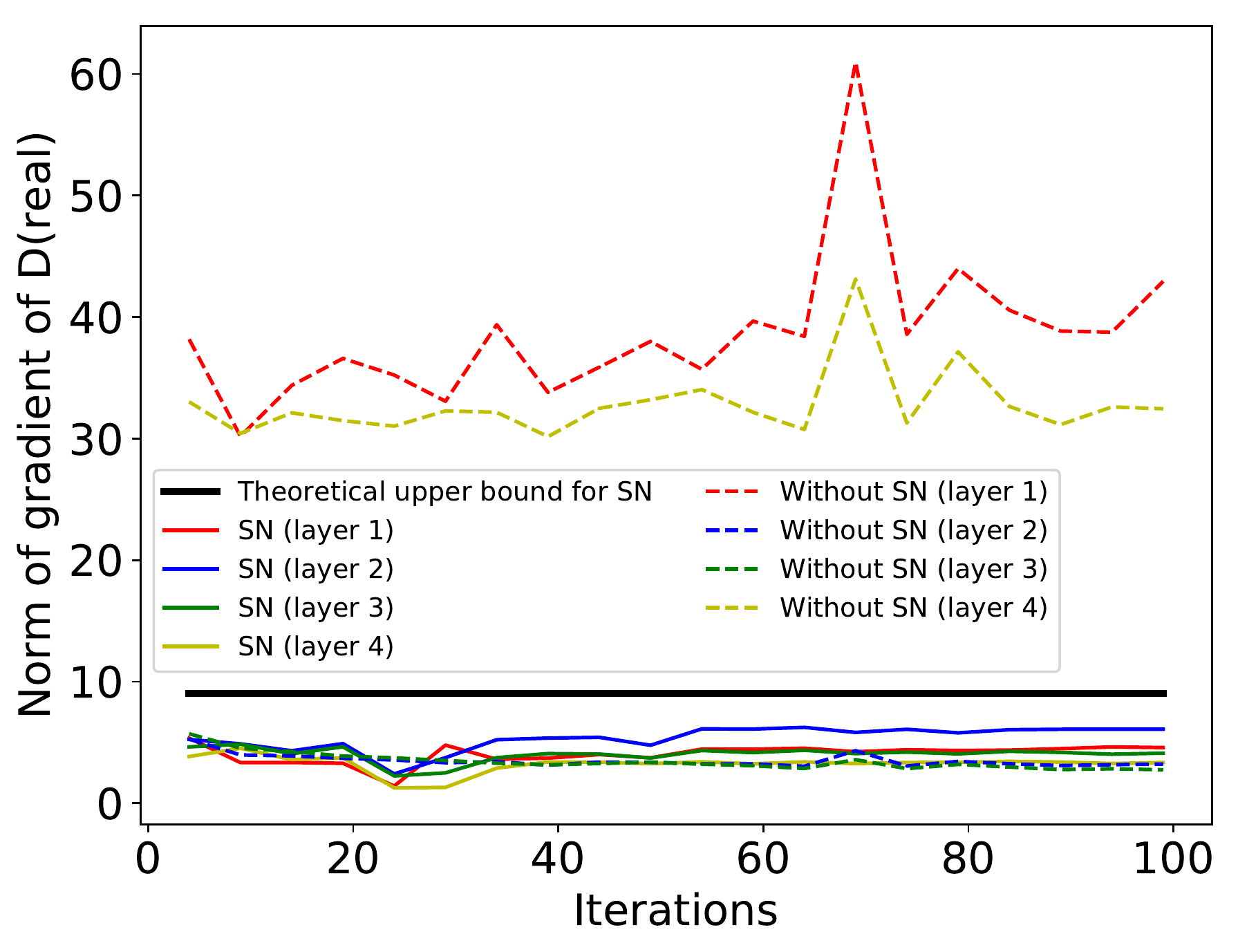}
	\vspace{-0.5cm}
	\caption{Gradient norms of each discriminator layer in \mnist{}. }
	\label{fig:gnorm-mnist}
	\vspace{-0.3cm}
\end{figure}
	\cite{goodfellow2014generative}), the above bounds become
	$\brf{\nabla_{w_t} D_{\theta}(x)}  \leq 0.25 \brn{x}$ and $\brn{\nabla_{\theta} D_{\theta}(x)} \leq 0.25 \sqrt{L}\brn{x}$.
	A comparable bound can also be found to limit the norm of the  Hessian, which we defer to \cref{app:hessian}.%

The bound in \cref{thm:gradient-upperbound-sn} has a significant effect in practice. 
\cref{fig:gnorm-mnist} shows the norm of the gradient for each layer of a GAN trained on MNIST with and without spectral normalization. 
Without spectral normalization, some layers have extremely large gradients throughout training, which makes the overall gradient large.
With spectral normalization, the gradients of all layers are upper bounded as shown in \cref{thm:gradient-upperbound-sn}. 
We see similar results in other datasets and network architectures (\cref{app:gradient-norm}).

\myparatightest{Optimal spectral norm allocation}
Common implementations of SN advocate setting the spectral norm of \emph{each layer} to the same value \cite{miyato2018spectral,farnia2018generalizable}. 
However, the following proposition states that we can set the spectral norms of different layers to different constants, without changing the network's behavior on the input samples, as long as the \emph{product} of the spectral norm bounds is the same. 

\begin{proposition}
	\label{thm:scaling}
	For any discriminator 
	$D_\theta=a_L \circ l_{w_L} \circ a_{L-1} \circ l_{w_{L-1}} \circ \ldots \circ a_1 \circ l_{w_1}$ 
	and 
	$D_\theta' = a_L \circ l_{c_L\cdot w_L} \circ a_{L-1} \circ l_{c_{L-1}\cdot w_{L-1}} \circ \ldots \circ a_1 \circ l_{c_1\cdot w_1}$
	where the internal activation functions $\brc{a_i}_{i=1}^{L-1}$ are ReLU or leaky ReLU, and positive constant scalars $c_1,...,c_L$ satisfy that $\prod_{i=1}^{L}c_i=1$, we have
	\begin{eqnarray}
	 &D_\theta(x)=D_\theta'(x) \quad \forall x \text{ and } \nonumber\\
	 &\frac{\partial^n D_\theta(x)}{\partial x^n} = \frac{\partial^n D_\theta'(x)}{\partial x^n }  
	  \forall x, \forall n\in \Zb^+\;.\nonumber
	 \end{eqnarray}
\end{proposition}
\textit{(Proof in \cref{app:proof-scaling}).} 
Given this observation, it is natural to ask if there is any benefit to setting the spectral norms of each layer equal.
It turns out that the answer is yes, under some assumptions that appear to approximately hold in practice. 
Let 
\begin{gather}
\scalebox{0.77}{$
\begin{aligned}
\mathcal D \triangleq \bigg \{&D_\theta = a_L \circ l_{w_L} \circ  \ldots \circ a_1 \circ l_{w_1} ~:~\frac{\brf{\nabla_{w_i}D_\theta(x)}}{\brf{\nabla_{w_j}D_\theta(x)}}  =  \frac{\brsp{w_j}}{\brsp{w_i}}, \nonumber\\
&a_i\in \{\text{ReLU, leaky ReLU}\} \;\forall i,j\in [1,L]\bigg \}.
\end{aligned}\numberthis$}\label{eq:setD}
\end{gather}
This intuitively describes the set of all discriminators for which scaling up the weight of one layer proportionally increases the gradient norm of all other layers; the definition of this set is motivated by our upper bound on the gradient norm (\cref{app:proof-gradient-upperbound}).
The following theorem shows that when optimizing over set $\mathcal D$, choosing every layer to have the same spectral norm gives the smallest possible gradient norm, for a given set of parameters.

\begin{theorem}
	\label{thm:scaling-upperbound-practical}
	Consider a given set of discriminator parameters $\theta=\brc{w_1,...,w_{L}}$. 
	For a vector $c=\brc{c_1,\ldots,c_L}$, we denote
	 $\theta_c \triangleq \brc{c_tw_t}_{t=1}^L$.
	 Let $\lambda_\theta=\prod_{i=1}^L \brsp{w_i}^{1/L}$ denote the geometric mean of the spectral norms of the weights. 
	 Then we have 
	 \begin{equation}
	 \begin{aligned}
	&\quad\left \{\frac{\lambda_\theta }{ \brsp{w_1}},\ldots,\frac{\lambda_\theta }{ \brsp{w_L}}\right \} \\
	&= 
	\arg \min_{c:\; D_{\theta_c} \in \mathcal D,\; \prod_{i=1}^L c_i=1,\; c_i\in \Rb^+}  \quad  \brf{\nabla_{\theta_c} D_{\theta_c}(x)} \nonumber
	 \end{aligned}
	 \end{equation}
\end{theorem}

\begin{figure}[t]
	\centering
	\includegraphics[width=0.7\linewidth]{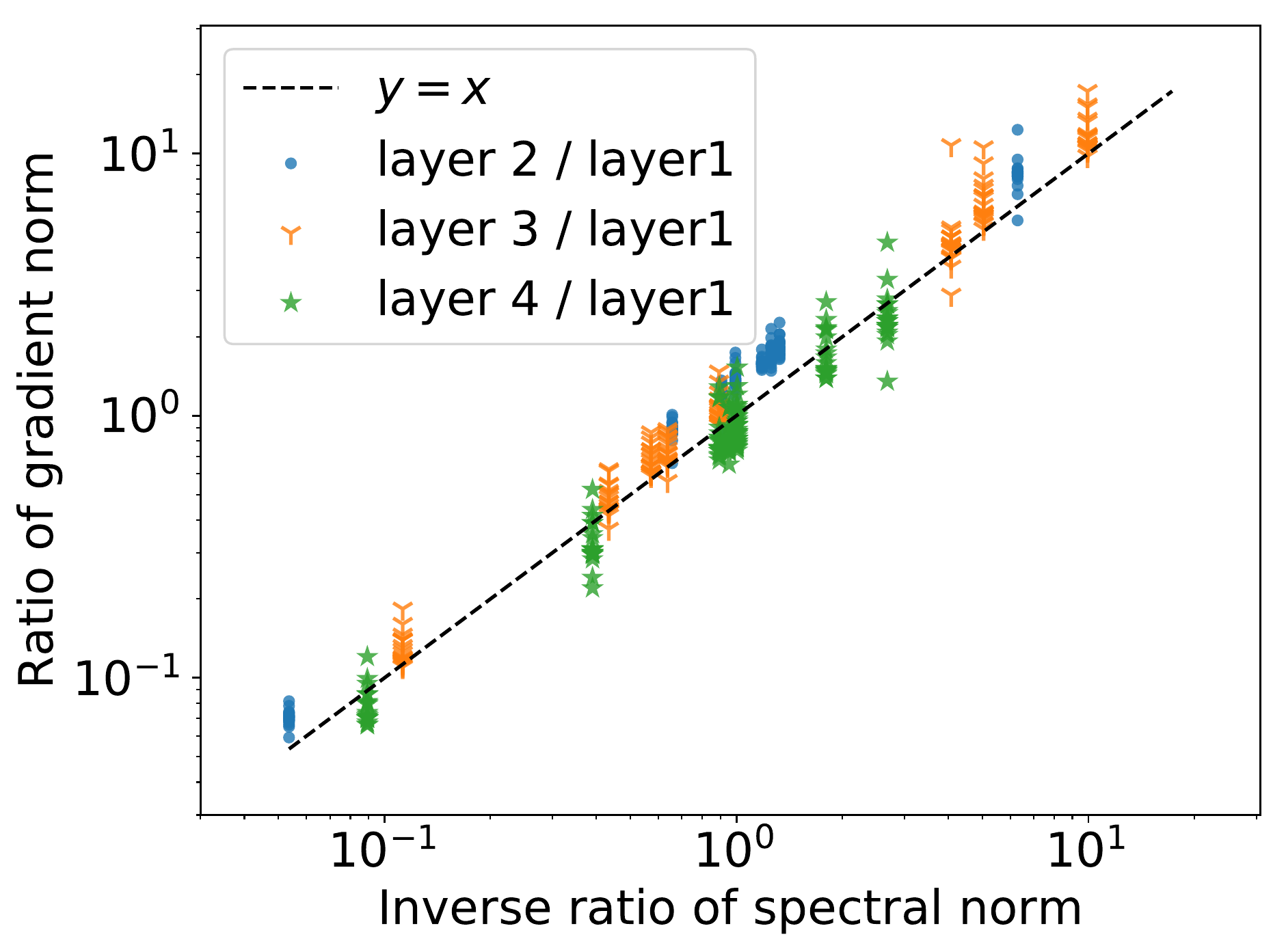}
	\vspace{-0.3cm}
	\caption{Ratio of gradient norm v.s. inverse ratio of spectral norm in \mnist{}. }
	\label{fig:gratio-mnist}
	\vspace{-0.5cm}
\end{figure}
\emph{(Proof in \cref{app:scaling-upperbound-practical}).}
The key constraint in this theorem is that we optimize only over discriminators in set $\mathcal D$ in \cref{eq:setD}.
To show that this constraint is realistic  (i.e., SN GAN discriminator optimization tends to choose models in $\mathcal D$),
we trained a spectrally-normalized GAN with four hidden layers on MNIST,  computing the ratios of the gradient norms at each layer and the ratios of the spectral norms, as dictated by \cref{eq:setD}.
We computed these ratios at different epochs during training, as well as for different randomly-selected rescalings of the spectral normalization vector $c$.
Each point in \cref{fig:gratio-mnist} represents the results averaged over 64 real samples at a specific epoch of training for a given (random) $c$. 
Vertical series of points are from different epochs of the same run, therefore their ratio of spectral norms is the same. 
The fact that most of the points are near the diagonal line suggests  that training naturally favors discriminators that are in or near $\mathcal D$; we confirm this intuition in other experimental settings in \cref{app:result_setd}.
This observation, combined with \cref{thm:scaling-upperbound-practical}, suggests that it is better to force the spectral norms of every layer to be equal. 
Hence, existing SN implementations \cite{miyato2018spectral,farnia2018generalizable} chose the correct, uniform normalization across layers to upper bound discriminator's gradients.

\section{Vanishing Gradients}
\label{sec:vanishing}

An equally troublesome failure mode of GAN training is vanishing gradients \cite{arjovsky2017principled}. 
Prior work has proposed new objective functions to mitigate this problem \cite{arjovsky2017principled,arjovsky2017wasserstein,gulrajani2017improved}, but these approaches do not fully solve the problem (see \cref{fig:grad_size}).  
In this section, we show that SN also helps to control vanishing gradients.

\myparatightest{How SN controls vanishing gradients}
Gradients tend to vanish for two reasons. 
First, gradients vanish when the objective function \emph{saturates} \cite{LeCun1998,arjovsky2017principled},
which is often associated with function parameters growing too large. 
Common loss functions (e.g., hinge loss) and activation functions (e.g., sigmoid, tanh) saturate for inputs of large magnitude. 
Large parameters tend to amplify the inputs to the activation functions and/or loss function, causing saturation.
Second, gradients vanish when function parameters (and hence, internal outputs) grow too small.  
This is because backpropagated gradients are scaled by the function parameters (\cref{app:proof-gradient-upperbound}).

These insights motivated the LeCun initialization technique \cite{LeCun1998}.
The key idea is that to prevent gradients from vanishing, we must ensure that the outputs of each neuron do not vanish or explode.  
If the inputs to a neural unit are uncorrelated random variables with variance 1, then to ensure that the unit's output also has variance (approximately) 1, the weight parameters should be zero-mean random variables with variance of $\frac{1}{n_i}$, where $n_i$ denote the fan-in  (number of incoming connections) of layer $i$ \cite{LeCun1998}. 
Hence, LeCun initialization prevents gradient vanishing by controlling the variance of the individual parameters. 
In the following theorem, we show that SN enforces a similar condition.

\begin{theorem}[Parameter variance of SN]
	\label{thm:sn-variance}
	For a matrix $A\in\Rb^{m\times n}$ with i.i.d. entries $a_{ij}$ from a symmetric distribution with zero mean (e.g., zero-mean Gaussian or uniform), we have 
	\begin{gather}
	\scalebox{1}{$
	 \var\bra{\frac{a_{ij}}{\brsp{A}}} \leq \frac{1}{\max\brc{m, n}} \;\;.
	 \label{eq:ub-var}
	$}
	\end{gather}
	Furthermore, if $m,n\geq 2$ and $\max\brc{m, n}\geq 3$, and $a_{ij}$ are from a zero-mean Gaussian, we have
	\begin{gather} 
	\scalebox{1}{$
	\frac{L}{\max\brc{m, n} \log\bra{ \min\brc{m,n} }}  \leq \var\bra{\frac{a_{ij}}{\brsp{A}}} \leq \frac{1}{\max\brc{m, n}} \;\;,\nonumber
	$}
	\end{gather}
	where $L$ is a constant which does not depend on $m, n$.
\end{theorem}
\emph{(Proof in \cref{app:sn-variance})}.
In other words, spectral normalization forces zero-mean parameters to have a  variance that scales inversely with  $\max\{m,n\}$.
The proof relies on a characterization of extreme values of random vectors drawn uniformly from the surface of a high-dimensional unit ball. 
Many fully-connected, 
feed-forward neural networks have a fixed width across hidden 
layers, so $\max\{m,n\}$ corresponds precisely to the fan-in of any neuron in a hidden layer, implying that SN has an effect like LeCun initialization. 

\myparatightest{Why \snw{}  works better than \snconv{}}
\label{par:snw_snconv}
In a CNN, the interpretation of $\max\{m,n\}$ depends on how SN is implemented. 
Recall that the implementation \snw{} by  \cite{miyato2018spectral}  does not strictly implement SN, but a variant that normalizes by the spectral norm of $\brmiyato{w_i} =\spwshape{w_i}$.
In architectures like DCGAN \cite{radford2015unsupervised}, the larger dimension of $\brmiyato{w_i}$ for \emph{hidden layers} tends to be $\cin k_w k_h$, which is exactly the fan-in. 
This means that SN gets the right variance for hidden layers in CNN.

    \begin{figure}[t]
 	\centering
	\includegraphics[width=0.6\linewidth]{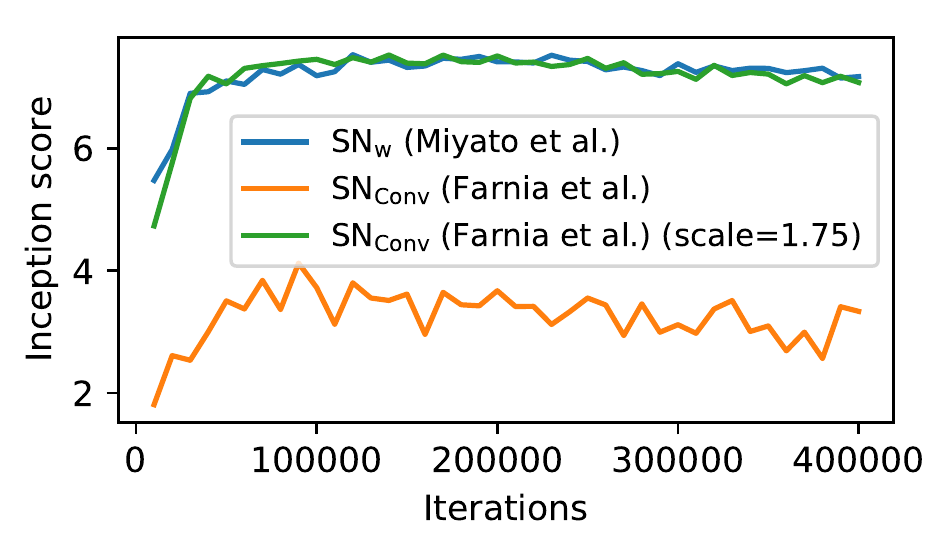}
	\vspace{-0.5cm}
	\caption{Inception score of different SN variants in \cifar{}.}
	\label{fig:sn_gradient_vani_cifar_is}
	\vspace{-0.2cm}
    \end{figure}%

    \begin{figure}[t]
      	\centering
		\includegraphics[width=0.6\linewidth]{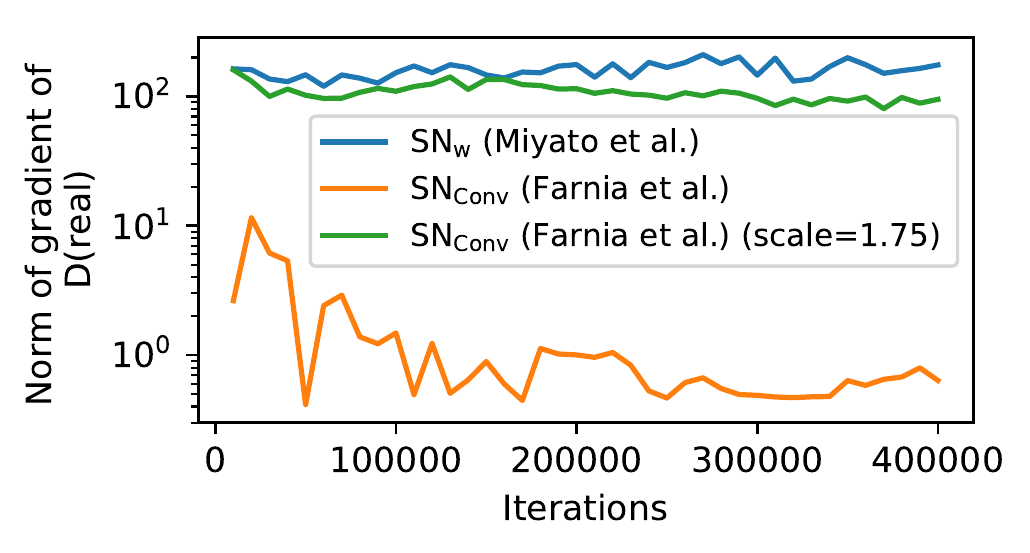}
		\vspace{-0.5cm}
		\caption{Gradient norms of different SN variants in \cifar{}.}
		\label{fig:sn_gradient_vani_cifar_real_gnorm}
		\vspace{-0.4cm}
    \end{figure}

	\begin{figure}[t]
		\centering
		\includegraphics[width=0.6\linewidth]{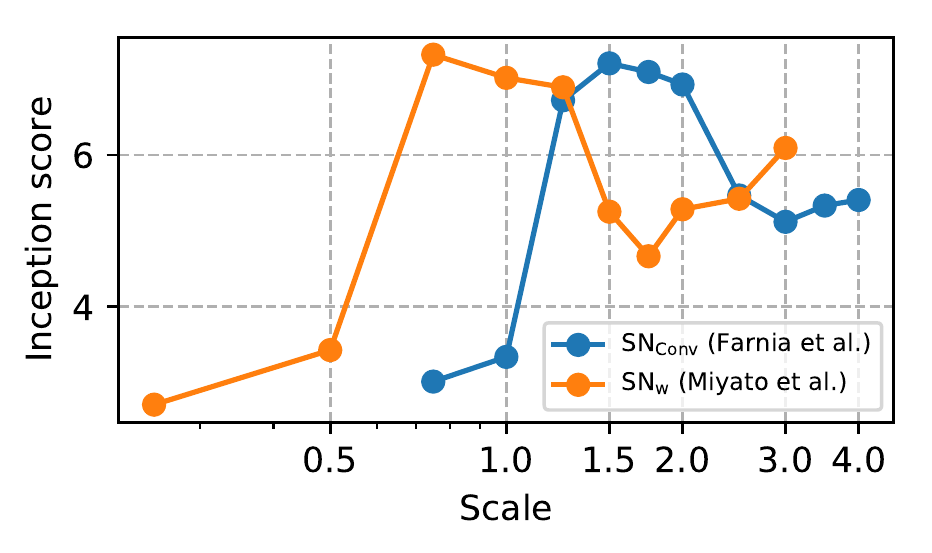}
		\vspace{-0.5cm}
		\caption{Inception score of scaled SN in \cifar{}.}
		\label{fig:sn_scale}
		\vspace{-0.4cm}
	\end{figure}

Perhaps surprisingly, we find empirically that the strict implementation \snconv{} of \cite{farnia2018generalizable} does \emph{not} prevent gradient vanishing. 
\cref{fig:sn_gradient_vani_cifar_is,fig:sn_gradient_vani_cifar_real_gnorm} shows the gradients of \snconv{}  vanishing when trained on CIFAR10, leading to a comparatively poor inception score, whereas the gradients of \snw{} remain stable. 
To understand this phenomenon, recall that  \snconv{} normalizes by the spectral norm of an expanded matrix $\brconv{w_i}$ derived from $w_i$. 
\cref{thm:sn-variance} does not hold for $\brconv{w_i}$ since its entries are not i.i.d. (even at initialization); hence it cannot be used to explain this effect.
However,  Corollary 1 in \cite{tsuzuku2018lipschitz} shows that $\brspw{w_i} \leq \brconvsp{w_i} \leq \alpha \brspw{w_i}$, where $\alpha$ is a constant only depends on kernel size, input size, and stride size of the convolution operation.
This result has two implications:

(1) $\brconvsp{w_i} \leq \alpha \brspw{w_i}$:  Although \snw{} does not strictly normalize the matrix with the actual spectral norm of the layer, it does upper bound the spectral norm of the layer. Therefore, all our analysis in \cref{sec:explosion} still applies for \snw{}  by changing the spectral norm constant from 1 to $\alpha \brspw{w_i}$. This means that \snw{} can still prevent  gradient explosion.

(2) $\brspw{w_i} \leq \brconvsp{w_i}$: 
This implies that \snconv{} normalizes by a factor that is at least as large as \snw{}.
In fact, we observe empirically that $\brconvsp{w_i}$ is strictly larger than $\brspw{w_i}$  during training (\cref{app:compare-sn}). This means that for the same $w_i$, a discriminator using \snconv{} will have smaller outputs than the discriminator using \snw{}. 
We hypothesize that the different scalings explain why  \snconv{} has vanishing gradients but \snw{} does not.

To confirm this hypothesis, for \snw{} and \snconv{}, we propose to multiply all the normalized weights by a scaling factor $s$, which is fixed throughout the training.  
\cref{fig:sn_scale} shows that
\snconv{} seems to be a shifted version of \snw{}. 
\snconv{} with $s=1.75$ has similar inception score (\cref{fig:sn_gradient_vani_cifar_is}) to  \snw{}, as well as similar gradients (\cref{fig:sn_gradient_vani_cifar_real_gnorm}) and parameter variances (\cref{app:parameter-variance-scaled-sn}) throughout training. 
This, combined with \cref{thm:sn-variance}, suggests that \snw{} inherently finds the correct scaling for the problem, whereas ``proper" spectral normalization \snconv{} requires additional hyper-parameter tuning.

\begin{figure}[t]
	\centering
	\includegraphics[width=\linewidth]{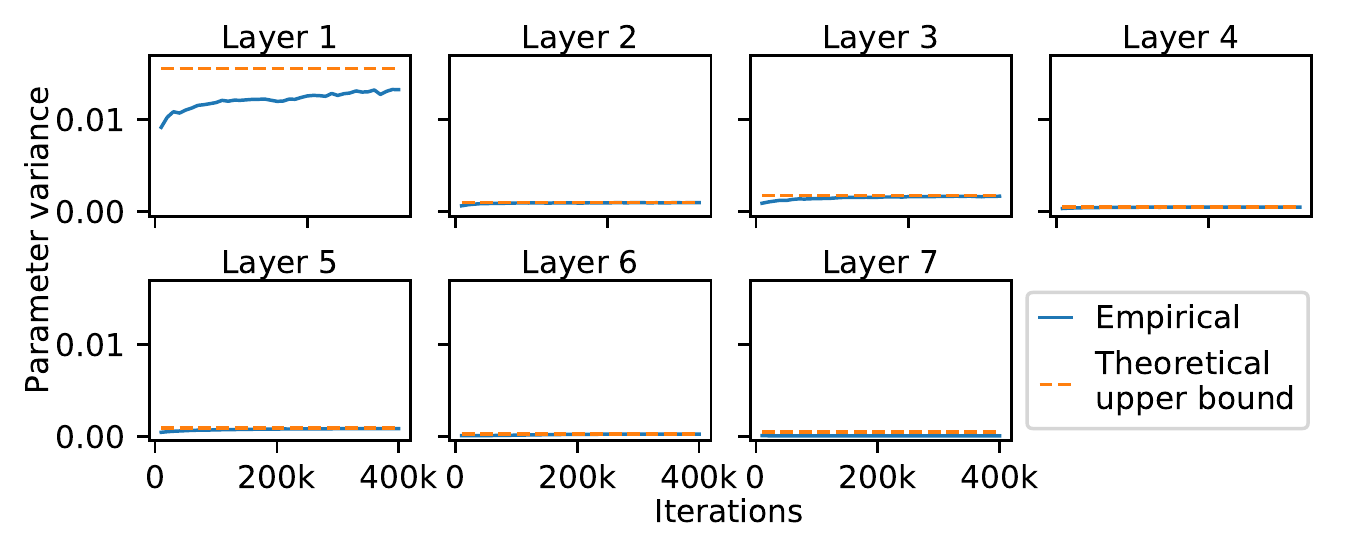}
	\vspace{-0.9cm}
	\caption{Parameter variances throughout training in \cifar{}. The blue lines show the parameter variances of different layers when SN is applied, and the original line shows our theoretical upper bound given in \cref{eq:ub-var}.}
	\label{fig:lecun_training}
	\vspace{-0.3cm}
\end{figure}

\myparatightest{SN has good parameter variances throughout training}
Our theoretical analysis only applies at initialization, when the parameters are selected randomly. 
However, unlike LeCun initialization which only controls the variance at initialization, we find empirically that \cref{eq:ub-var} for SN appears to hold \emph{throughout training} (\cref{fig:lecun_training}). 
As a comparison, if trained without SN, the variance increases and the gradient decreases, which makes sample quality bad (\cref{app:para-var-withwithoutsn}).
This explains why in practice GANs trained with SN are stable \emph{throughout training}.

\section{Extensions of Spectral Normalization}
\label{sec:approach}
Given the above theoretical insights, we propose an extension of spectral normalization called \namebssn{} (\namebssnshort{}). It combines two key ideas: bidirectional normalization and weight scaling. 

\subsection{Bidirectional Normalization}
\label{sec:ours}
Glorot and Bengio \cite{glorot2010understanding} built on the intuition of LeCun \cite{LeCun1998} to design an improved initialization, commonly called \emph{Xavier initialization}.
Their key observation was that to limit gradient vanishing (and explosion), it is not enough to control only feed-forward outputs;
we should also control the variance of backpropagated gradients.
Let $n_i,m_i$ denote the fan-in and fan-out of layer $i$. (In fully-connected layers, $n_i=m_{i-1}=$ the width of layer $i$.)
Whereas LeCun chooses initial parameters with variance  $\frac{1}{n_i}$, 
Glorot and Bengio choose them with variance $\frac{2}{n_i + m_{i}}$,
 a compromise between $\frac{1}{n_i}$ (to control output variance) and $\frac{1}{m_{i}}$ (to control variance of backpropagated gradients).
 
 The first component of \namebssnshort{} is  \namebsn{} (\namebsnshort{}), which applies a similar intuition to improve the spectral normalization of Miyato \emph{et al.} \cite{miyato2018spectral}.
 For fully connected layers, \namebsnshort{} keeps the normalization the same as \snw{} \cite{miyato2018spectral}.
For convolution layers, instead of normalizing by $\brspwshape{w}$, we normalize by $\sigma_w \triangleq \frac{\brspwshape{w} + \brspwshapet{w}}{2}$, where $\brspwshapet{w}$ is the spectral norm of the reshaped convolution kernel of dimension $\cin \times \bra{\cout k_wk_h}$. For calculating these two spectral norms, we use the same power iteration method in \cite{miyato2018spectral}. The following theorem gives the theoretical explanation.

\begin{theorem}[Parameter variance of \namebsnshort{}]
	\label{thm:ours-variance}
	For a convolutional kernel $w\in\Rb^{\cout \cin k_wk_h}$ with i.i.d. entries $w_{ij}$ from a symmetric distribution with zero mean (e.g. zero-mean Gaussian or uniform) where $k_wk_h \geq \max\brc{\frac{\cout}{\cin}, \frac{\cin}{\cout}}$, and $\sigma_w$ defined as above,
we have
	$$
		\var\bra{\frac{w_{ij}}{\sigma_w}} \leq \frac{2}{ \cin k_wk_h + \cout k_wk_h } \;\;.
	$$
	Furthermore, if $\cin,\cout\geq 2$ and $\cin k_wk_h, \cout k_wk_h\geq 3$, and $w_{ij}$ are from a zero-mean Gaussian distribution, there exists a constant $L$ that does not depend on $\cin, \cout,k_w,k_h$ such that
	\begin{align*}
	\frac{L}{  \cin k_w k_h \log(\cout) + \cout k_w k_h \log(\cin)  }  
	\\\leq \var\bra{\frac{w_{ij}}{\sigma_w} }
	\leq \frac{2}{ \cin k_wk_h + \cout k_wk_h }. 
	\end{align*}
\end{theorem}
\emph{(Proof in \cref{app:ours-variance})}.
Note that in convolution layers, $n_i=\cin k_w k_h$ and $m_i=\cout k_w k_h$. Therefore, \namebsnshort{} sets the variance of parameters to  scale as $\frac{2}{n_i + m_{i}}$, as dictated by Xavier initialization. Moreover, \namebsnshort{} naturally inherits the benefits of SN discussed in \cref{sec:vanishing} (e.g., controlling variance throughout the training).

\subsection{Weight Scaling}
\label{sec:scaling_approach}

The second component of \namebssnshort{} is to multiply all the normalized weights by a constant scaling factor (i.e., as we did in \cref{fig:sn_scale}). We call the combination of \namebsnshort{} and this weight scaling  \emph{\namebssn{}} (\namebssnshort{}). 
Note that scaling can also be applied independently to SN, which we call \namessn{} (\namessnshort{}).
The scaling is motivated by the following reasons. 

(1) The analysis in LeCun and Xavier initialization assumes that the activation functions are linear, which is not true in practice. More recently, Kaiming initialization was proposed to include the effect of non-linear activations \cite{he2015delving}. The result is that we should set the variance of parameters to be $2/(1+a^2)$ times the ones in LeCun or Xavier initialization, where $a$ is the negative slope of leaky ReLU. This suggests the importance of a constant scaling. 

(2) However, we found that the scaling constants proposed in LeCun/Kaiming initialization do not always perform well for GANs. Even more surprisingly, there are \emph{multiple modes} of good scaling. \cref{fig:sn_scaling_approach} shows the sample quality of LeCun initialization with different scaling on the discriminator. We see that there are at least two good modes of scaling: one at around 0.2 and the other at around 1.2. This phenomenon cannot be explained by the analysis in LeCun/Kaiming initialization. 

Recall that SN has similar properties as LeCun initialization (\cref{sec:vanishing}).  
Interestingly, we see that \namessnshort{} also has two good modes of scaling (\cref{fig:sn_scaling_approach}).  Although the best scaling constants for LeCun initialization and SN are very different, there indeed exists an interesting mode correspondence in terms of parameter variances (\cref{app:result_scaling_approach}).  We hypothesize that the shift of good scaling from Kaiming initialization we see here could result from  adversarial training, and defer the theoretical analysis to future work. These results highlight the need for a separate scaling factor.

(3) The bounds in \cref{thm:sn-variance} and \cref{thm:ours-variance} only imply that in SN and \namebsnshort{} the \emph{order} of parameter variance w.r.t. the network size is correct, but constant scaling is unknown.

\begin{figure}[th]
	\centering
	\includegraphics[width=0.7\linewidth]{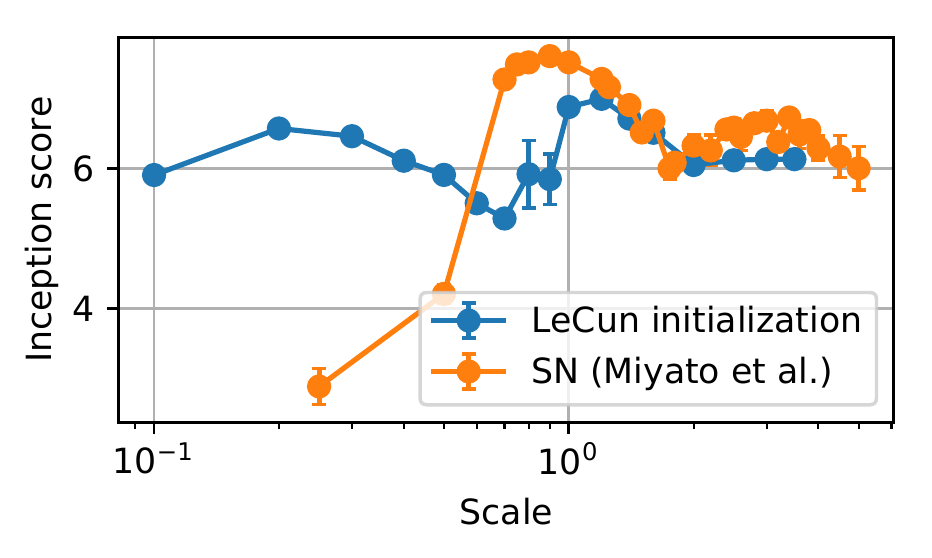}
	\vspace{-0.5cm}
	\caption{Inception score of \namessnshort{} and scaled LeCun initialization in CIFAR10. Mean and standard error of the best score during training across multiple runs are shown.}
	\label{fig:sn_scaling_approach}
	\vspace{-0.4cm}
\end{figure} 

\subsection{Results}
\label{sec:results}

In this section we verify the effectiveness of \namebssnshort{}
with extensive experiments. The code for reproducing the results is at \url{https://github.com/fjxmlzn/BSN}.

\cite{miyato2018spectral} already compares SN with many other regularization techniques like WGAN-GP \cite{gulrajani2017improved}, batch normalization \cite{ioffe2015batch}, layer normalization \cite{ba2016layer}, weight normalization \cite{salimans2016weight},  and orthogonal regularization \cite{brock2016neural}, and SN is shown to outperform them all. 
Therefore, we focus on comparing the performance of SN with \namebssnshort{} here. Additionally, to isolate the effects of the two components proposed in \namebssnshort{}, we include comparison against bidirectional normalization without scaling (\namebsnshort) and scaling without bidirectional normalization (\namessnshort).

We conduct experiments across \emph{different datasets} (from low-resolution to high-resolution) and \emph{different network architectures} (from standard CNN to ResNets). More specifically, we conducts experiments on \cifar{}, \stl{}, \celeba{}, and \imagenet{} (\ilsvrc{}), following the same settings in \cite{miyato2018spectral}. All experimental details are attached in \cref{app:result,app:sn-variant,app:result_cifar10,app:result_stl10,app:result_celeba,app:result_ilsvrc}. The results are summarized in \cref{tbl:all}.

\begin{table*}[ht]
	\centering
	\setlength{\tabcolsep}{4.5pt}
	\scalebox{1.0}{
		\begin{tabular}{l | c c | c c | c | c c}
			\toprule
			& \multicolumn{2}{c|}{\cifar{}} & \multicolumn{2}{c|}{\stl{}} &\multicolumn{1}{c|}{\celeba{}} &\multicolumn{2}{c}{\ilsvrc{}} \\
			\midrule
			& IS $\uparrow$ & FID $\downarrow$ & IS $\uparrow$ & FID $\downarrow$ & FID $\downarrow$ & IS $\uparrow$ & FID $\downarrow$\\
			\midrule
			Real data & 11.26 & 9.70 & 26.70 & 10.17 & 4.44 & 197.37 & 15.62\\
			\hline
			SN&  7.12 $\pm$ 0.07 & 31.43 $\pm$ 0.90   & 9.05 $\pm$ 0.05 & 44.35 $\pm$ 0.54    & 9.43 $\pm$ 0.09  &12.84 $\pm$ 0.33 & 75.06 $\pm$ 2.38 \\\hline
			\namessnshort{} & 7.38 $\pm$ 0.06 & 29.31 $\pm$ 0.23   & \textbf{9.28 $\pm$ 0.03} & 43.52 $\pm$ 0.26   & \textbf{8.50 $\pm$ 0.20}    & 12.84 $\pm$ 0.33 & 73.21 $\pm$ 1.92\\
			\namebsnshort{}& \textbf{7.54 $\pm$ 0.04} & \textbf{26.94 $\pm$ 0.58}   & 9.25 $\pm$ 0.01 & 42.98 $\pm$ 0.54  & 9.05 $\pm$ 0.13    & 1.77 $\pm$ 0.13 & 265.20 $\pm$ 19.01\\
			\namebssnshort{} & \textbf{7.54 $\pm$ 0.04} & \textbf{26.94 $\pm$ 0.58}   & 9.25 $\pm$ 0.01 & \textbf{42.90 $\pm$ 0.17}  & 9.05 $\pm$ 0.13   & \textbf{13.23 $\pm$ 0.16} & \textbf{69.04 $\pm$ 1.46}\\
			\bottomrule
		\end{tabular}
	}
	\vspace{-0.2cm}
	\caption{Inception score (IS) and FID on \cifar{}, \stl{}, \celeba{}, and \ilsvrc{}. 
		The last three rows are proposed in this work, with \namebssnshort{} representing our final proposal---a combination of \namebsnshort{} and  \namessnshort{}. %
		Each experiment is conducted with 5 random seeds except that the last three rows on \ilsvrc{} is conducted with 3 random seeds. Mean and standard error across these random seeds are reported. 
		We follow the common practice of excluding IS in \celeba{} as the inception network is pretrained on ImageNet, which is very different from CelebA. The bold font marks the best numbers in that column.}
	\vspace{-0.3cm}
	\label{tbl:all}
\end{table*}

\trainingcurveinception{\cifar{}}{cifar10/296}{0.0001}{0.0001}{1}{cifar}

\myparatightest{\namebsnshort{} v.s. SN (showing the effect of bidirectional normalization \cref{sec:ours})}
By comparing \namebsnshort{} with SN in \cref{tbl:all}, we can see that \namebsnshort{} outperforms SN by a large margin in all metrics except in \ilsvrc{} (discussed later). 
\newline\ul{More importantly, the superiority of \namebsnshort{} is stable across hyper-parameters.} In \cref{app:result}, we vary the learning rates ($\alpha_g, \alpha_d$)  and momentum parameters of generator and discriminator, and the number of discriminator updates per generator update ($n_{dis}$). We see that \namebsnshort{} consistently outperforms SN in most of the cases.
\newline\ul{Moreover, \namebsnshort{} is more stable in the entire training process.}  We see that as training proceeds, the sample quality of SN often drops, whereas the sample quality of \namebsnshort{} appears to monotonically increase (\cref{fig:cifar-g0.0001-d0.0001-ndis1-inception}, more in \cref{app:result_cifar10,app:result_stl10,app:result_celeba}). In most cases, \namebsnshort{} not only outperforms SN in final sample quality (i.e., at the end of training), but also in \emph{peak} sample quality. 
This means that \namebsnshort{} makes the training process more stable, which is the  purpose of SN (and \namebsnshort{}). 

\myparatightest{\namessnshort{} v.s. SN (showing the effect of scaling \cref{sec:scaling_approach})}
By comparing \namessnshort{} with SN in \cref{tbl:all}, we see that scaling consistently improves (or has the same metric) \emph{in all cases}.  This verifies our intuition in \cref{sec:scaling_approach} that the inherent scaling in SN is not optimal, and a extra constant scaling is needed to get the best results.

\myparatightest{\namebssnshort{} v.s. \namebsnshort{} (showing the effect of scaling \cref{sec:scaling_approach})} By comparing \namebssnshort{} with \namebsnshort{} in \cref{tbl:all}, we see that in some cases the optimal scale of \namebsnshort{} happens to be 1 (e.g., in \cifar{}), but in other cases, scaling is critical.  For example, in \ilsvrc{}, \namebsnshort{} without any scaling has the same gradient vanishing problem we observe for \snconv{} \cite{farnia2018generalizable} in \cref{sec:vanishing}, which causes bad sample quality. \namebssnshort{} successfully solves the gradient vanishing problem and achieves the best sample quality.

\myparatightest{Additional results} Because of the space constraints, we defer other results (e.g., generated images, training curves, more  comparisons and analysis) to \cref{app:result,app:sn-variant,app:result_cifar10,app:result_stl10,app:result_celeba,app:result_ilsvrc}.

\myparatightest{Summary}
In summary, both designs we proposed can effectively stabilize training and achieve better sample quality. Combining them  together, \namebssnshort{} achieves the best sample quality in most of the cases. This demonstrates the practical value of the theoretical insights in \cref{sec:explosion,sec:vanishing}.

\section{Discussion}
\label{sec:discussion}

\myparatightest{Other reasons contributing to the stability of SN} 
In the paper we present one possible reason (i.e., SN avoids exploding and vanishing gradients), and show such correlation through extensive theoretical and empirical analysis. However, there could exists many other parallel factors. For example, SN paper \cite{miyato2018spectral} points out that SN could speed up training by encouraging the weights to be updated along directions orthogonal to itself. This is orthogonal to the reasons we discuss in the paper.

\myparatightest{Related work}
A related result to our upper bound was shown in \cite{santurkar2018does}, which shows that batch normalization (BN)  makes the scaling of the Hessian along the direction of  the gradient smaller, thereby making gradients more predictive. 
Given \cref{thm:gradient-upperbound-sn}, we can apply the reasoning from  \cite{santurkar2018does} to explain why spectrally-normalized GANs are robust to different learning rates as shown in \cite{miyato2018spectral}.
However, our insights regarding the gradient vanishing problem are the more surprising result; this notion is not discussed in \cite{santurkar2018does}. 
An interesting question for future work is whether BN similarly controls vanishing gradients.

In parallel to this work, some other approaches have been proposed to improve SN. For example, \cite{fang2021precondition} finds out that even with SN, the condition numbers of the weights can still be large, which causes the instability. To solve the issue, they borrow the insights from linear algebra and propose precondition layers to improve the condition numbers and therefore promote stability.

\myparatightest{Future directions}
Our results suggest that SN stabilizes GANs by controlling exploding and vanishing gradients in the discriminator.  
However, our analysis also applies to the training of any feed-forward neural network. 
This connection partially explains why SN helps train generators as well as  discriminators \cite{zhang2018self,brock2018large}, and why SN is more generally useful  in training neural networks \cite{farnia2018generalizable,gouk2018regularisation,yoshida2017spectral}.
We focus on GANs in this paper because SN seems to have a disproportionately beneficial effect on GANs \cite{miyato2018spectral}.
Formally extending this analysis to understand the effects of adversarial training is an interesting direction for future work.

Related to the weight initialization and training dynamics, a series of work \cite{pennington2017resurrecting,saxe2013exact} has shown that Gaussian weights or ReLU activations cannot achieve dynamical isometry (all singular values of the network Jacobian are near 1), a desired property for training stability. Orthogonal weight initialization may be better at achieving the goal. In this paper, we focus the theoretical analysis and experiments on Gaussian weights and ReLU activations as they are the predominant implementations in GANs. We defer the study of other networks to future work.

\section*{Acknowledgements}
This work was supported in part by faculty research awards from Google, JP Morgan Chase, and the Sloan Foundation, as well as a gift from Siemens AG.
This research was sponsored in part by National Science Foundation Convergence Accelerator award 2040675 and the U.S. Army Combat
Capabilities Development Command Army Research Laboratory and was accomplished under Cooperative Agreement
Number W911NF-13-2-0045 (ARL Cyber Security CRA).
The views and conclusions contained in this document are
those of the authors and should not be interpreted as representing the official policies, either expressed or implied, of the
Combat Capabilities Development Command Army Research
Laboratory or the U.S. Government. The U.S. Government
is authorized to reproduce and distribute reprints for Government purposes notwithstanding any copyright notation here
on.
This work used the Extreme Science and Engineering Discovery Environment (XSEDE) \cite{xsede}, which is supported by National Science Foundation grant number ACI-1548562. Specifically, it used the Bridges system \cite{bridges}, which is supported by NSF award number ACI-1445606, at the Pittsburgh Supercomputing Center (PSC).

\bibliography{reference}
\bibliographystyle{icml2021}

\appendix
\onecolumn

\section{Proof of \cref{thm:gradient-upperbound-sn}}
\label{app:proof-gradient-upperbound}

The proposition makes use of the following observation: 
	For the discriminator defined in \eqref{eq:d}, the norm of gradient for $w_t$ is upper bounded by 
	\begin{equation}
	 \brf{\nabla_{w_t} D_{\theta}(x)} \leq \brn{x} \cdot \prod_{i=1}^{L}\brlip{a_i} \cdot \prod_{i=1}^{L}\brsp{w_i} \bigg/ \brsp{w_t} \quad\text{  for }\forall t\in [1, L]
	\label{eq:lem1}
	\end{equation}

To prove this, for  simplicity of notation, let $o_a^{i} = a_i\circ l_{w_i}\circ \ldots \circ a_1 \circ l_{w_1} $, and $o_l^{i} = l_{w_i}\circ a_{i-1}\circ \ldots \circ a_1 \circ l_{w_1} $.

	It is straightforward to show that the norm of each internal output of discriminator is bounded by
	\begin{equation}
	\brn{o_a^{t}(x)} \leq \brn{x} \cdot \prod_{i=1}^{t} \brlip{a_i} \cdot \prod_{i=1}^{t} \brsp{w_i}
	\label{eq:in1}
	\end{equation}
	and
	\begin{equation}
	\brn{o_l^t(x)} \leq \brn{x} \cdot \prod_{i=1}^{t-1} \brlip{a_i} \cdot \prod_{i=1}^{t} \brsp{w_i}.
	\label{eq:in2}
	\end{equation}
	This holds because
	\begin{align*}
		\brn{o_a^t(x)} = \brn{a_i\bra{o_l^t(x)}} \leq \brlip{a_i} \cdot \brn{o_l^t(x)}
	\end{align*}
	and 
	\begin{align*}
		\brn{o_l^t(x)} = \brn{l_{w_i}\bra{o_a^{t-1}(x)}} \leq \brsp{w_t} \cdot \brn{o_a^{t-1}(x)}\;,
	\end{align*}
	from which we can show the desired inequalities by induction.

	Next, we observe that the norm of each internal gradient is bounded by
	\begin{equation}
	 \brn{ \nabla_{o_a^t\bra{x}} D_\theta\bra{x} } \leq \prod_{i=t+1}^{L} \brlip{a_i} \cdot \prod_{i=t+1}^{L} \brsp{w_i} 
	 \label{eq:in3}
	 \end{equation}
	and 
	\begin{equation} 
	\brn{ \nabla_{o_l^t\bra{x}} D_\theta\bra{x} } \leq \prod_{i=t}^{L} \brlip{a_i} \cdot \prod_{i=t+1}^{L} \brsp{w_i}.
	\label{eq:in4}
	\end{equation}
	This holds because
	\begin{align*}
		\brn{ \nabla_{o_a^t\bra{x}} D_\theta\bra{x} }  = \brn{ w_{t+1}^T\nabla_{o_l^{t+1}(x)} D_\theta(x) } \leq \brsp{w_{t+1}} \brn{\nabla_{a_l^{t+1}(x)}D_\theta(x)}
	\end{align*}
	and
	\begin{align*}
		\brn{ \nabla_{o_l^t\bra{x}} D_\theta\bra{x} } = \brn{ \inner{\nabla_{o_a^t\bra{x}} D_\theta\bra{x} }{ \brb{ a_t'(x)|_{x=o_l^t(x)} } } } \leq \brlip{a_t} \brn{ \nabla_{o_a^t\bra{x}} D_\theta\bra{x} } \;,
	\end{align*}
	from which we can show inequalities \cref{eq:in3,eq:in4} by induction.

Now we have that
	\begin{align*}
		\brf{\nabla_{w_t} D_{\theta}(x)} &= \brf{ \nabla_{o_l^t\bra{x}} D_\theta\bra{x}  \cdot  \bra{o_a^{t-1}(x)}^\T} \\
		&= \brn{\nabla_{o_l^t\bra{x}} D_\theta\bra{x}} \cdot \brn{o_a^{t-1}(x)} \\
		&\leq \prod_{i=t}^{L} \brlip{a_i} \cdot \prod_{i=t+1}^{L} \brsp{w_i} \cdot \brn{x} \cdot \prod_{i=1}^{t-1} \brlip{a_i} \cdot \prod_{i=1}^{t-1} \brsp{w_i}\\
		&=\brn{x} \cdot \prod_{i=1}^{L}\brlip{a_i} \cdot \prod_{i=1}^{L}\brsp{w_i} \bigg/ \brsp{w_t} 
	\end{align*}
	where we use \cref{eq:in1,eq:in2,eq:in3,eq:in4} at the inequality.
		The upper bound of gradient's Frobenius norm for spectrally-normalized discriminators follows directly.

\section{Proof of \cref{thm:scaling}}
\label{app:proof-scaling}
\begin{proof}
	As $l_{w}(x)$ is a linear transformation, we have $l_{cw}(x)=c\cdot l_{w}(x)$, and $l_w(cx)=c\cdot l_w(x)$. Moreover, since ReLU and leaky ReLU is linear in $\Rb^+$ and $\Rb^-$ region, we have $a_i(cx)=c\cdot a_i(x)$. Therefore, we have
	\begin{align*}
	D_\theta'(x) 
	&= \bra{a_L \circ l_{c_L\cdot w_L} \circ a_{L-1} \circ l_{c_{L-1}\cdot w_{L-1}} \circ \ldots \circ a_1 \circ l_{c_1\cdot w_1}}(x)\\
	&= \prod_{i=1}^{L}c_i\cdot \bra{a_L \circ l_{w_L} \circ a_{L-1} \circ l_{w_{L-1}} \circ \ldots \circ a_1 \circ l_{w_1}}(x)\\
	&= D_\theta(x)
	\end{align*}
\end{proof}

\section{Additional Analysis of Gradient}
\label{app:gradient}

In \cref{sec:explosion}, we discuss the gradients with respect to $w_i'=\frac{w_i}{u_i^Tw_iv_i}$, where $u_i,v_i$ are the singular vectors corresponding to the largest singular values. In this section we discuss the gradients with respect the actual parameter $w_i$. From Eq. (12) in \cite{miyato2018spectral} we know
\begin{align*}
	\nabla_{w_t} D_{\theta}(x) = \frac{1}{\brsp{w_t}} \bra{ \nabla_{w_t'}D_{\theta}(x) -  \bra{\bra{\nabla_{o_l^t\bra{x}} D_\theta\bra{x}}^To_l^t\bra{x}} \cdot u_tv_t^T}
\end{align*}

From \cref{app:proof-gradient-upperbound}, we know that $\brf{\nabla_{w_t'}D_{\theta}(x)}$, $\brn{\nabla_{o_l^t\bra{x}} D_\theta\bra{x}}$, and $\brn{o_l^t\bra{x}}$ have upper bounds. Furthermore, $\brf{u_tv_t^T}=1$. Therefore, $\brf{\nabla_{w_t'}D_{\theta}(x) -  \bra{\bra{\nabla_{o_l^t\bra{x}} D_\theta\bra{x}}^To_l^t\bra{x}} \cdot u_tv_t^T}$ has an upper bound. From Theorem 1.1 in \cite{seginer2000expected} we know that if $w_t$ is initialized with i.i.d random variables from uniform or Gaussian distribution, $\Eb\bra{\brsp{w_t}}$ is lower bounded away from zero at initialization. 
So $\brf{\nabla_{w_t} D_{\theta}(x)}$ is upper bounded at initialization. 
Moreover, we observe empirically that $\brsp{w_t}$ is usually increasing during training. Therefore, $\brf{\nabla_{w_t} D_{\theta}(x)}$ is typically upper bounded during training as well.

\section{Analysis of Hessian}
\label{app:hessian}

The following proposition states that spectral  normalization also gives an upper bound on $\brsp{H_{w_i}(D_\theta)(x)}$ for networks with ReLU or leaky ReLU internal activations.

\begin{proposition}[Upper bound of Hessian's spectral norm]
	\label{thm:hessian-upperbound}
	Consider the discriminator defined in \cref{eq:d}. Let $H_{w_i}(D_\theta)(x)$ denote the Hessian of $D_\theta$ at x with respect with the vector form of $w_i$. If the internal activations are ReLU or leaky ReLU, the spectral norm of $H_{w_i}(D_\theta)(x)$ is upper bounded by 
	$$\brsp{H_{w_i}(D_\theta)(x)} \leq \brsp{ H_{o_l^L(x)} D_\theta(x) } \cdot \brn{x}^2 \cdot \prod_{i=1}^L \brsp{w_{i}}^2 \bigg/ \brsp{w_{t}}^2$$
	
\end{proposition}
The proof is in \cref{sec:proof_hessian}.
Following \cref{thm:hessian-upperbound}, we can easily show the upper bound of Hessian's spectral norm for spectral normalized discriminators.
\begin{corollary}[Upper bound of Hessian's spectral norm for spectral normalization]
	If the internal activations are ReLU or leaky ReLU, and $\brsp{w_i} \leq 1$ for all $i\in [1, L]$, then
	$$\brsp{H_{w_i}(D_\theta)(x)} \leq \brsp{ H_{o_l^L(x)} D_\theta(x) } \cdot \brn{x}^2 \;.$$
	Moreover, if the activation for the last layer is sigmoid (e.g., for vanilla GAN \cite{goodfellow2014generative}), we have
	$$\brsp{H_{w_i}(D_\theta)(x)} \leq 0.1\brn{x}^2 \;;$$
	if the activation function for the last layer is identity (e.g., for WGAN-GP \cite{gulrajani2017improved}), we have
	$$\brsp{H_\theta(D_\theta)(x)}=0 \;.$$
\end{corollary}

\subsection{Proof of \cref{thm:hessian-upperbound}}
\label{sec:proof_hessian}

\begin{lemma}
	\label{thm:internal_hessian}
	The spectral norm of  each internal Hessian is bounded by
	$$ \brsp{ H_{o_a^t(x)}D_\theta(x) } \leq \brsp{ H_{o_l^L(x)} D_\theta(x) } \cdot \prod_{i=t+1}^L \brsp{w_{i}}^2$$
	and 
	$$ \brsp{ H_{o_l^t(x)}D_\theta(x) } \leq  \brsp{ H_{o_l^L(x)} D_\theta(x) } \cdot \prod_{i=t+1}^L \brsp{w_{i}}^2 $$
\end{lemma}
\begin{proof}
	We have 
	\begin{align*}
		\brsp{ H_{o_a^t(x)}D_\theta(x) } &= \brsp{ w_{t+1}^T \cdot \nabla_{a_l^{t+1}(x)} D_\theta(x) \cdot w_{t+1}} \\
		&\leq \brsp{\nabla_{a_l^{t+1}(x)} D_\theta(x) } \brsp{w_{t+1}}^2 \;.
	\end{align*}
	We also have 
	\begin{align*}
		\brsp{ H_{o_l^t(x)}D_\theta(x) } 
		&= \brsp{ \diag{ [a_t'(x)]_{x=o_a^t(x) } } \cdot H_{o_a^{t+1}(x)} D_\theta (x) \cdot \diag{ [a_t'(x)]_{x=o_a^t(x) } } } \\
		&\leq \brsp{ H_{o_a^{t+1}(x)} D_\theta (x) }
	\end{align*}
	where we use the property that ReLU or leaky ReLU is piece-wise linear.
	The desired inequalities then follow by induction.
\end{proof}
Now let's come back to the proof for \cref{thm:hessian-upperbound}.
\begin{proof}
	We have
	\begin{align*}
		\frac{\partial D_\theta}{\partial \bra{w_t}_{ij} \partial \bra{w_t}_{kl}}
		= \bra{ H_{o_l^t}(D_\theta)(x) }_{ik} \cdot \bra{o_a^{t-1}(x)}_j \cdot  \bra{o_a^{t-1}(x)}_l \;.
	\end{align*}
	Therefore,
	\begin{align*}
		\brsp{H_{w_i}(D_\theta)(x)} \leq \brsp{H_{o_l^t}(D_\theta)(x) } \brinfinity{o_a^{t-1}(x)}^2 \leq \brsp{H_{o_l^t}(D_\theta)(x) } \brn{{o_a^{t-1}(x)}}^2
	\end{align*}
	Applying \cref{eq:in1} and \cref{thm:internal_hessian} we get
	\begin{align*}
		\brsp{H_{w_i}(D_\theta)(x)} 
		&\leq  \brsp{ H_{o_l^L(x)} D_\theta(x) } \cdot \prod_{i=t+1}^L \brsp{w_{i}}^2\cdot 
		\brn{x}^2 \cdot \prod_{i=1}^{t-1} \brsp{w_i}^2\\
		&= \brsp{ H_{o_l^L(x)} D_\theta(x) } \cdot \brn{x}^2 \cdot \prod_{i=1}^L \brsp{w_{i}}^2 \bigg/ \brsp{w_{t}}^2
	\end{align*}
\end{proof}

\section{Proof of \cref{thm:scaling-upperbound-practical}}
\label{app:scaling-upperbound-practical}
\begin{proof}
	For any discriminator $D_\theta=a_L \circ l_{w_L} \circ a_{L-1} \circ l_{w_{L-1}} \circ \ldots \circ a_1 \circ l_{w_1}$, consider $\theta'=\brc{w_t' \triangleq c_tw_t}_{t=1}^L$ with the constraint $\prod_{i=1}^{L} c_i=1$ and $c_i\in \Rb^+$. Let $Q=\brf{\nabla_{w_i'}D_{\theta'}(x)} \brsp{w_i'}$. We have
	\begin{align*}
		\brf{\nabla_{\theta'} D_{\theta'}(x)}
		&= \sqrt{ \sum_{i=1}^{L} \brf{ \nabla_{w_i'}D_{\theta'}(x) }^2 }\\
		&= \sqrt{ \sum_{i=1}^{L} \frac{Q^2}{c_i^2\brsp{w_i}^2} }\\
		&\geq \sqrt{ L  \bra{\prod_{i=1}^{L} \frac{Q^2}{c_i^2\brsp{w_i}^2} }^{1/L} }\\
		&= \sqrt{L}\cdot Q^{1/L} \cdot \bra{\prod_{i=1}^{L} \brsp{w_i}}^{-1/L}
	\end{align*}
	and the equality is achieved iff $c_i^2\brsp{w_i}^2=c_j^2\brsp{w_j}^2,\;\forall i,j\in [1,L]$ according to AM-GM inequality. When $c_i^2\brsp{w_i}^2=c_j^2\brsp{w_j}^2,\;\forall i,j\in [1,L]$, we have $c_t=\prod_{i=1}^L \brsp{w_i}^{1/L} \bigg/ \brsp{w_t}$.
\end{proof}

\section{Proof of \cref{thm:sn-variance}}
\label{app:sn-variance}
\begin{proof}

	Since $a_{ij}$ are symmetric random variables, we know $\Eb\bra{\frac{a_{ij}}{\brsp{A}}} = 0$. 
	Further, by symmetry, we have that for any $(i,j)\neq (h,\ell)$, 
	$\Eb\bra{ \frac{a_{ij}^2}{\brsp{A}^2} } = \Eb\bra{ \frac{a_{h\ell}^2}{\brsp{A}^2} }$.
	Therefore, we have
	\begin{align*}
		\var\bra{ \frac{a_{ij}}{\brsp{A}} }
		= \Eb\bra{ \frac{a_{ij}^2}{\brsp{A}^2} }
		= \frac{1}{m n}\cdot  \Eb\bra{\frac{\sum_{i=1}^{m}\sum_{j=1}^{n} a_{ij}^2}{\brsp{A}^2} }
		= \frac{1}{mn}\cdot \Eb\bra{ \frac{\brf{A}^2}{ \brsp{A}^2}}
	\end{align*}
	Our approach will be to upper and lower bound the quantity $\frac{1}{mn}\cdot \Eb\bra{ \frac{\brf{A}^2}{ \brsp{A}^2}}$.
	
	\paragraph{Upper bound}
	Assume the singular values of $A$ are $\sigma_1\geq \sigma_2\geq \ldots \geq \sigma_{\min\brc{m,n}}$. We have 
	\begin{align*}
		\frac{1}{mn}\cdot \Eb\bra{ \frac{\brf{A}^2}{ \brsp{A}^2}}
		=\frac{1}{mn} \cdot \Eb\bra{ \frac{\sum_{i=1}^{\min\brc{m,n}}  \sigma_i^2}{\sigma_1^2} } \leq \frac{\min\brc{m,n}}{mn} = \frac{1}{\max\brc{m, n}} \;\;,
	\end{align*}
	which gives the desired upper bound.

	\paragraph{Lower bound}	
	Now for the lower bound, if $a_{ij}$ are drawn from zero-mean Gaussian distribution and $\max\brc{m, n}\geq 3$, we have
	\begin{align}
		&\;\;\;\; \frac{1}{mn}\cdot \Eb\bra{ \frac{\brf{A}^2}{ \brsp{A}^2}} \label{eq:original} \\
		&= \frac{1}{mn}\cdot \Eb\bra{ \frac{1}{ \brsp{A}^2 / \brf{A}^2}} \nonumber \\
		&\geq \frac{1}{mn} \cdot \frac{1}{\Eb \bra{\brsp{ \frac{A}{\brf{A}} }^2}} \nonumber\\
		&= \frac{1}{mn} \cdot \frac{1}{\Eb\bra{ \brsp{B}^2 }} \label{eq:B} 
	\end{align}
	where $B \in R^{m\times n}$ is uniformly sampled from the sphere of $m\times n$-dimension unit ball. 
	We use the following lemma to lower bound \eqref{eq:B}. 
	\begin{lemma}[Theorem 1.1 in \cite{seginer2000expected}]
	\label{thm:spnorm-sphere}
	Assume $A \in R^{m\times n}$ is uniformly sampled from the sphere of $m\times n$-dimension unit ball. When $\max\brc{m,n}\geq 3$, we have
	$$\Eb\bra{ \brsp{A}^2 } \leq K^2\bra{ \Eb \bra{ \max_{1\leq i\leq m}\brn{a_{i\bullet}}^2 }  +\Eb \bra{ \max_{1\leq j\leq n}\brn{a_{\bullet j}}^2 } }\;\;,$$
	where $K$ is a constant which does not depend on $m, n$. Here $a_{i\bullet}$ denotes the $i$-th row of $A$, and $a_{\bullet j}$ denotes the $j$-th column of $A$.\footnote{Note that the original theorem in \cite{seginer2000expected} requires that the entries of $A$ be i.i.d. symmetric random variables, whereas in our case the entries are not i.i.d., as we require $\brf{A}=1$. However, the i.i.d. assumption in their proof is only used to ensure that $A$, $S_{\sigma^{(1)},\epsilon^{(1)}}\bra{A}$, and $S_{\sigma^{(2)},\epsilon^{(2)}}\bra{A}$ have the same distribution, where 
$\sigma^{(t)}$ for $t=0,1$ are vectors of independent random permutations; $\epsilon^{(t)}$ for $t=0,1$ are matrices of i.i.d. random variables with equal probability of being $\pm 1$; and
$S_{\sigma^{(1)},\epsilon^{(1)}}\bra{A} = \bra{\epsilon^{(1)}_{ij}\cdot a_{i,\sigma^{(1)}_i(j)}}_{i,j}$ and $S_{\sigma^{(2)},\epsilon^{(2)}}\bra{A} = \bra{\epsilon^{(2)}_{ij}\cdot a_{\sigma^{(2)}_j(i), j}}_{i,j}$. Our matrix $A$ satisfies this requirement, and therefore the same theorem holds.}
\end{lemma}

	We thus have that 
	\begin{align*}
	   \frac{1}{mn} \cdot \frac{1}{\Eb\bra{ \brsp{B}^2 }} 
		\geq \frac{1}{mn} \cdot \frac{1}{K^2\bra{ \Eb \bra{ \max_{1\leq i\leq m}\brn{b_{i\bullet}}^2 }  +\Eb \bra{ \max_{1\leq j\leq n}\brn{b_{\bullet j}}^2 } }}.
	\end{align*}

	Hence, we need to upper bound $\Eb \bra{ \max_{1\leq i\leq m}\brn{b_{i\bullet}}^2 }$ and $\Eb \bra{ \max_{1\leq j\leq n}\brn{b_{\bullet j}}^2 }$.
	Let $z\in \Rb^m$ be a vector uniformly sampled from the sphere of $m$-dimension unit ball. 
	Observe that $z \deq [\brn{b_{1\bullet}}, ..., \brn{b_{m\bullet}}]$.
	The following lemma upper bounds the square of the infinity norm of this vector. 
	
	\begin{lemma}
	\label{thm:inifinty-norm-sphere}
	Assume $z=\brb{z_1,z_2,...,z_n}$ is uniformly sampled from the sphere of $n$-dimension unit ball, where $n\geq 2$. Then we have
	$$\Eb\bra{ \max_{1\leq i \leq n} z_i^2  } \leq \frac{4\log(n)}{n-1} .$$
\end{lemma}
	(Proof in \cref{thm:inifinty-norm-sphere-proof})
	
	 Hence, when $m,n\geq 2$, we have
	\begin{align*}
		\Eb \bra{ \max_{1\leq i\leq m}\brn{b_{i\bullet}}^2 }  \leq \frac{4\log \bra{m}}{m-1}
	\end{align*}
	Similarly, we have
	\begin{align*}
		\Eb \bra{ \max_{1\leq j\leq n}\brn{b_{\bullet j}}^2 }  \leq \frac{4\log \bra{n}}{n-1}
	\end{align*}
	Therefore,
	\begin{align*}
		&\;\;\;\;\var\bra{ \frac{a_{ij}}{\brsp{A}} } \\
		&\geq \frac{1}{mn} \cdot \frac{1}{K^2 \bra{ \frac{4\log \bra{m}}{m-1} +  \frac{4\log \bra{n}}{n-1}}}\\
		&\geq \frac{1}{8K^2} \cdot \frac{1}{n\log\bra{m} + m\log\bra{n}} \\
		&\geq \frac{1}{16K^2} \cdot\frac{1}{ \max\brc{m, n} \log\bra{ \min\brc{m,n} }}
	\end{align*}
	which gives the result.

\end{proof}

\subsection{Proof of \cref{thm:inifinty-norm-sphere}}
\label{thm:inifinty-norm-sphere-proof}
\begin{proof}
	\begin{align}
	&\;\;\;\;	\Eb\bra{ \max_{1\leq i \leq n} z_i^2  } \nonumber \\
	&= \int_0^1 \Pb\bra{ \max_{1\leq i \leq n} z_i^2 \geq \delta } d\delta \nonumber\\
	&\leq \int_0^1 \min \brc{1, n\cdot \Pb\bra{ z_1^2 \geq \delta}} d\delta \label{eq:union} 
	\end{align}
	where \eqref{eq:union} follows from the union bound. 
	Next, we use the following lemma to upper bound $\Pb\bra{ z_1^2 \geq \delta}$.

\begin{lemma}
	\label{thm:tail-bound-sphere}
	Assume $z=\brb{z_1,z_2,...,z_n}$ is uniformly sampled from the sphere of $n$-dimension unit ball, where $n\geq 2$. Then for $\frac{1}{n} \leq \delta < 1$ and $\forall i \in [1, n]$, we have
	$$\Pb\bra{ z_i^2 \geq \delta } \leq e^{-\frac{n-1}{2} \cdot \delta + 1}.$$
\end{lemma}
(Proof in \cref{thm:tail-bound-sphere-proof}).
This in turn gives
\begin{align}
	\int_0^1 \min \brc{1, n\cdot \Pb\bra{ z_1^2 \geq \delta}} d\delta  
	&\leq \int_0^{\min\brc{1, \frac{2\log(n) + 2}{n-1}}} 1\cdot  d\delta + \int_{\min\brc{1, \frac{2\log(n) + 2}{n-1}}}^{1} n\cdot e^{-\frac{n-1}{2} \cdot \delta + 1} \cdot d\delta \label{eq:split-integral}\\
	&\leq \left\{\begin{matrix}
	1 & (n \leq 6) \\
	\frac{2\log\bra{n} + 2}{n-1}  - \frac{2n}{n-1} e^{-\frac{n-3}{2}} + \frac{2}{n-1} & (n \geq 7) 
	\end{matrix}
	\right. \nonumber \\
	&\leq \frac{4\log(n)}{n-1} \nonumber
\end{align}
where \cref{eq:split-integral} follows from \cref{thm:tail-bound-sphere}.

\end{proof}

\subsection{Proof of \cref{thm:tail-bound-sphere}}
\label{thm:tail-bound-sphere-proof}

\begin{proof}
	Due to the symmetry of $z_i$, we only need to prove the inequality for $i=1$ case.
	Let $x = [x_1,...,x_n] \sim \Nc\bra{ \bm{0}, I_n }$, where $I_n$ is the identity matrix in $n$ dimension. We know that $\frac{x_1^2}{\sum_{i=1}^{n}x_i^2} \deq z_1^2$. Therefore, we have
	\begin{align*}
		\Pb\bra{ z_1^2 \geq \delta }
		= \Pb\bra{ \frac{x_1^2}{\sum_{j=1}^{n}x_j^2}  \geq \delta }
		= \Pb\bra{ \frac{x_1^2}{\bra{\sum_{i=2}^{n} x_i^2} / (n-1)}  \geq \frac{\bra{n-1} \delta}{1-\delta}} \;\;.
	\end{align*}
	Note that $x_1^2$ and $\sum_{i=2}^{n} x_i^2$ are two independent chi-squared random variables, therefore, we know that $\frac{x_1^2}{\bra{\sum_{i=2}^{n} x_i^2} / (n-1) }\sim F(1, n-1)$, where $F$ denotes the central F-distribution. Therefore, 
	\begin{align*}
		\Pb\bra{ \frac{x_1^2}{\bra{\sum_{i=2}^{n} x_i^2} / (n-1)}  \geq \frac{\bra{n-1} \delta}{1-\delta}} 
		&= 1 - I_{\delta}\bra{ \frac{1}{2}, \frac{n-1}{2} } \\
		&= I_{1-\delta}\bra{ \frac{n-1}{2}, \frac{1}{2} } \\
		&= \frac{B_{1-\delta}\bra{\frac{n-1}{2}, \frac{1}{2}}}  {B\bra{\frac{n-1}{2}, \frac{1}{2}}} \numberthis \label{eq:beta-ratio}\;\;,
	\end{align*}
	where $I_x(a,b)$ is the regularized incomplete beta function, $B_x(a, b)$ is the incomplete beta function, and $B(a, b)$ is beta function.
	
	For the ease of computation, we take the $\log$ of \cref{eq:beta-ratio}. The numerator gives
	\begin{align*}
		&\;\;\;\;\log\bra{ B_{1-\delta}\bra{\frac{n-1}{2}, \frac{1}{2}} }\\
		&= \log\bra{  \frac{\bra{1-\delta}^{(n-1)/2}}{(n-1)/2} \hyg{ \frac{n-1}{2}, \frac{1}{2}; \frac{n+1}{2}; 1-\delta } }  \\
		&= \frac{n-1}{2} \log\bra{1-\delta} - 
		\log(n-1) + \log\bra{\hyg{ \frac{n-1}{2}, \frac{1}{2}; \frac{n+1}{2}; 1-\delta }} + \log(2) \;\;,\numberthis \label{eq:log-in-co-beta}
	\end{align*}
	where $\hyg{\cdot}$ is the hypergeometric function. Let $\poc{q}{i} = \left\{ \begin{matrix}
	1 & (i = 0) \\
	q(q+1)\ldots (q + i - 1) &  (i > 0)
	\end{matrix} \right.\;\;,$ we have
	\begin{align*}
		&\;\;\;\; \hyg{ \frac{n-1}{2}, \frac{1}{2}; \frac{n+1}{2}; 1-\delta }\\
		&=\sum_{i=0}^{\infty}  \frac{\poc{\frac{n-1}{2}}{i} \poc{\frac{1}{2}}{i}  \bra{1-\delta}^i} {  \poc{\frac{n+1}{2}}{i}  \cdot i!}\\
		&\leq \sum_{i=0}^{\infty}  \frac{ \poc{\frac{1}{2}}{i}  \bra{1-\delta}^i} {   \cdot i!}\\
		&=\delta^{-\frac{1}{2}} \numberthis \label{eq:delta}
	\end{align*}
	Substituting it into \cref{eq:log-in-co-beta} gives
	\begin{align*}
		\log\bra{ B_{1-\delta}\bra{\frac{n-1}{2}, \frac{1}{2}} }
		\leq \frac{n-1}{2} \log\bra{1-\delta} - 
		\log\bra{n-1} -\frac{1}{2}\log\bra{\delta} + \log(2) \numberthis \label{eq:log-in-co-beta2}\;\;.
	\end{align*}
	
	The $\log$ of the denominator of \eqref{eq:beta-ratio} is
	\begin{align*}
	&\;\;\;\;	\log \bra{ B\bra{\frac{n-1}{2}, \frac{1}{2}} } \\
	&= \log\bra{ \frac{\Gamma\bra{\frac{n-1}{2}} \Gamma\bra{\frac{1}{2}} }{\Gamma\bra{ \frac{n}{2} }} }\\
	&\geq \log\bra{ \sqrt{\pi}\cdot \bra{\frac{n+1}{2}}^{-\frac{1}{2}} }\\
	&= -\frac{1}{2} \log(n + 1) + \frac{1}{2} \log(2) + \frac{1}{2} \log(\pi) \numberthis \label{eq:log-beta}\;\;.
	\end{align*}
	where $\Gamma$ denotes the Gamma function and we use the Gautschi's inequality: $\frac{\Gamma\bra{x+1}}{\Gamma\bra{x+\frac{1}{2}}} < \bra{x+1}^{\frac{1}{2}}$ for positive real number $x$.
	
	Combining \cref{eq:beta-ratio}, \cref{eq:log-in-co-beta2}, and \cref{eq:log-beta} we get
	\begin{align*}
		&\;\;\;\; \log\bra{\Pb\bra{ \frac{x_1^2}{\bra{\sum_{i=2}^{n} x_i^2} / (n-1)}  \geq \frac{\bra{n-1} \delta}{1-\delta}} } \\
		&\leq \frac{n-1}{2} \log\bra{1-\delta} - 
		\log\bra{n-1} +\frac{1}{2} \log(n + 1) -\frac{1}{2}\log\bra{\delta} + \frac{1}{2}\log(2/\pi) \\
		&\leq  \frac{n-1}{2} \log\bra{1-\delta} - \frac{1}{2} \log(n-1) -\frac{1}{2}\log(\delta) + \frac{1}{2} \log(6/\pi)\\
		&\leq \frac{n-1}{2} \log\bra{1-\delta}  - \frac{1}{2}\log\bra{\frac{n-1}{n}}+ \frac{1}{2} \log(6/\pi) \\
		&\leq \frac{n-1}{2} \log\bra{1-\delta}  + \frac{1}{2} \log\frac{12}{\pi}\\
		&\leq -\frac{n-1}{2} \cdot \delta + 1
	\end{align*}
	Therefore, we have
	\begin{align*}
		\Pb\bra{ z_1^2 \geq \delta } \leq e^{-\frac{n-1}{2} \cdot \delta + 1}
	\end{align*}
\end{proof}

\section{Proof of \cref{thm:ours-variance}}
\label{app:ours-variance}
\begin{proof}
	Let $s_w = \cin\cout k_w k_h$.
	Since $w_{ij}$ are symmetric random variables, we know $\Eb\bra{\frac{w_{ij}}{\sigma_w}} = 0$. Therefore, we have
	\begin{align*}
	\var\bra{ \frac{w_{ij}}{\sigma_w} }
	= \Eb\bra{ \frac{w_{ij}^2}{\sigma_w^2} }
	= \frac{1}{s_w}\cdot  \Eb\bra{\frac{\sum_{i=1}^{m}\sum_{j=1}^{n} w_{ij}^2}{ \sigma_w^2} }
	= \frac{1}{s_w}\cdot \Eb\bra{ \frac{\brf{w}^2}{ \sigma_w^2}}
	\end{align*}
	
	Note that 
	\begin{align*}
	 \frac{1}{s_w}\cdot \Eb\bra{ \frac{\brf{w}^2}{ \sigma_w^2}} \in \bigg[ &\frac{2}{s_w}\cdot \Eb\bra{  \frac{\brf{w}^2}{\brspwshape{w}^2 + \brspwshapet{w}^2} },\\
	 &\frac{4}{s_w}\cdot \Eb\bra{  \frac{\brf{w}^2}{\brspwshape{w}^2 + \brspwshapet{w}^2} } \bigg]\;\;.
	\end{align*}
	
	Assume the singular values of $\wwshape{w}$ are $\sigma_1\geq \sigma_2\geq \ldots \geq \sigma_{\cout}$, and the singular values of $\wwshapet{w}$ are $\sigma'_1\geq \sigma'_2\geq \ldots \geq \sigma'_{\cin}$. We have 
	\begin{align*}
	&\;\;\;\;\frac{4}{s_w}\cdot \Eb\bra{  \frac{\brf{w}^2}{\brspwshape{w}^2 + \brspwshapet{w}^2} }\\
	&= \frac{4}{s_w} \cdot \Eb\bra{ \frac{1}{2}\cdot \frac{\sum_{i=1}^{\cout}  \sigma_i^2}{\sigma_1^2}   + \frac{1}{2}\cdot \frac{\sum_{i=1}^{\cin}  {\sigma'}_i^2}{{\sigma'}_1^2} } 
	\leq\frac{2\bra{ \cout + \cin}}{s_w} 
	= \frac{2}{\cin k_w k_h + \cout k_w k_h} \;\;,
	\end{align*}
	which gives the desired upper bound.
	
	As for the lower bound, observe that
	\begin{align*}
	&\;\;\;\;\frac{2}{s_w}\cdot \Eb\bra{  \frac{\brf{w}^2}{\brspwshape{w}^2 + \brspwshapet{w}^2} }\\
	&= \frac{2}{s_w}\cdot \Eb\bra{  \frac{1}{\brsp{\frac{\wwshape{w}}{\brf{w}}}^2 + \brsp{\frac{\wwshapet{w}}{\brf{w}}}^2} }\\
	&\geq \frac{2}{s_w}\cdot  \frac{1}{\Eb\bra{\brsp{\frac{\wwshape{w}}{\brf{w}}}^2} + \Eb\bra{\brsp{\frac{\wwshapet{w}}{\brf{w}}}^2}} \\
	\end{align*}
	Then we can follow the same approach in \cref{app:sn-variance} for bounding $\Eb\bra{\brsp{\frac{\wwshape{w}}{\brf{w}}}^2}$ and 
	$\Eb\bra{\brsp{\frac{\wwshapet{w}}{\brf{w}}}^2}$,
	which gives the desired lower bound.
\end{proof}

\section{Datasets and Metrics}
\label{app:dataset-metric}

\subsection{Datasets}

\paragraph{\mnist{} \cite{lecun-mnisthandwrittendigit-2010}}
We use the training set for our experiments, which contains 60000 images of handwritten digits of shape $28\times 28\times 1$. The pixels values are normalized to $[0,1]$ before feeding to the discriminators. 

\paragraph{\cifar{} \cite{krizhevsky2009learning}}
We use the training set for our experiments, which contains 50000 images of shape $32\times 32\times 3$. The pixels values are normalized to $[-1,1]$ before feeding to the discriminators. 

\paragraph{\stl{} \cite{coates2011analysis}}
We use the unlabeled set for our experiments, which contains 100000 images of shape $96\times 96\times 3$. Following \cite{miyato2018spectral}, we resize the images to $48\times 48\times 3$ for training. The pixels values are normalized to $[-1,1]$ before feeding to the discriminators. 

\paragraph{\celeba{} \cite{liu2015faceattributes}}
This dataset contains 202599 images. For each image, we crop the center $128\times 128$, and resize it to $64\times 64\times 3$ for training. The pixels values are normalized to $[-1,1]$ before feeding to the discriminators.

\paragraph{\imagenet{} (\ilsvrc{}) \cite{ILSVRC15}}
The dataset contains 1281167 images. Following \cite{miyato2018spectral}, for each images, we crop the central square of the images according to min(width, height), and then reshape it to $128\times 128\times 3$ for training. The pixels values are normalized to $[-1,1]$ before feeding to the discriminators. 

\subsection{Metrics}

\paragraph{Inception score \cite{salimans2016improved}} 
Following \cite{miyato2018spectral}, we use 50000 generated images and split them into 10 sets for computing the score.

\paragraph{FID \cite{heusel2017gans}}
Following \cite{miyato2018spectral}, we use 5000 real images and 10000 generated images for computing the score.

\section{Gradient Explosion and Vanishing in GANs}
\label{app:exp-vani}

\subsection{Results}
To illustrate that gradient explosion and vanishing are closely related to the instability in GANs, we trained a WGAN \cite{gulrajani2017improved} on the \cifar{} dataset with different hyper-parameters leading to stable training, exploding gradients, and vanishing gradients over 40,000 training iterations (more experimental details in \cref{app:exp-vani-settings}). 
\cref{fig:gradient_vanishing}  shows the resulting inception scores for each of these runs, and \cref{fig:grad_size} shows the corresponding magnitudes of the gradients over the course of training. 
Note that the stable run has improved sample quality and stable gradients throughout training.
This phenomenon has also been observed in prior literature \cite{arjovsky2017principled,brock2018large}.
We will demonstrate that by controlling these gradients, SN (and \snw{} in particular) is able to achieve more stable training and better sample quality.

\begin{figure}[!htb]
	\centering
	\begin{minipage}{.45\textwidth}
		\centering
		\includegraphics[width=\linewidth]{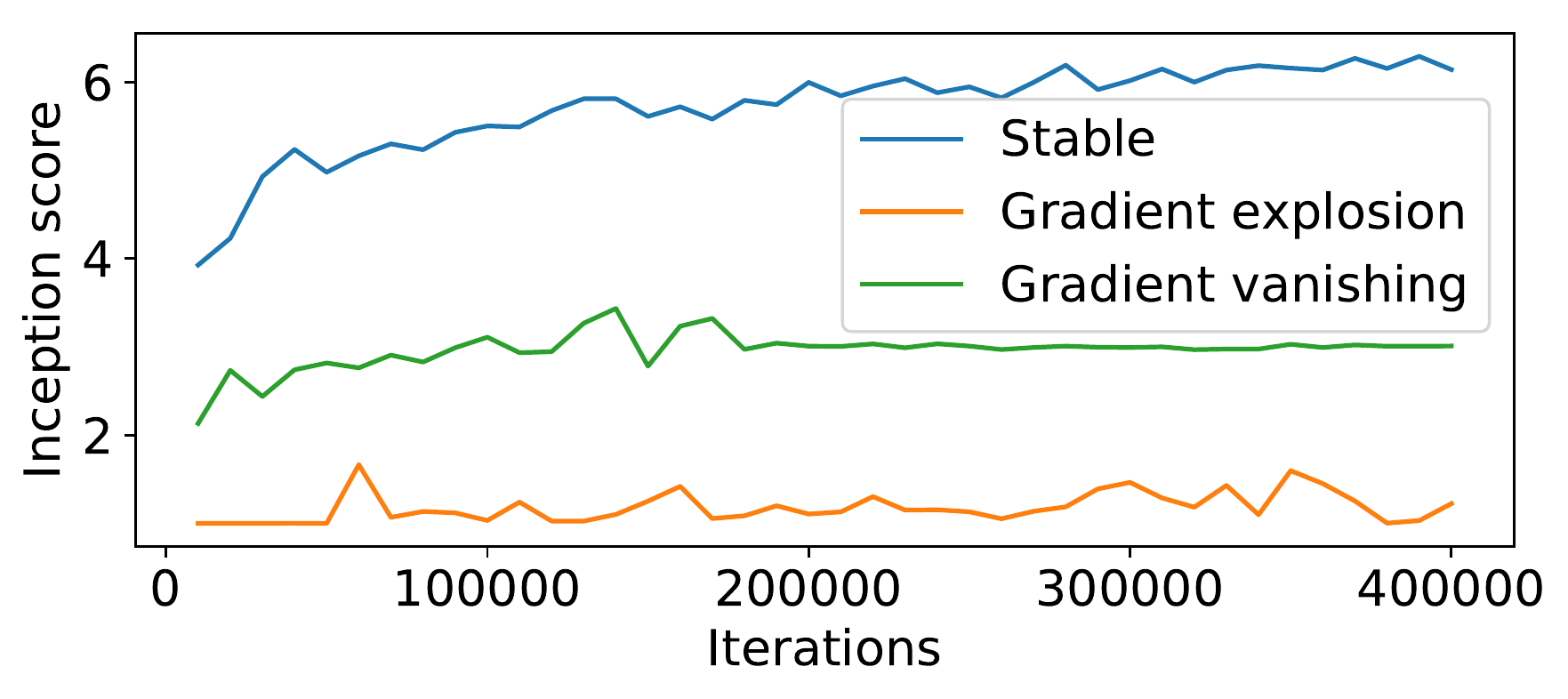}
		\caption{Inception score over the course of training.
			The ``gradient vanishing" inception score plateaus as training is stalled.}
		\label{fig:gradient_vanishing}
	\end{minipage}%
	~~~
	\begin{minipage}{0.47\textwidth}
		\centering
		\includegraphics[width=\linewidth]{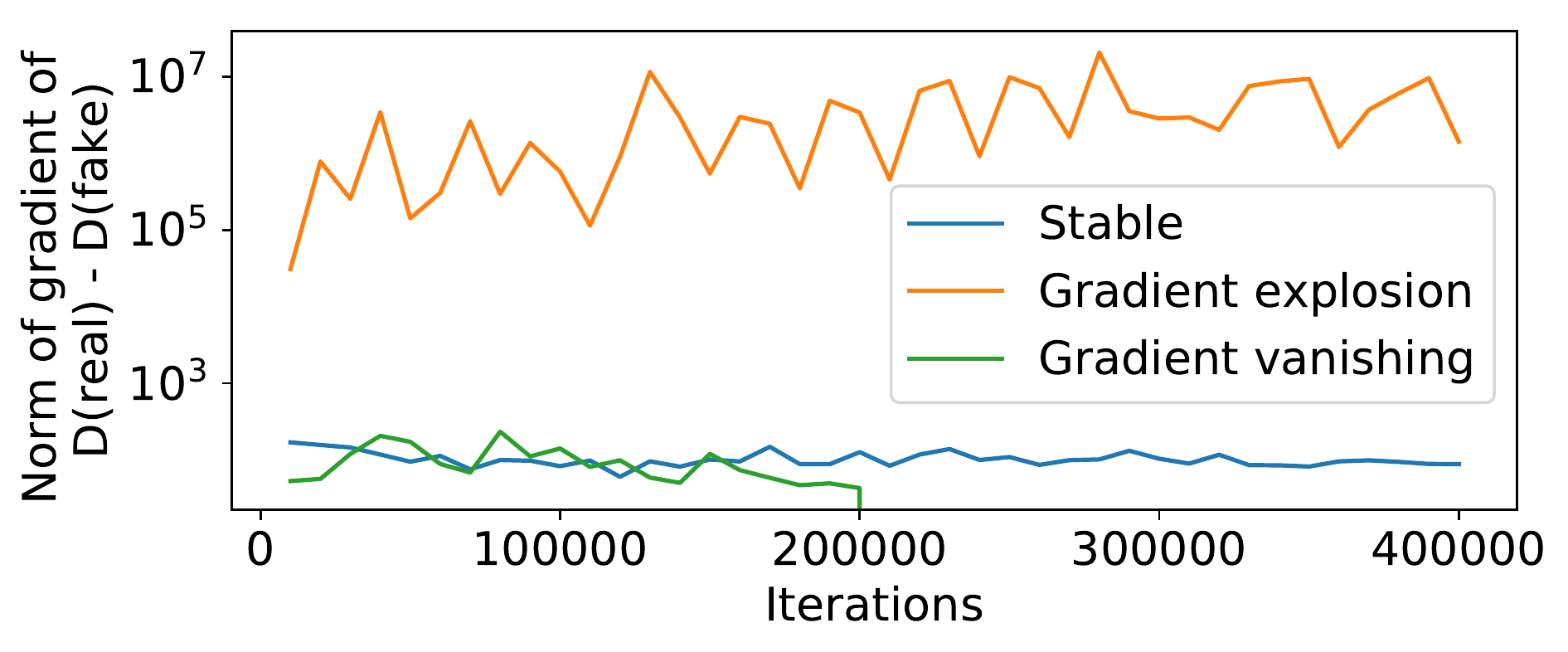}
		\caption{Norm of gradient with respect to parameters during training. The vanishing gradient collapses after 200k iterations.}
		\label{fig:grad_size}
	\end{minipage}
\end{figure}

\subsection{Experimental Details}
\label{app:exp-vani-settings}

The network architectures are shown in \cref{tbl:gen-network-cifar10-stl10,tbl:dis-network-cifar10-stl10}. The dataset is \cifar{}. 
All experiments are run for 400k iterations. Batch size is 64. The optimizer is Adam.
Let $\lambda$ be the WGAN's gradient penalty weight \cite{gulrajani2017improved}. For the stable run, $\alpha_g=0.0001, \alpha_d=0.0002, \beta_1=0.5, \beta_2=0.999, \lambda=10, n_{dis}=1$. For the gradient explosion run, $\alpha_g=0.001, \alpha_d=0.001, \beta_1=0.5, \beta_2=0.999, \lambda=10, n_{dis}=1$. For the gradient vanishing run, $\alpha_g=0.001, \alpha_d=0.001, \beta_1=0.5, \beta_2=0.999, \lambda=50, n_{dis}=1$, and the activation functions in the discriminator are changed from leaky ReLU to ReLU. 

\begin{table}
	\centering
	\begin{tabular}{c}
		\toprule
		$z\in \Rb^{128}\sim \Nc(0, I)$\\\hline
		Fully connected ($M_g \times M_g\times 512$). BN. ReLU.\\\hline
		Deconvolution ($c=256, k=4, s=2$). BN. ReLU.\\\hline
		Deconvolution ($c=128, k=4, s=2$). BN. ReLU.\\\hline
		Deconvolution ($c=64, k=4, s=2$). BN. ReLU.\\\hline
		Deconvolution ($c=3, k=3, s=1$). Tanh.\\
		\bottomrule
	\end{tabular}
	\caption{Generator network architectures for \cifar{}, \stl{}, and \celeba{} experiments (from \cite{miyato2018spectral}). For \cifar{}, $M_g=4$. For \stl{}, $M_g=6$. For \celeba{}, $M_g=8$. BN stands for batch normalization. $c$ stands for number of channels. $k$ stands for kernel size. $s$ stands for stride.}
	\label{tbl:gen-network-cifar10-stl10}
\end{table}

\begin{table}
	\centering
	\begin{tabular}{c}
		\toprule
		$x\in \Rb^{M\times M\times 3}$\\\hline
		Convolution ($c=64, k=3, s=1$). Leaky ReLU (0.1).\\\hline
		Convolution ($c=64, k=4, s=2$). Leaky ReLU (0.1).\\\hline
		Convolution ($c=128, k=3, s=1$). Leaky ReLU (0.1).\\\hline
		Convolution ($c=128, k=4, s=2$). Leaky ReLU (0.1).\\\hline
		Convolution ($c=256, k=3, s=1$). Leaky ReLU (0.1).\\\hline
		Convolution ($c=256, k=4, s=2$). Leaky ReLU (0.1).\\\hline
		Convolution ($c=512, k=3, s=1$). Leaky ReLU (0.1).\\\hline
		Fully connected (1). \\
		\bottomrule
	\end{tabular}
	\caption{Discriminator network architectures for \cifar{}, \stl{}, and \celeba{} experiments (from \cite{miyato2018spectral}). For \cifar{}, $M=32$. For \stl{}, $M=48$. For \celeba{}, $M=64$. $c$ stands for number of channels. $k$ stands for kernel size. $s$ stands for stride.}
	\label{tbl:dis-network-cifar10-stl10}
\end{table}

\section{Experimental Details and Additional Results on Gradient Norms}
\label{app:gradient-norm}

\subsection{Experimental Details}
For the \mnist{} experiment, the network architectures are shown in \cref{tbl:gen-network-mnist,tbl:dis-network-mnist}.
All experiments are run for 100 epochs. Batch size is 64. The optimizer is Adam.
$\alpha_g=0.001, \alpha_d=0.001, \beta_1=0.5, \beta_2=0.999, n_{dis}=1$. 

For the \cifar{} experiment, , the network architectures are shown in \cref{tbl:gen-network-cifar10-stl10,tbl:dis-network-cifar10-stl10}.
All experiments are run for 400k iterations. Batch size is 64. The optimizer is Adam.
$\alpha_g=0.0001, \alpha_d=0.0001, \beta_1=0.5, \beta_2=0.999, n_{dis}=1$. 

Let $\lambda$ be the WGAN's gradient penalty weight \cite{gulrajani2017improved}. For the runs without SN, $\lambda=10$. For the runs with SN, we use the strict SN implementation \cite{farnia2018generalizable} in order to verifying the theoretical results (the popular SN implementation \cite{miyato2018spectral} only gives a loose bound on the actual spectral norm of layers, see \cref{sec:vanishing}). Since it already ensures that the Lipschitz constant of the discriminator is no more than 1, we discard the gradient penalty loss from training.

For all the results, the gradient norm only considers the weights and excludes the biases (if exist), so as to be consistent with the theoretical analysis.

\begin{table}
	\centering
	\begin{tabular}{c}
		\toprule
		$z\in \Rb^{100}\sim \uniform(-1, 1)$\\\hline
		Fully connected ($7 \times 7\times 128$). Leaky ReLU (0.2). BN.\\\hline
		Deconvolution ($c=64, k=5, s=2$). Leaky ReLU (0.2). BN.\\\hline
		Deconvolution ($c=1, k=5, s=2$). Sigmoid.\\
		\bottomrule
	\end{tabular}
	\caption{Generator network architectures for \mnist{} experiments. BN stands for batch normalization. $c$ stands for number of channels. $k$ stands for kernel size. $s$ stands for stride.}
	\label{tbl:gen-network-mnist}
\end{table}

\begin{table}
	\centering
	\begin{tabular}{c}
		\toprule
		$x\in \Rb^{28\times 28\times 1}$\\\hline
		Convolution ($c=64, k=5, s=2$, no bias). Leaky ReLU (0.2).\\\hline
		Convolution ($c=128, k=5, s=2$, no bias). Leaky ReLU (0.2).\\\hline
		Convolution ($c=256, k=5, s=2$, no bias). Leaky ReLU (0.2).\\\hline
		Fully connected (1, no bias). \\
		\bottomrule
	\end{tabular}
	\caption{Discriminator network architectures for \mnist{} experiments. $c$ stands for number of channels. $k$ stands for kernel size. $s$ stands for stride.}
	\label{tbl:dis-network-mnist}
\end{table}

\subsection{Additional Results}
\cref{fig:gnorm-mnist-epoch50,fig:gnorm-cifar10-epoch10000} show the gradient norms of each discriminator layer in \mnist{} and \cifar{}. Despite the difference on the network architecture and dataset, we see the similar phenomenon: when training without SN, some layers have extremely large gradient norms, which causes the overall gradient norm to be large; when training with SN, the gradient norms are much smaller and are similar across different layers.

\begin{figure}
	\centering
	\begin{minipage}{0.45\linewidth}
		\centering
		\includegraphics[width=\linewidth]{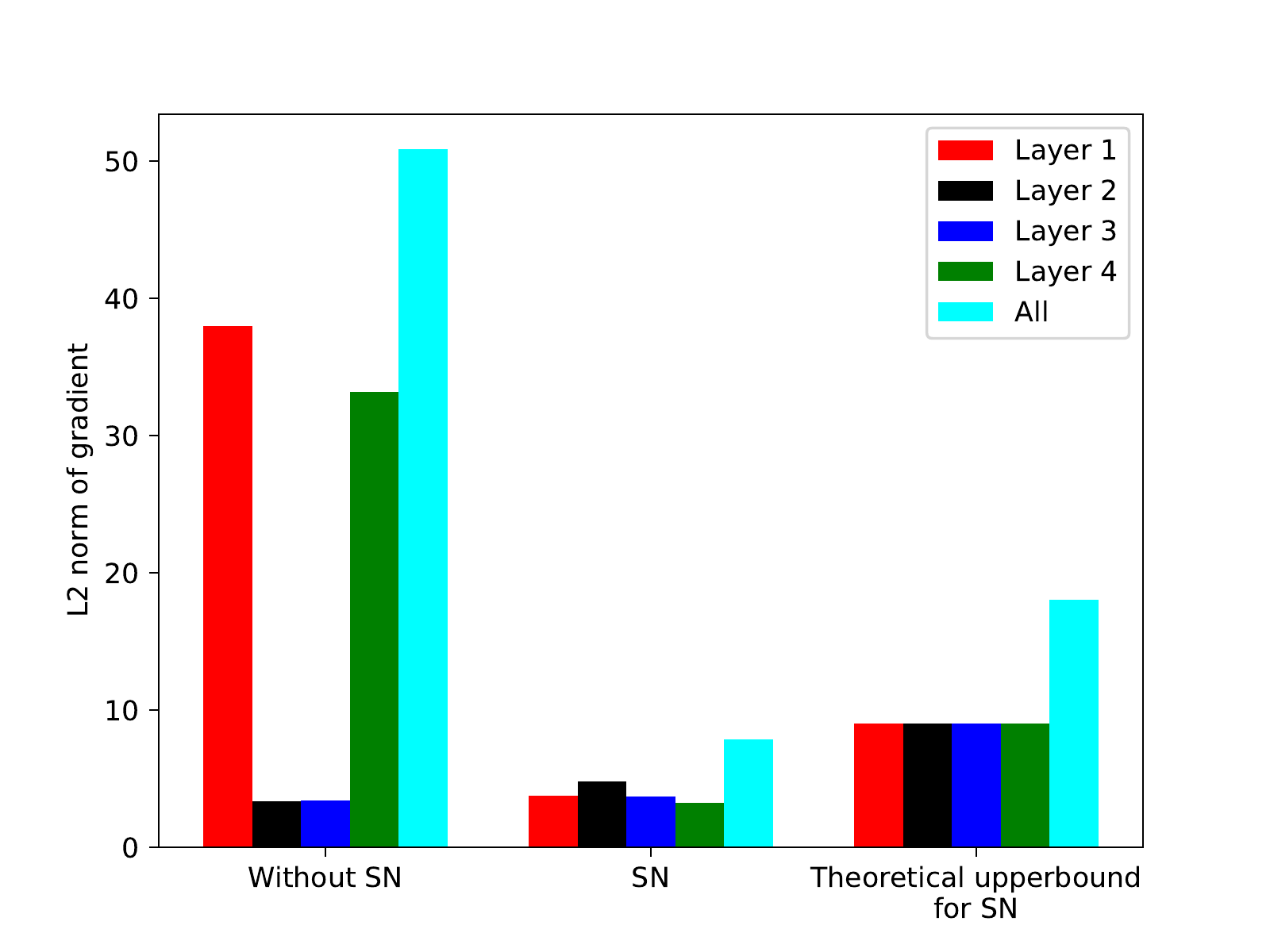}
		\caption{Gradient norms of each discriminator layer in \mnist{} at epoch 50. }
		\label{fig:gnorm-mnist-epoch50}
	\end{minipage}
	~~~
	\begin{minipage}{0.45\linewidth}
		\centering
		\includegraphics[width=\linewidth]{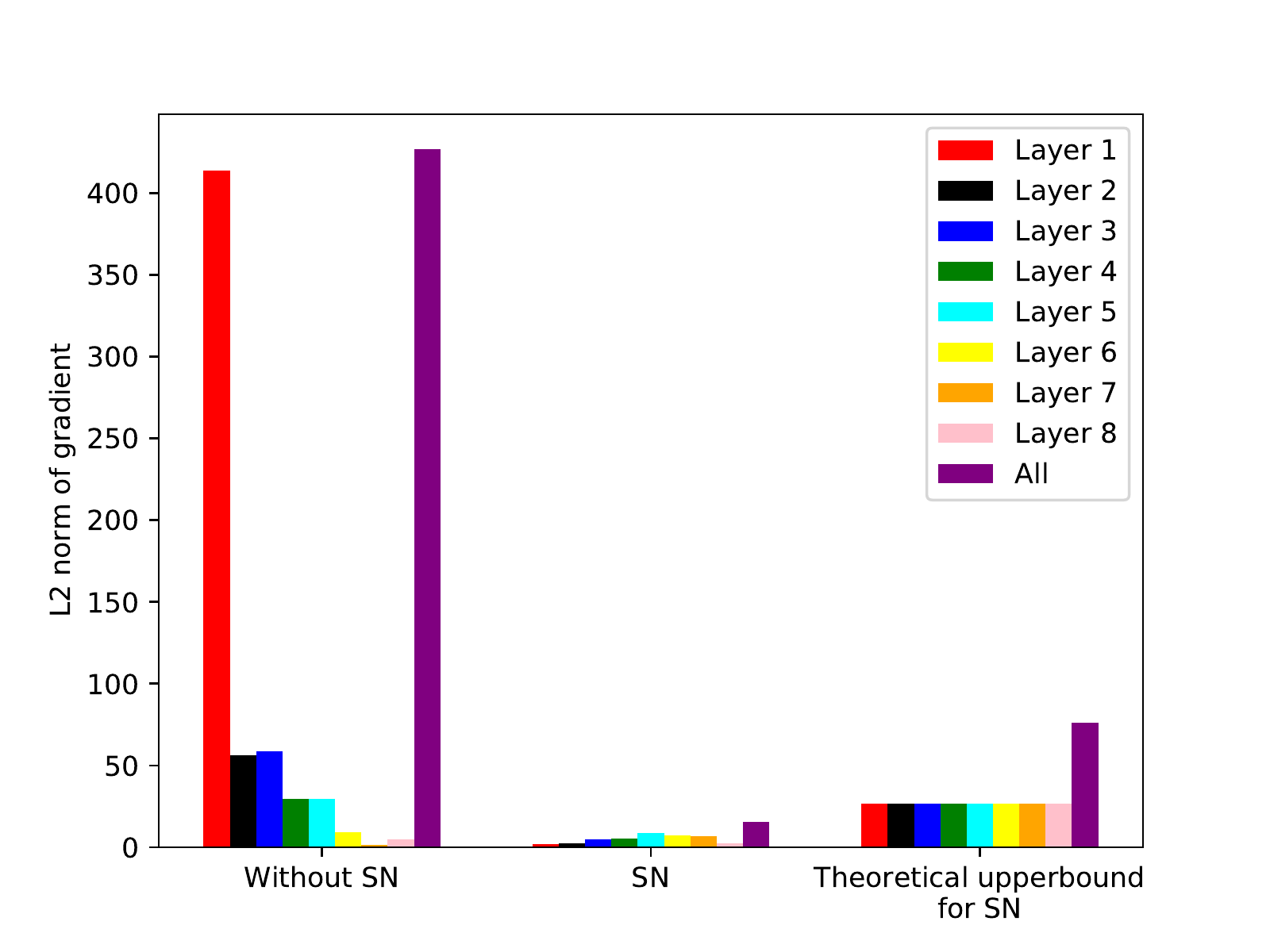}
		\caption{Gradient norms of each discriminator layer in \cifar{} at iteration 10000. }
		\label{fig:gnorm-cifar10-epoch10000}
	\end{minipage}
\end{figure}

\section{Experimental Details and Additional Results for Confirming \cref{eq:setD}}
\label{app:result_setd}

\subsection{Experimental Details}
For the \mnist{} experiment, the network architectures are shown in \cref{tbl:gen-network-mnist,tbl:dis-network-mnist}.
All experiments are run for 100 epochs. Batch size is 64. The optimizer is Adam.
$\alpha_g=0.001, \alpha_d=0.001, \beta_1=0.5, \beta_2=0.999, n_{dis}=1$. We use WGAN loss with the strict SN implementation \cite{farnia2018generalizable}.  Since it already ensures that the Lipschitz constant of the discriminator is no more than 1, we discard the gradient penalty loss from training. The random scaling are selected in a way the geometric mean of spectral norms of all layers equals 1. 

For the \cifar{} and \stl{} experiments , the network architectures are shown in \cref{tbl:gen-network-cifar10-stl10,tbl:dis-network-cifar10-stl10}.
All experiments are run for 400k iterations. Batch size is 64. The optimizer is Adam.
$\alpha_g=0.0001, \alpha_d=0.0001, \beta_1=0.5, \beta_2=0.999, n_{dis}=1$. We use hinge loss \cite{miyato2018spectral} with the strict SN implementation \cite{farnia2018generalizable}. The random scaling are selected in a way the geometric mean of spectral norms of all layers equals 1.75, which avoids the gradient vanishing problem as seen in \cref{sec:vanishing}.

\subsection{Additional Results}
\cref{fig:gratio-cifar10,fig:gratio-stl10} show the ratios of the gradient norms at each layer and the inverse ratios of the spectral norms in \cifar{} and \stl{}. Generally, we see that the most of the points are near the diagonal line, which means that the assumption in \cref{eq:setD} is reasonably true in practice. However, we note that the last layer (layer 8) somehow has slightly smaller gradient, as the points of ``layer 8 / layer 1'' are slightly lower than the diagonal line. This could result from the fact that layer 8 is a fully connected layer whereas all other layers are convolutional layers. We defer the more detailed analysis of this phenomenon to future work.

\begin{figure}
	\centering
	\begin{minipage}{0.45\linewidth}
		\centering
		\includegraphics[width=\linewidth]{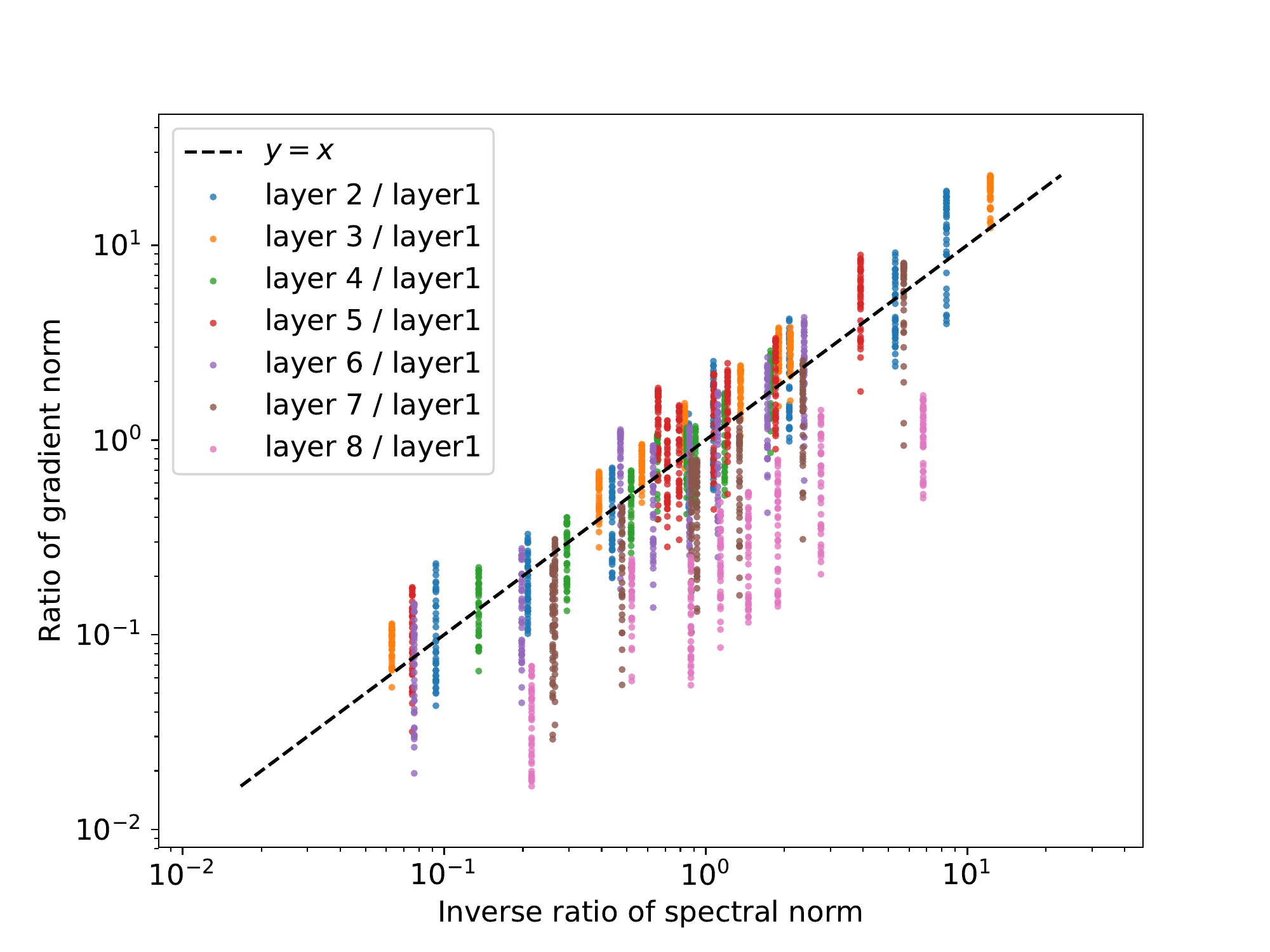}
		\caption{Ratio of gradient norm v.s. inverse ratio of spectral norm in \cifar{}.}
		\label{fig:gratio-cifar10}
	\end{minipage}
	~~~
	\begin{minipage}{0.45\linewidth}
		\centering
		\includegraphics[width=\linewidth]{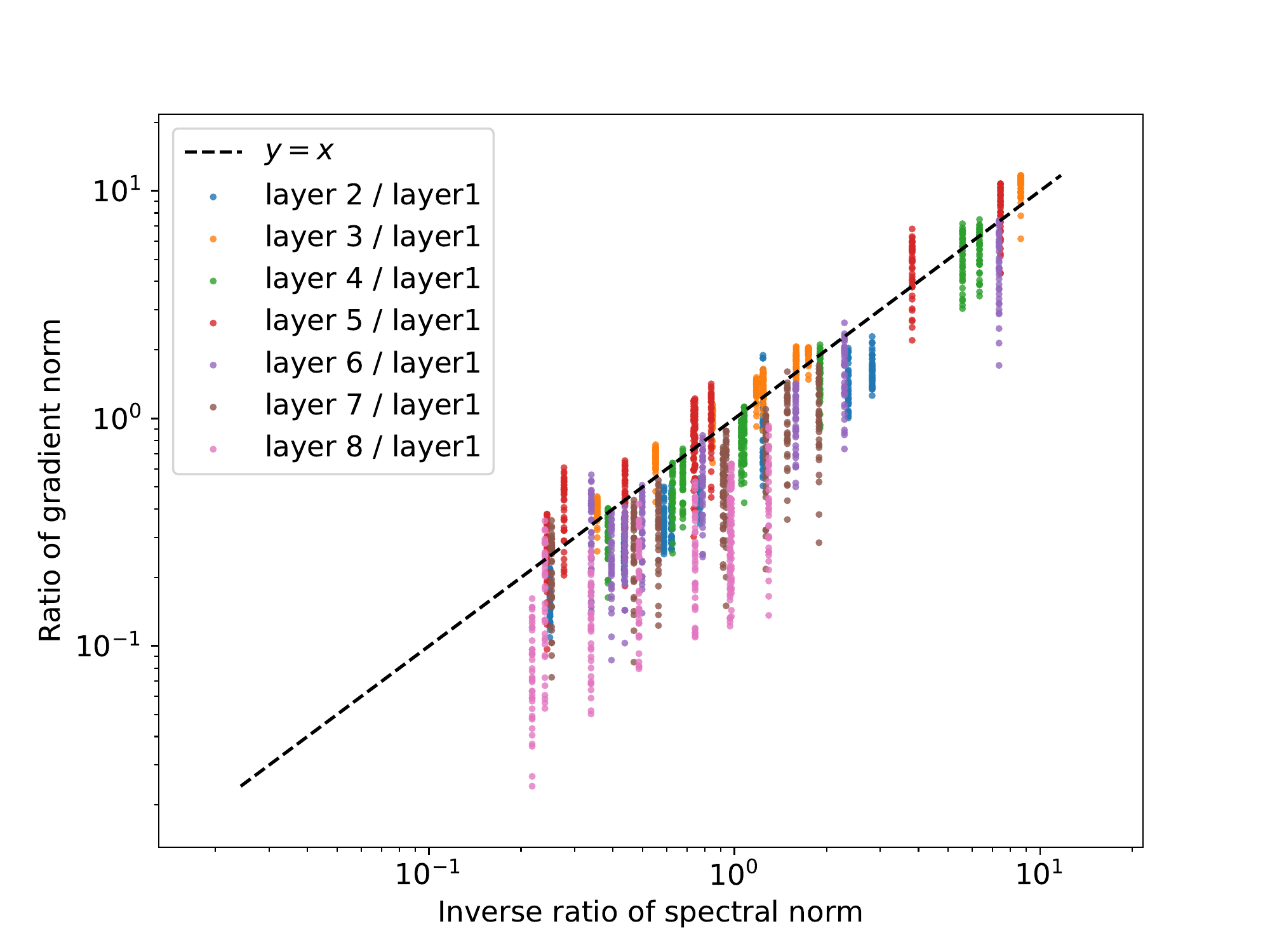}
		\caption{Ratio of gradient norm v.s. inverse ratio of spectral norm in \stl{}.}
		\label{fig:gratio-stl10}
	\end{minipage}
\end{figure}

\section{Experimental Details and Additional Results on Vanishing Gradient}
\label{app:gradient-vani}

\subsection{Experimental Details}

The network architectures are shown in \cref{tbl:gen-network-cifar10-stl10,tbl:dis-network-cifar10-stl10}. The dataset is \cifar{}. 
All experiments are run for 400k iterations. Batch size is 64. The optimizer is Adam.
$\alpha_g=0.0001, \alpha_d=0.0001, \beta_1=0.5, \beta_2=0.999, n_{dis}=1$.
We use hinge loss \cite{miyato2018spectral}. 

\subsection{Parameter Variance With and Without SN}
\label{app:para-var-withwithoutsn}
\cref{fig:paravar-withoutsn-cifar10,fig:paravar-sn-cifar10} show the parameter variance of each layer without and with SN. Note that \cref{fig:paravar-sn-cifar10} is just collecting the empirical lines in \cref{fig:lecun_training} for the ease of comparison here. \cref{fig:gnorm-withwithoutsn-cifar10,fig:inception-withwithoutsn-cifar10} show the gradient norm and inception score.

We can see that when training with SN, the parameter variance is stable throughout training (\cref{fig:paravar-sn-cifar10}), and the magnitude of gradient is also stable (\cref{fig:gnorm-withwithoutsn-cifar10}) . However, when training without SN, the parameter variance tends to increase throughout training (\cref{fig:paravar-withoutsn-cifar10}), which causes a quick decrease in the magnitude of  gradient in the begining of training (\cref{fig:gnorm-withwithoutsn-cifar10}) because of the saturation of hinge loss (\cref{sec:vanishing}). Because SN promotes the stability of the variance and gradient throughout training, we see that SN improves the sample quality significantly (\cref{fig:inception-withwithoutsn-cifar10}).

\begin{figure}[t]
	\centering
	\begin{minipage}{.47\textwidth}
		\centering
		\includegraphics[width=\linewidth]{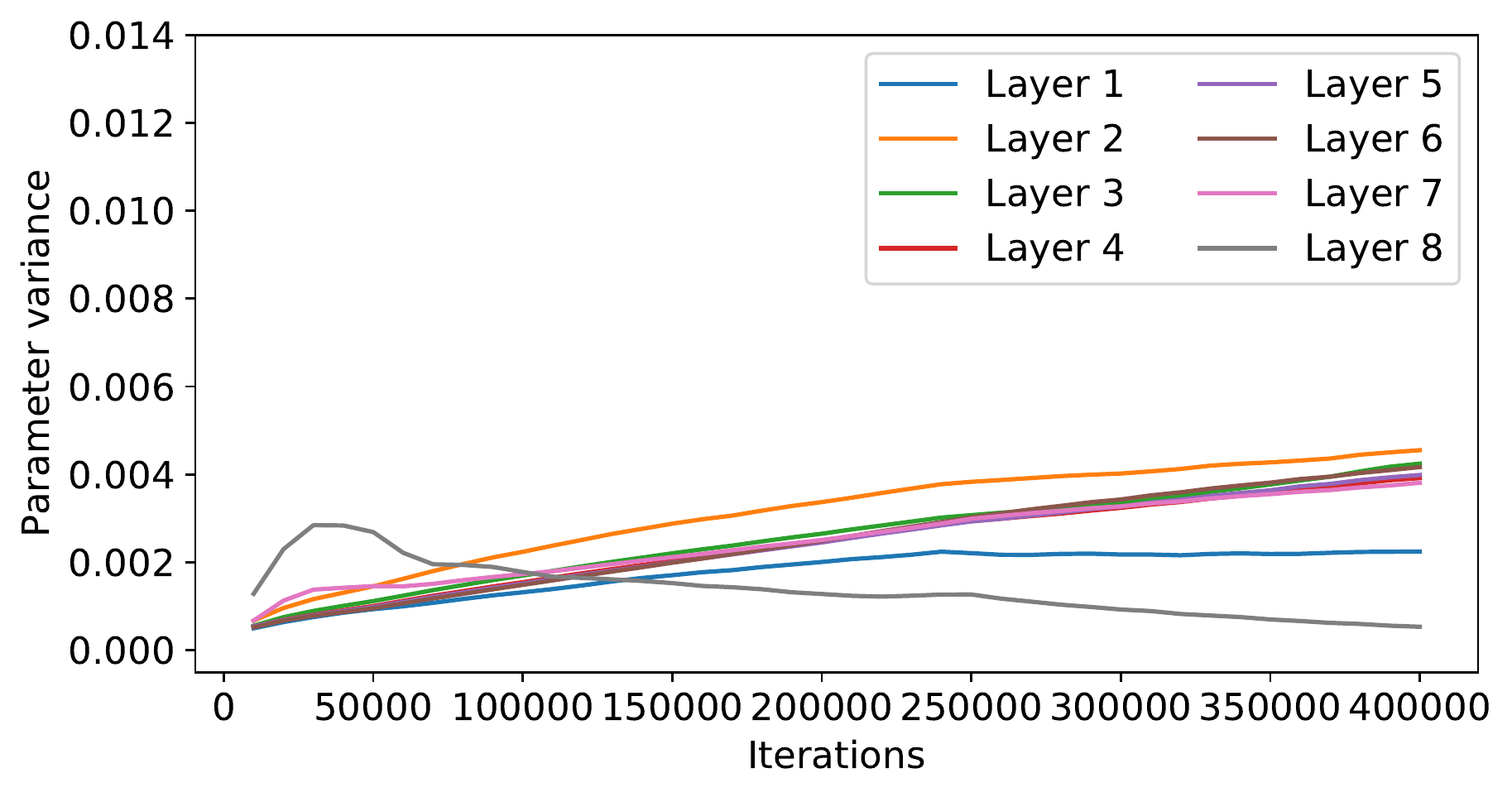}
		\caption{Parameter variance without SN in \cifar{}.}
		\label{fig:paravar-withoutsn-cifar10}
	\end{minipage}%
	~~
	\begin{minipage}{0.47\textwidth}
		\centering
		\includegraphics[width=\linewidth]{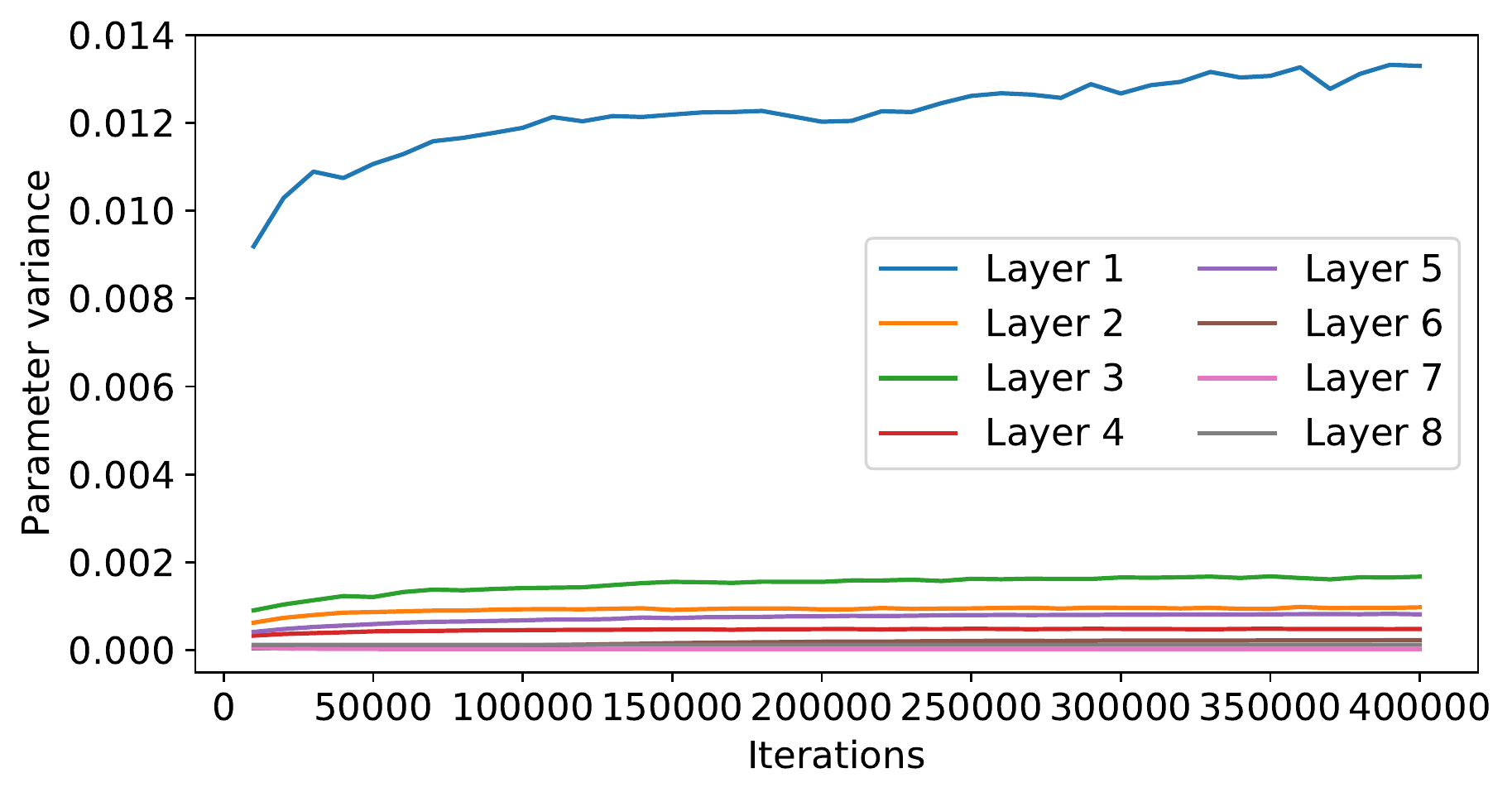}
		\caption{Parameter variance with SN in \cifar{}.}
		\label{fig:paravar-sn-cifar10}
	\end{minipage}
\end{figure}

\begin{figure}[t]
	\centering
	\begin{minipage}{.47\textwidth}
		\centering
		\includegraphics[width=\linewidth]{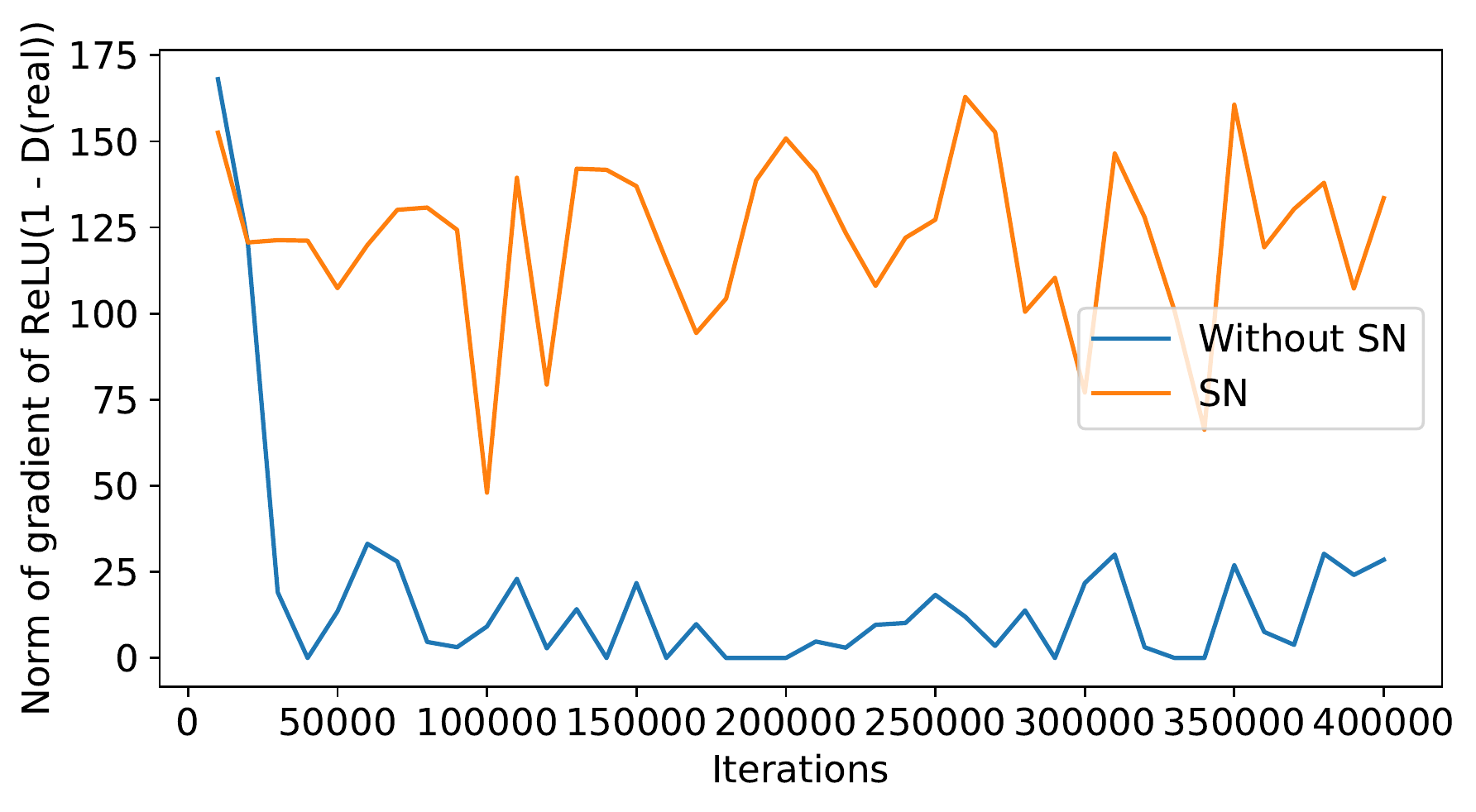}
		\caption{Gradient norm with and without SN in \cifar{}.}
		\label{fig:gnorm-withwithoutsn-cifar10}
	\end{minipage}%
	~~
	\begin{minipage}{0.47\textwidth}
		\centering
		\includegraphics[width=\linewidth]{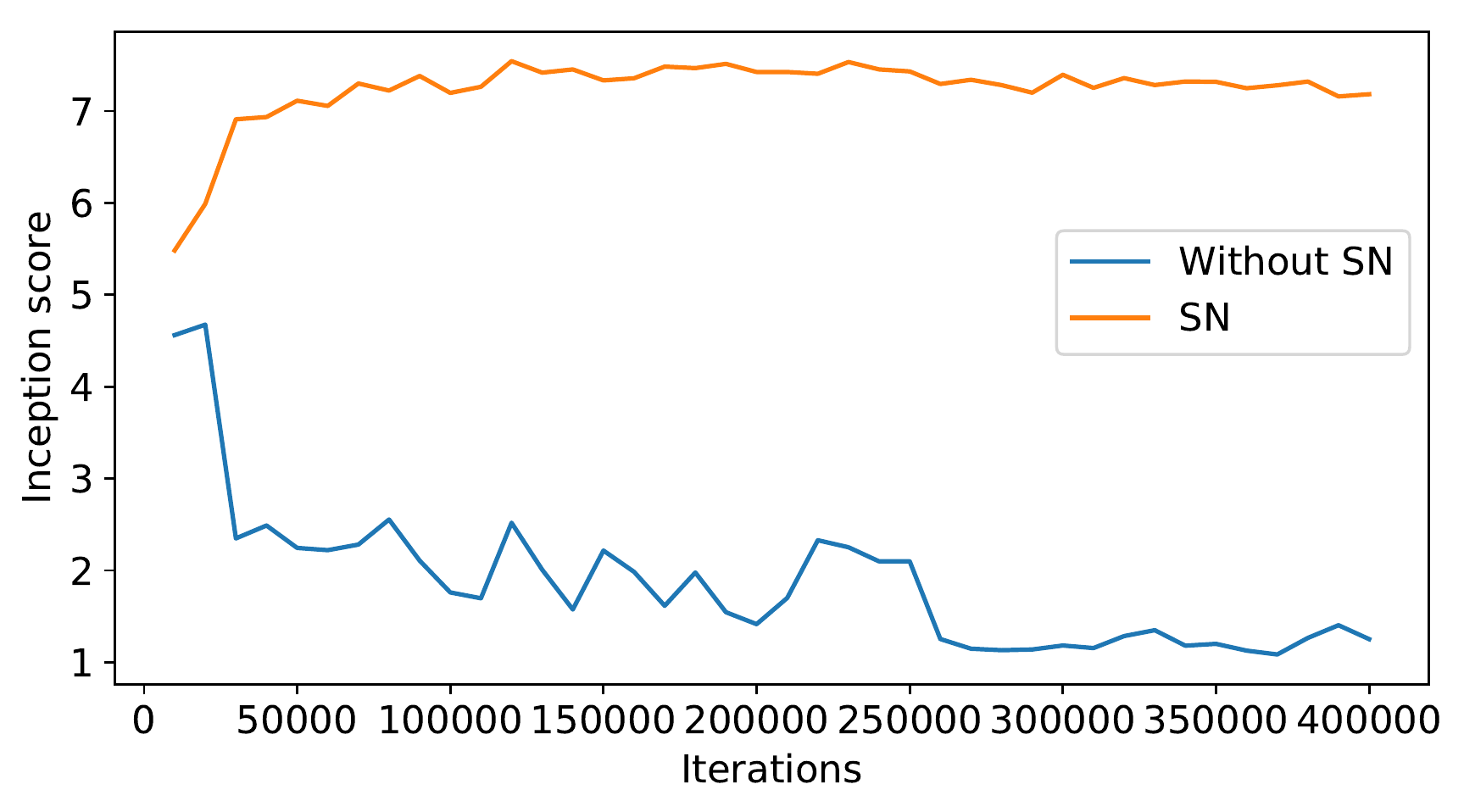}
		\caption{Inception score with and without SN in \cifar{}.}
		\label{fig:inception-withwithoutsn-cifar10}
	\end{minipage}
\end{figure}

\subsection{Comparing Two Variants Spectral Norms}
\label{app:compare-sn}
\cref{fig:sns-sn-cifar10,fig:sns-realsn-cifar10} show the ratio between two versions of spectral norm \cite{miyato2018spectral,farnia2018generalizable} throughout the training of the popular SN \cite{miyato2018spectral} and the strict SN \cite{farnia2018generalizable}.
$\brsp{Conv}$ denotes the spectral norm of the expanded matrix $\brconvsp{w}$  used in \cite{farnia2018generalizable}. $\brsp{w}$ denotes the spectral norm of reshaped matrix $\brspw{w}$ used in \cite{miyato2018spectral}. 
The theoretical lower and upper bound are calculated according to Corollary 1 in \cite{tsuzuku2018lipschitz}. 
We can see that no matter in which architecture, $\brconvsp{w}$ is usually strictly larger than $\brspw{w}$.
Note that the reason why in some cases the ratio exceeds the upper bound in \cref{fig:sns-sn-cifar10} is because the spectral norms are calculated using power iteration \cite{miyato2018spectral,farnia2018generalizable} which has approximation error.

\begin{figure}[t]
	\centering
	\begin{minipage}{.47\textwidth}
		\centering
		\includegraphics[width=\linewidth]{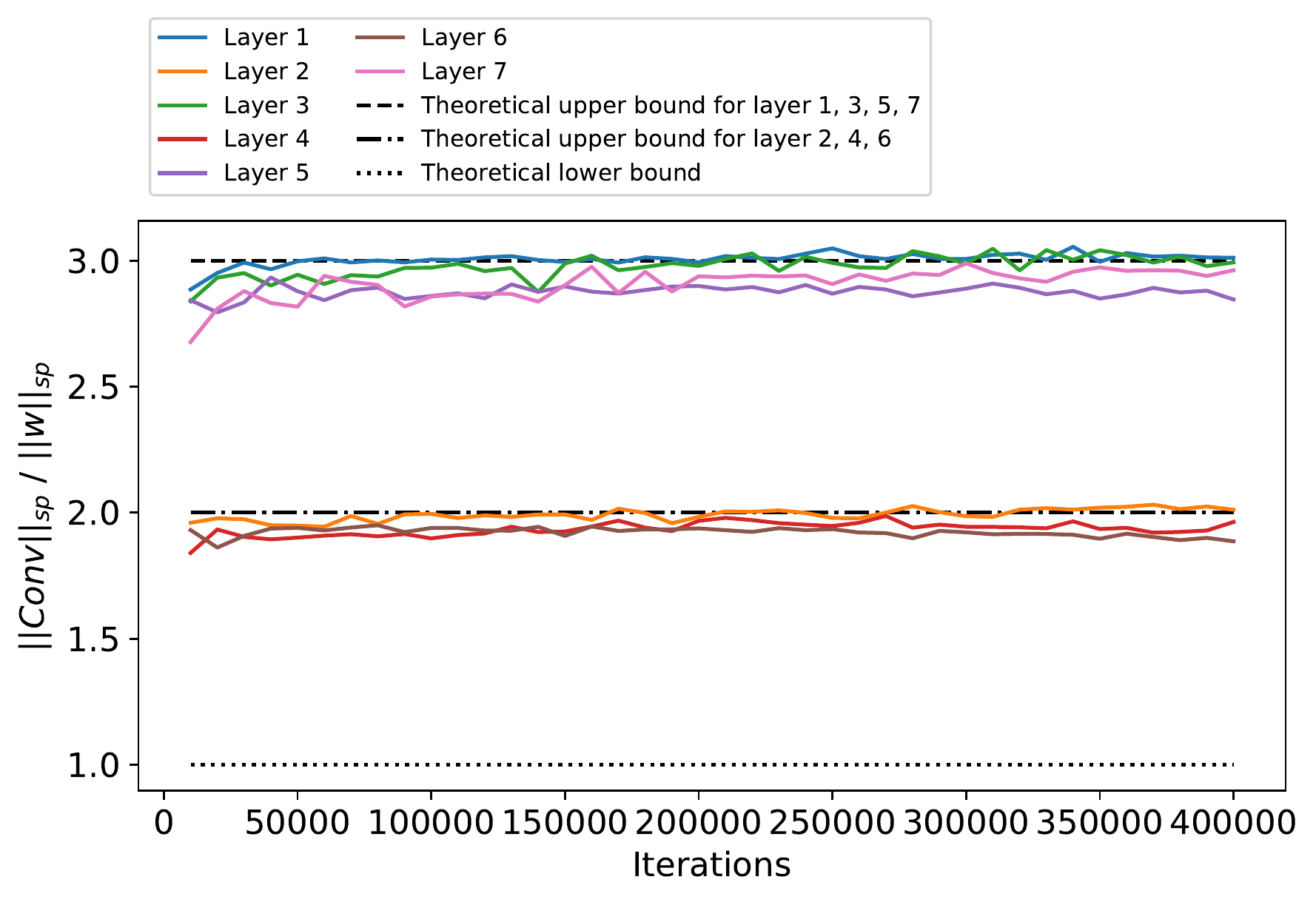}
		\caption{The ratio of two spectral norms throughout the training of the popular SN \cite{miyato2018spectral} in \cifar{}.}
		\label{fig:sns-sn-cifar10}
	\end{minipage}%
	~~
	\begin{minipage}{0.47\textwidth}
		\centering
		\includegraphics[width=\linewidth]{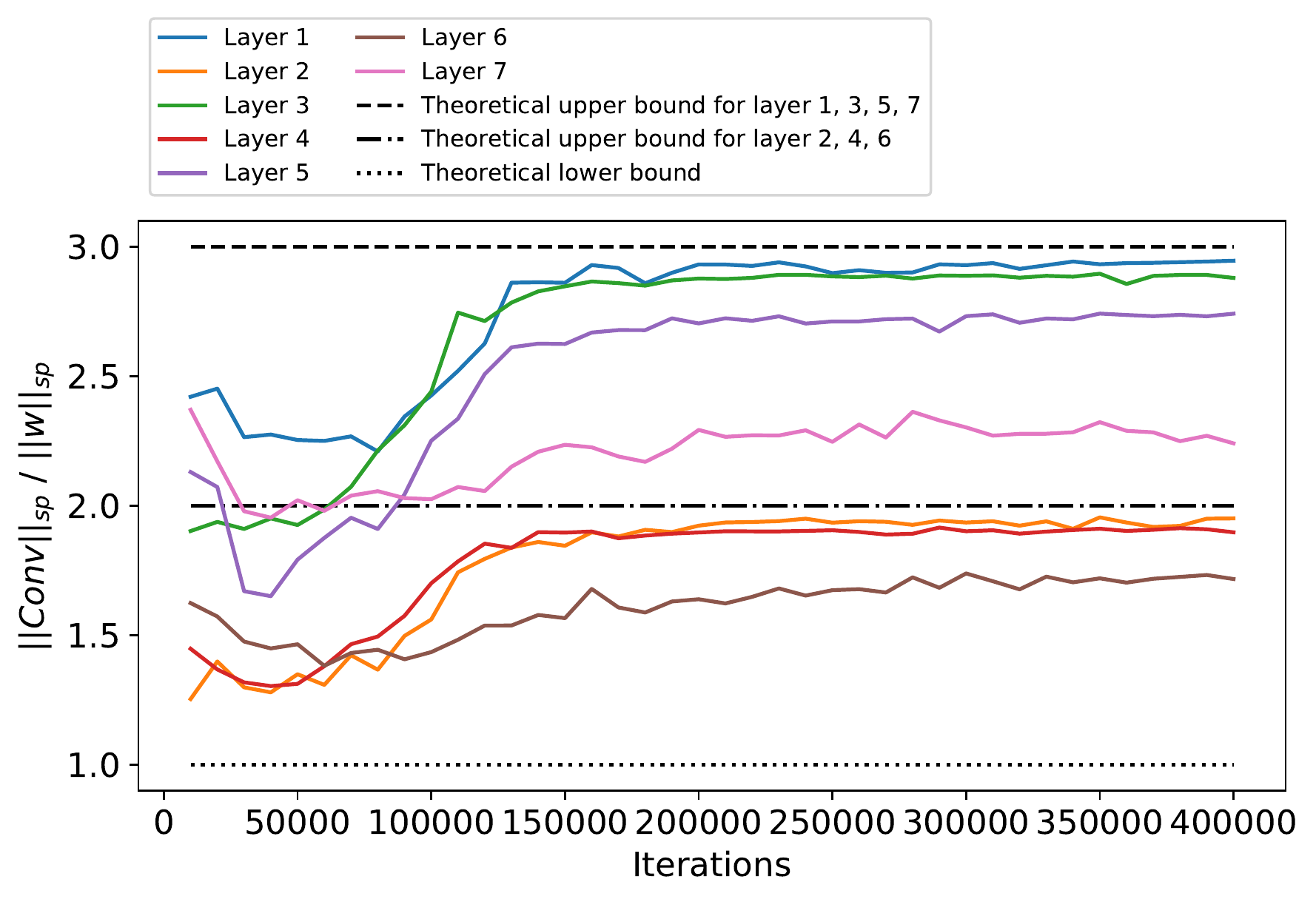}
		\caption{The ratio of two spectral norms throughout the training of the strict SN \cite{farnia2018generalizable} in \cifar{}.}
		\label{fig:sns-realsn-cifar10}
	\end{minipage}
\end{figure}

\subsection{Parameter Variance of Scaled SN}
\label{app:parameter-variance-scaled-sn}
Figure \cref{fig:paravar-scaledsn-cifar10} shows the parameter variance of scaled SN for both SN versions \cite{miyato2018spectral,farnia2018generalizable}. We can see that when scale=1.75, the product of parameter variances for \snconv{} \cite{farnia2018generalizable} is similar to the one of \snw{} \cite{miyato2018spectral}. Moreover, by comparing \cref{fig:paravar-scaledsn-cifar10} and \cref{fig:sn_scale} we can see that when the products of variances of two SN variants are similar, the sample quality is also similar. This confirms the intuition from LeCun initialization \cite{LeCun1998} that the magnitude of variance plays an important role on the performance of neural network, and it should not be too large nor too small.

\begin{figure}
	\centering
	\includegraphics[width=0.5\linewidth]{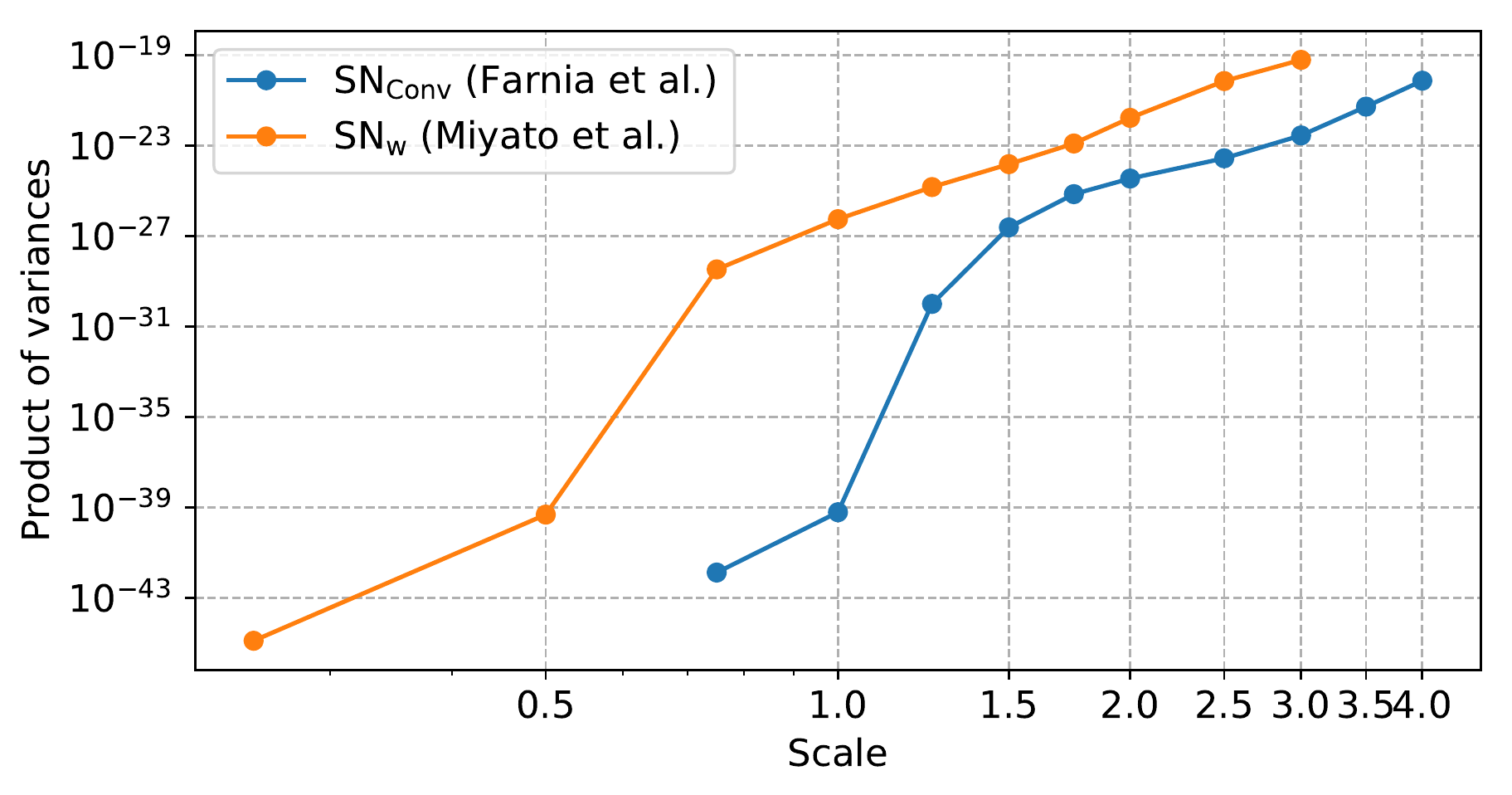}
	\caption{The parameter variance of scaled SN in \cifar{}.}
	\label{fig:paravar-scaledsn-cifar10}
\end{figure}

\section{Experimental Details and Additional Results on Scaling (\cref{sec:scaling_approach})}
\label{app:result_scaling_approach}

\subsection{Experimental Details}
The network architectures are shown in \cref{tbl:gen-network-cifar10-stl10,tbl:dis-network-cifar10-stl10}.
SN models are run for 400k iterations. 
LeCun initialization models are run till the sample quality converges or starts dropping (usually within 400k iterations).
Batch size is 64. The optimizer is Adam.
$\alpha_g=0.0001, \alpha_d=0.0001, \beta_1=0.5, \beta_2=0.999, n_{dis}=1$.
We use hinge loss \cite{miyato2018spectral}. 

Since LeCun initialization is unstable when the scaling is not proper, in \cref{fig:sn_scaling_approach}, we plot the best score during training instead of the score at the end of training.

\subsection{Additional Results}
Although the good scaling modes for SN and LeCun initialization seem to be very different in \cref{fig:sn_scaling_approach}, there indeed exists a (perhaps  coincidental) correspondence in terms of parameter variances. In \cref{fig:sn_scaling_approach_variance}, we show the inception score v.s. parameter variances for SN and LeCun initialization. We can see that the first good mode occurs when log of the product of parameter variances is around -70 to -60, and the second mode is around -50 to -40.

\begin{figure}[ht]
	\centering
	\includegraphics[width=0.3\linewidth]{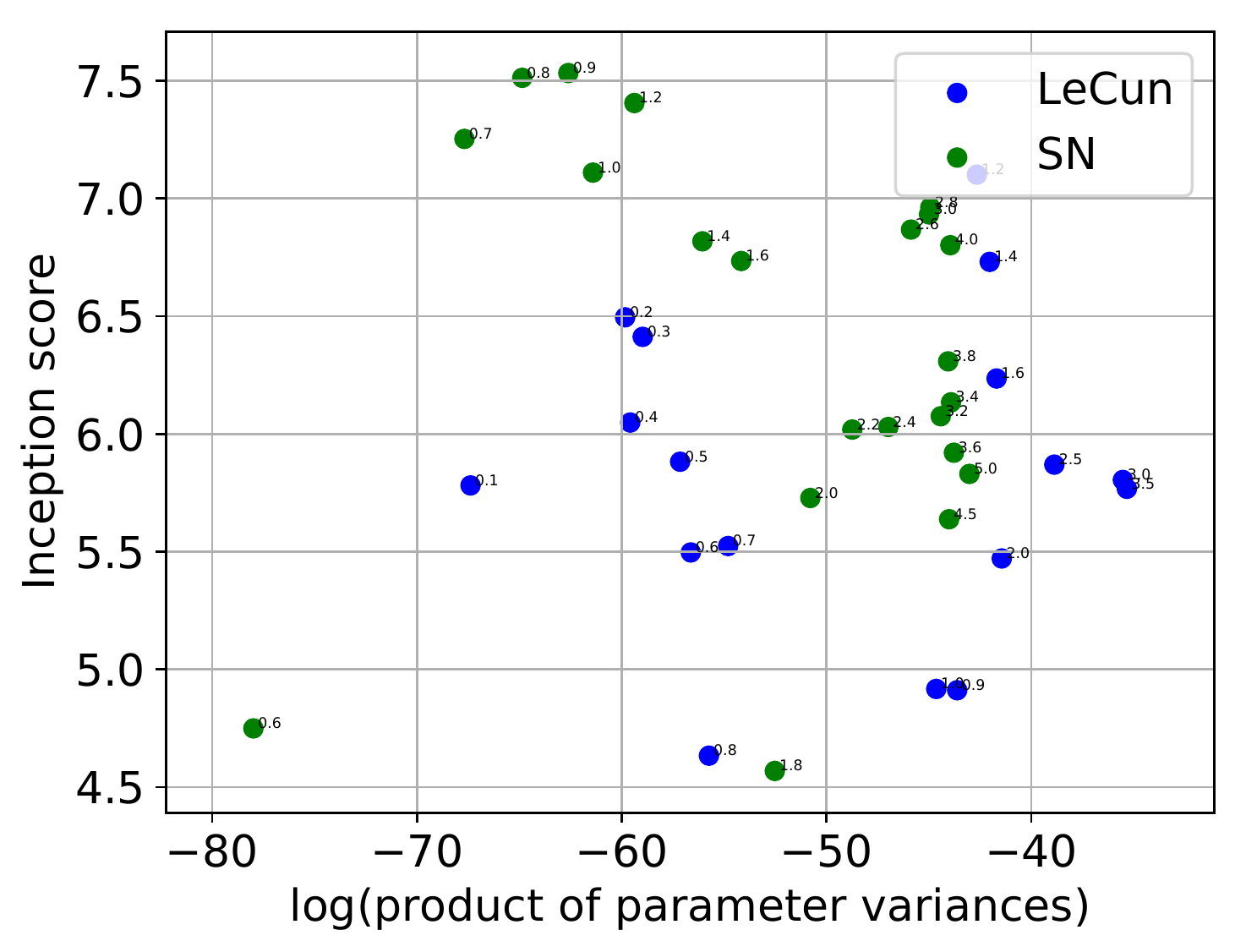}
	\caption{Inception score v.s. parameter variances of scaled SN and scaled LeCun initialization in \cifar{}. Each point corresponds to one run, at the point when the score is the best during training. The numbers near each point indicate the scaling.}
	\label{fig:sn_scaling_approach_variance}
\end{figure}

\section{Results with Different Hyper-parameters and SN Variants}
\label{app:result}

In addition to SN \cite{miyato2018spectral},  we compare against two variants of SN proposed in the appendix of  \cite{miyato2018spectral}, which we denote ``same $\gamma$'' and ``diff. $\gamma$'' (details in \cref{app:sn-variant}). 
{These two variants are reported to be worse than SN in \cite{miyato2018spectral} and are not used in practice, but we include them here for reference.}
We run experiments on \cifar{}, \stl{}, \celeba{}, and \imagenet{}, with two widely-used metrics for sample quality: inception score \cite{salimans2016improved} and Frechet Inception Distance (FID) \cite{heusel2017gans} (details in \cref{app:dataset-metric}).

\begin{table*}[t]
	\centering
	\scalebox{0.98}{
		\begin{tabular}{l | c c | c c | c c}
			\toprule
			& \multicolumn{2}{c|}{\cifar{}} & \multicolumn{2}{c|}{\stl{}} &\multicolumn{1}{c}{\celeba{}} \\
			\midrule
			& Inception score $\uparrow$ & FID $\downarrow$ & Inception score $\uparrow$ & FID $\downarrow$ & FID $\downarrow$\\
			\midrule
			Real data & 11.26 & 9.70 & 26.70 & 10.17 & 4.44\\ \hline
			SN (same $\gamma$) & 6.46 $\pm$ 0.06 & 42.35 $\pm$ 0.74 &   \textbf{8.86 $\pm$ 0.03} &   \textbf{54.61 $\pm$ 0.51}  & \textbf{7.74 $\pm$ 0.11}\\
			BSN (same $\gamma$) & \textbf{6.69 $\pm$ 0.05} & \textbf{39.62 $\pm$ 0.40} & 8.76 $\pm$ 0.03 & 55.04 $\pm$ 0.48 & 7.83 $\pm$ 0.09\\\hline
			SN (diff. $\gamma$) & 6.53 $\pm$ 0.01 & 41.88 $\pm$ 0.50 &   8.79 $\pm$ 0.03 &  56.76 $\pm$ 0.44  & \textbf{\red{7.54 $\pm$ 0.08}}\\ 
			BSN (diff. $\gamma$)  & \textbf{6.72 $\pm$ 0.05} & \textbf{38.15 $\pm$ 0.72} & \textbf{8.80 $\pm$ 0.03} & \textbf{53.99 $\pm$ 0.33} & 7.67 $\pm$ 0.04\\\hline
			SN& 7.22 $\pm$ 0.09 & 31.43 $\pm$ 0.90 &   9.16 $\pm$ 0.03&   \textbf{\red{42.89 $\pm$ 0.54}} & 9.09 $\pm$ 0.32\\
			\nameshort{}& \textbf{\red{7.58 $\pm$ 0.04}} & \textbf{\red{26.62 $\pm$ 0.21}} &    \textbf{\red{9.25 $\pm$ 0.01}}&  42.98 $\pm$ 0.54 & \textbf{8.54 $\pm$ 0.20}\\
			\bottomrule
		\end{tabular}
	}
	\caption{Inception scores and FIDs on \cifar{}, \stl{}, and \celeba{}. Each experiment is conducted with 5 random seeds, with mean and standard error reported. 
		We follow the common practice of excluding Inception Score in \celeba{} as the inception network is pretrained on ImageNet, which is very different from CelebA. {The bold font marks the best numbers between SN and \nameshort{} using the same variant. The red color marks the best numbers among all runs. The``same $\gamma$" and ``diff. $\gamma$" variants are not used in practice and are reported to have bad performance in \cite{miyato2018spectral}.}}
	\label{tbl:cifar-stl}
\end{table*}

\label{sec:cifar_stl}
We use the network architecture from SN \cite{miyato2018spectral}. We controlled five  hyper-parameters (\cref{tbl:cifar-stl10-hyper}, \cref{app:result_cifar10}): $\alpha_g$ and $\alpha_d$, the generator/discriminator learning rates, $\beta_1, \beta_2$, Adam momentum parameters \cite{kingma2014adam}, and $n_{dis}$, the number of discriminator updates per generator update. 
Three hyper-parameter settings are from \cite{miyato2018spectral}, with equal discriminator and generator learning rates; 
the final two test  \emph{unequal} learning rates {for showing a more thorough comparison}.
More details are in \cref{app:result_cifar10,app:result_stl10}.

As in \cite{miyato2018spectral},  we report the metrics from the best hyper-parameter for each algorithm in \cref{tbl:cifar-stl}. 
\emph{\nameshort{} outperforms the standard SN in all sample quality metrics except FID score on \stl{}, where their metrics are within standard error of each other.} 
Regarding the SN variants with $\gamma$, in \cifar{} and \stl{}, they have worse performance than SN and \nameshort{}, same as reported in \cite{miyato2018spectral}. 
In \celeba{}, the SN variants have better performance for the best hyper-parameter setting. But in general, these SN variants are very sensitive to hyper-parameters (\cref{app:result_cifar10,app:result_stl10,app:result_celeba}), therefore they are not adopted in practice \cite{miyato2018spectral}. {Nevertheless, \nameshort{} is still able to improve or have similar performance on those two variants in most of the settings.}

More importantly, the superiority of \nameshort{} is stable across hyper-parameters. \cref{fig:cifar-bar-is,fig:stl-bar-is} show the inception scores of all the hyper-parameters we tested on \cifar{} and \stl{}. 
\nameshort{} has the best or competitive performance in most of the settings. 
The only exception is $n_{dis}=5$ setting in \stl{}, where we observe that the performance from both SN and \nameshort{} have larger variance across different random seeds, and the SN variants with $\gamma$ perform better. On \celeba{}, BSN also outperforms SN in FID across all hyper-parameters (\cref{app:result_celeba}), and it outperforms all SN variants in every hyper-parameter setting except one (\cref{fig:celeba-bar-fid}).

More results (generated images, training curves, FID plots) are in \cref{app:result_cifar10,app:result_stl10,app:result_celeba}.
\begin{figure}[ht]
	\centering
	\includegraphics[width=0.5\linewidth]{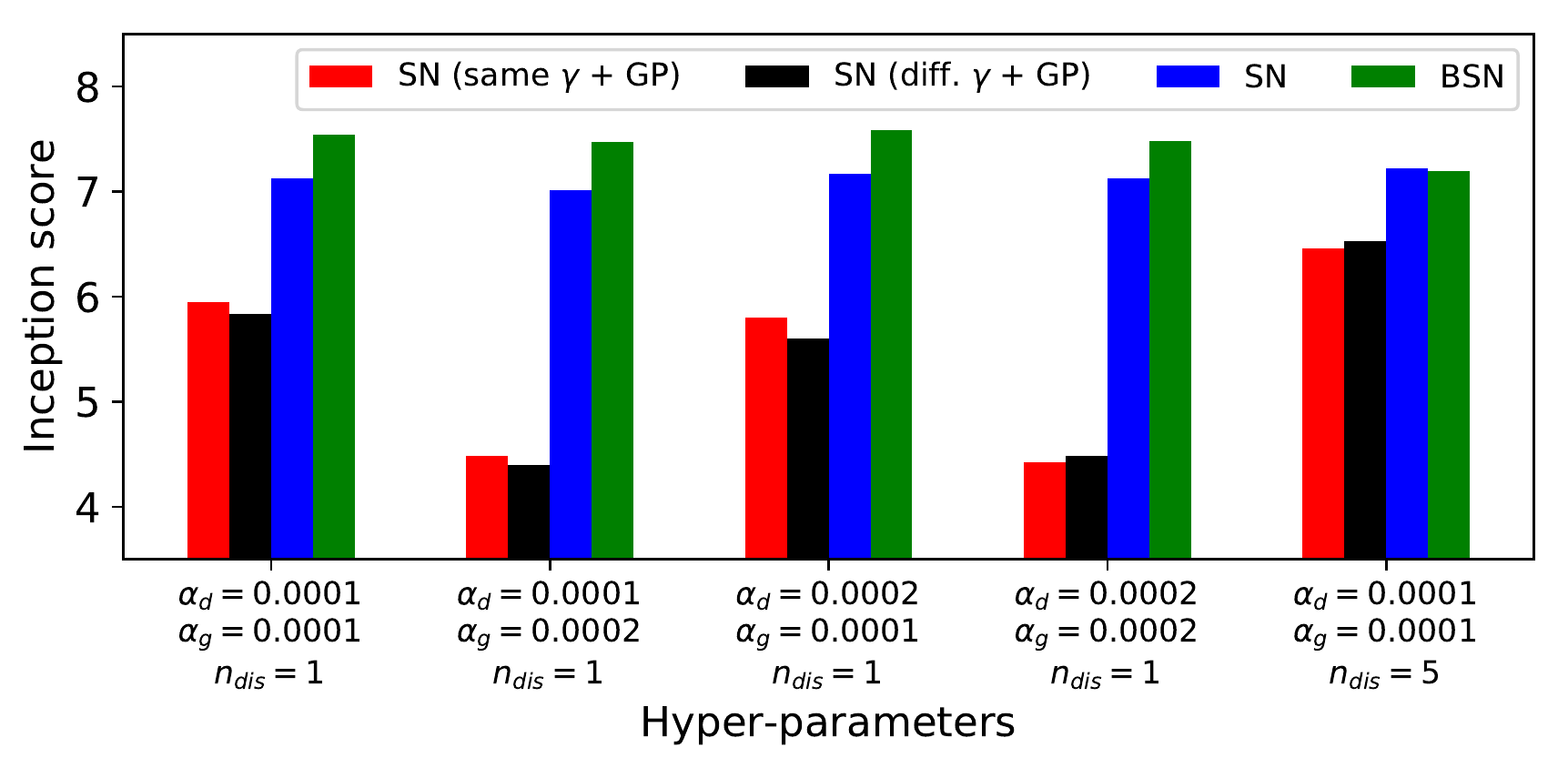}
	\caption{Inception score in \cifar{}. The results are averaged over 5 random seeds.}
	\label{fig:cifar-bar-is}
\end{figure}%
\begin{figure}[ht]
	\centering
	\includegraphics[width=0.5\linewidth]{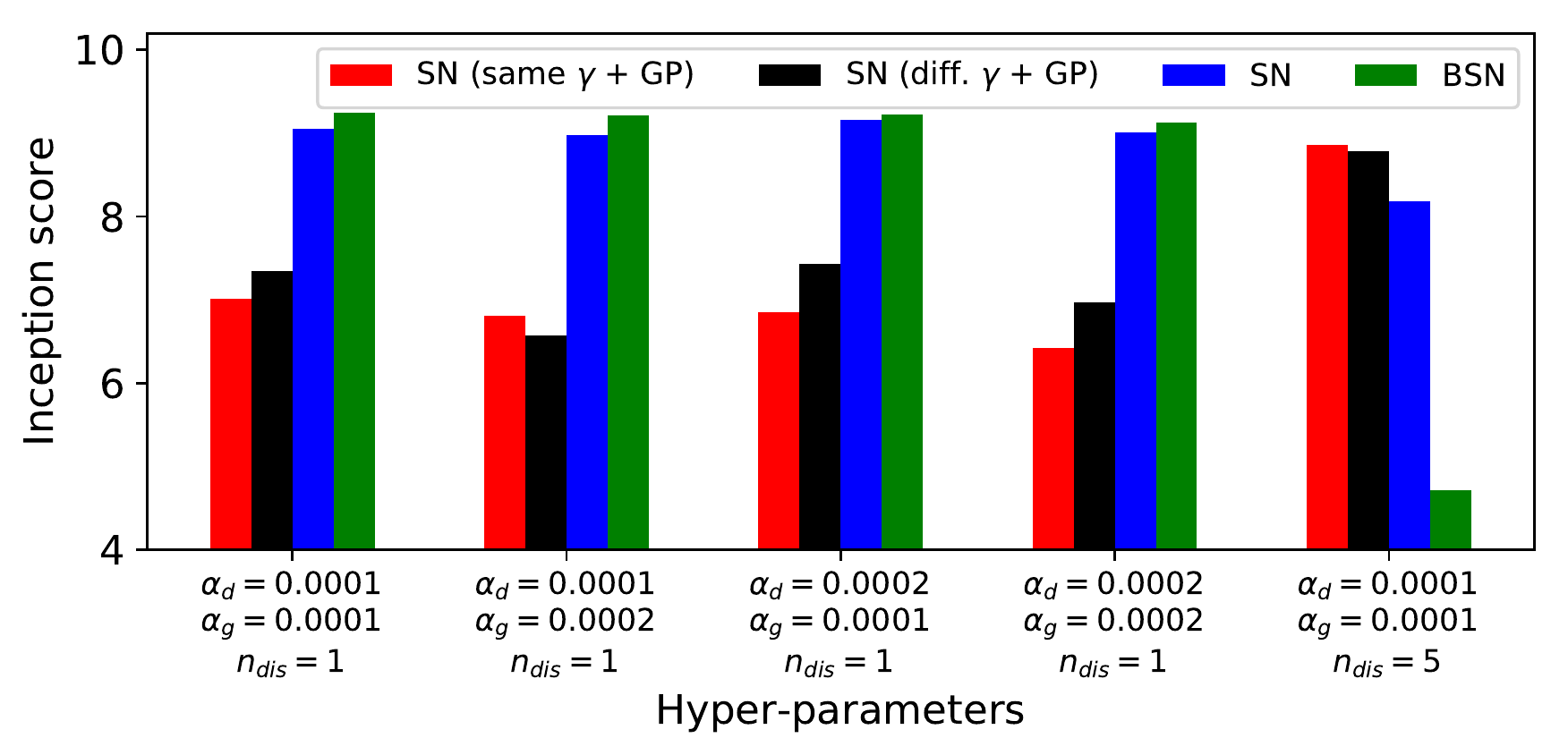}
	\caption{Inception score in \stl{}. The results are averaged over 5 random seeds.}
	\label{fig:stl-bar-is}
\end{figure}

\section{Details on SN Variants}
\label{app:sn-variant}

In Appendix E of \cite{miyato2018spectral}, a variant of SN is introduced. Instead of strictly setting the spectral norm of each layer, the idea of this approach is to release the constraint by multiplying each spectral normalized weights with a trainable parameter $\gamma$. However, this would make the gradient of discriminator arbitrarily large, which violates the original motivation of SN. Therefore, the approach incorporates gradient penalty \cite{gulrajani2017improved} for setting the Lipschitz constant of discriminator to 1. The gradient penalty weights are set to 10 in all experiments.

However, from the description in \cite{miyato2018spectral}, it is unclear if all layers have the same or separated $\gamma$. Therefore, we try both versions in our experiments. ``Same $\gamma$'' denotes that version where all layers share the same $\gamma$. ``Diff. $\gamma$'' denotes the version where each layer has a separate $\gamma$.

\section{Experimental Details and Additional Results on \cifar{}}
\label{app:result_cifar10}

\subsection{Experimental Details}
The network architectures are shown in \cref{tbl:gen-network-cifar10-stl10,tbl:dis-network-cifar10-stl10}.
All experiments are run for 400k iterations. Batch size is 64. The optimizer is Adam.
We use the five hyper-parameter settings listed in \cref{tbl:cifar-stl10-hyper}.
(In \cref{tbl:all} we only show the results from the first hyper-parameter setting.)
We use hinge loss with the popular SN implementation \cite{miyato2018spectral}. 

For \namessnshort{} in \cref{tbl:all}, we ran following scales: [0.7, 0.8, 0.9, 1.0, 1.2, 1.4, 1.6, 1.8, 2.0, 2.2, 2.4, 2.6, 2.8, 3.0, 3.2, 3.4, 3.6, 3.8, 4.0, 4.5, 5.0, 5.5, 6.0, 7.0, 8.0, 9.0, 10.0], and present the results from best one for each metric. For \namebssnshort{} in \cref{tbl:all}, we ran the following scales: [0.7, 0.8, 0.9, 1.0, 1.2, 1.4, 1.6, 1.8, 2.0, 2.2, 2.4, 2.6, 2.8, 3.0, 3.2, 3.4, 3.6, 3.8, 4.0], and present the results from the best one for each metric.

\begin{table}
	\centering
	\begin{tabular}{ccccc}
		\toprule  
		$\alpha_g$ & $\alpha_d$ & $\beta_1$ & $\beta_2$ & $n_{dis}$ \\
		\midrule
		0.0001 & 0.0001 & 0.5 & 0.9 & 5\\
		0.0001 & 0.0001 & 0.5 & 0.999 & 1\\
		0.0002 & 0.0002 & 0.5 & 0.999 & 1\\
		0.0001 & 0.0002 & 0.5 & 0.999 & 1\\
		0.0002 & 0.0001 & 0.5 & 0.999 & 1\\
		\bottomrule
	\end{tabular}
	\caption{Hyper-parameters tested in \cifar{} and \stl{} experiments. The first three settings are from \cite{miyato2018spectral,gulrajani2017improved, warde2016improving,radford2015unsupervised}. $\alpha_g$ and $\alpha_d$: learning rates for generator and discriminator. $\beta_1, \beta_2$: momentum parameters in Adam. $n_{dis}$: number of discriminator updates per generator update. }
	\label{tbl:cifar-stl10-hyper}
\end{table} 

\subsection{FID Plot}
\cref{fig:cifar-bar-fid} shows the FID score in \cifar{} dataset. We can see that \nameshort{} has the best performance in all 5 hyper-parameter settings.

\begin{figure}
		\centering
		\includegraphics[width=0.45\linewidth]{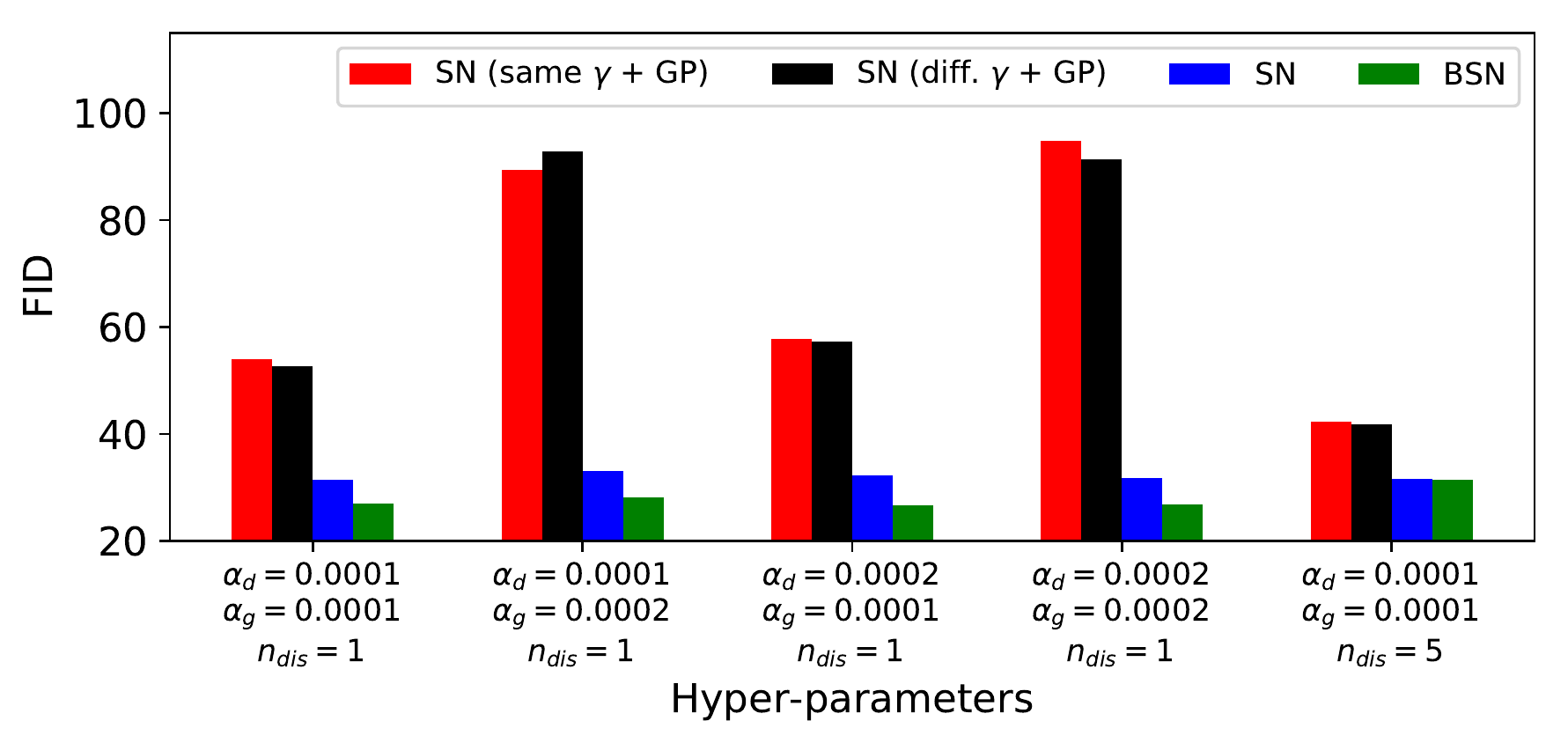}
		\caption{FID in \cifar{}. The results are averaged over 5 random seeds.}
		\label{fig:cifar-bar-fid}
\end{figure}

\subsection{Training Curves}
From \cref{sec:cifar_stl} we can see that SN (no $\gamma$) and \nameshort{} generally have the best performance. Therefore, in this section, we focus on comparing these two algorithms with the training curves. \cref{fig:cifar-g0.0001-d0.0001-ndis1-inception,fig:cifar-g0.0001-d0.0001-ndis1-fid,fig:cifar-g0.0001-d0.0002-ndis1-inception,fig:cifar-g0.0001-d0.0002-ndis1-fid,fig:cifar-g0.0002-d0.0001-ndis1-inception,fig:cifar-g0.0002-d0.0001-ndis1-fid,fig:cifar-g0.0002-d0.0002-ndis1-inception,fig:cifar-g0.0002-d0.0002-ndis1-fid,fig:cifar-g0.0001-d0.0001-ndis5-inception,fig:cifar-g0.0001-d0.0001-ndis5-fid} show the inception score and FID of these two algorithms during training. Generally, we see that \nameshort{} converges slower than SN \emph{at the beginning of training}. However, as training proceeds, the sample quality of SN often drops (e.g. \cref{fig:cifar-g0.0001-d0.0001-ndis1-inception,fig:cifar-g0.0001-d0.0001-ndis1-fid,fig:cifar-g0.0001-d0.0002-ndis1-inception,fig:cifar-g0.0001-d0.0002-ndis1-fid,fig:cifar-g0.0002-d0.0001-ndis1-inception,fig:cifar-g0.0002-d0.0001-ndis1-fid,fig:cifar-g0.0002-d0.0002-ndis1-inception,fig:cifar-g0.0002-d0.0002-ndis1-fid}), whereas the sample quality of \nameshort{} always increases and then stabilizes at the high level. In most cases, \nameshort{} not only outperforms SN at the end of training, but also outperforms the peak sample quality of SN during training (e.g. \cref{fig:cifar-g0.0001-d0.0001-ndis1-inception,fig:cifar-g0.0001-d0.0001-ndis1-fid,fig:cifar-g0.0001-d0.0002-ndis1-inception,fig:cifar-g0.0001-d0.0002-ndis1-fid,fig:cifar-g0.0002-d0.0001-ndis1-inception,fig:cifar-g0.0002-d0.0001-ndis1-fid,fig:cifar-g0.0002-d0.0002-ndis1-inception,fig:cifar-g0.0002-d0.0002-ndis1-fid}). From these results, we can conclude that \nameshort{} improves both the sample quality and training stability over SN.

\trainingcurvefid{\cifar{}}{cifar10/296}{0.0001}{0.0001}{1}{cifar}
\trainingcurve{\cifar{}}{cifar10/296}{0.0001}{0.0002}{1}{cifar}
\trainingcurve{\cifar{}}{cifar10/296}{0.0002}{0.0001}{1}{cifar}
\trainingcurve{\cifar{}}{cifar10/296}{0.0002}{0.0002}{1}{cifar}
\trainingcurve{\cifar{}}{cifar10/296}{0.0001}{0.0001}{5}{cifar}

\subsection{Generated Images}
\cref{fig:cifar-SN_same_gamma-generated-samples,fig:cifar-SN_diff_gamma-generated-samples,fig:cifar-SN_no_gamma-generated-samples,fig:cifar-BSN-generated-samples} show the generated images from the run with the best inception score for each algorithm.

    \begin{figure}
        \centering
        \includegraphics[width=0.6\linewidth]{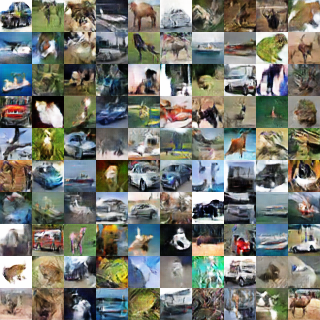}
        \caption{Generated samples from the best run of SN (same $\gamma$) in \cifar{}. 
            The hyper-parameters are: $\alpha_g=\protect\input{figure/sample_quality/cifar10/296/0.0,sn-double-gamma_best_sample/SN_same_gamma/g_lr.txt}$,
            $\alpha_d=\protect\input{figure/sample_quality/cifar10/296/0.0,sn-double-gamma_best_sample/SN_same_gamma/d_lr.txt}$,
            $n_{dis}=\protect\input{figure/sample_quality/cifar10/296/0.0,sn-double-gamma_best_sample/SN_same_gamma/n_dis.txt}$.
            Inception score is $\protect\input{figure/sample_quality/cifar10/296/0.0,sn-double-gamma_best_sample/SN_same_gamma/inception_score_mean.txt}$.
            FID is $\protect\input{figure/sample_quality/cifar10/296/0.0,sn-double-gamma_best_sample/SN_same_gamma/fid.txt}$.
        }
        \label{fig:cifar-SN_same_gamma-generated-samples}
    \end{figure}

    \begin{figure}
        \centering
        \includegraphics[width=0.6\linewidth]{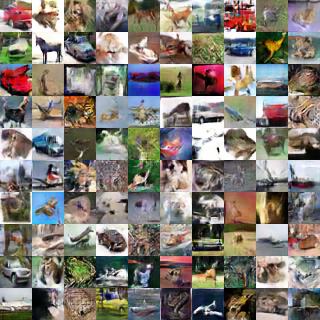}
        \caption{Generated samples from the best run of SN (diff. $\gamma$) in \cifar{}. 
            The hyper-parameters are: $\alpha_g=\protect\input{figure/sample_quality/cifar10/296/0.0,sn-double-gamma_best_sample/SN_diff_gamma/g_lr.txt}$,
            $\alpha_d=\protect\input{figure/sample_quality/cifar10/296/0.0,sn-double-gamma_best_sample/SN_diff_gamma/d_lr.txt}$,
            $n_{dis}=\protect\input{figure/sample_quality/cifar10/296/0.0,sn-double-gamma_best_sample/SN_diff_gamma/n_dis.txt}$.
            Inception score is $\protect\input{figure/sample_quality/cifar10/296/0.0,sn-double-gamma_best_sample/SN_diff_gamma/inception_score_mean.txt}$.
            FID is $\protect\input{figure/sample_quality/cifar10/296/0.0,sn-double-gamma_best_sample/SN_diff_gamma/fid.txt}$.
        }
        \label{fig:cifar-SN_diff_gamma-generated-samples}
    \end{figure}

    \begin{figure}
        \centering
        \includegraphics[width=0.6\linewidth]{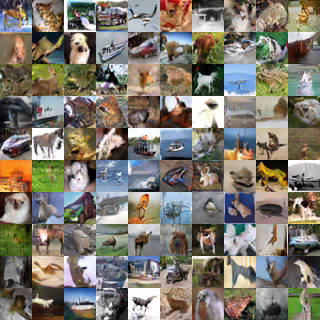}
        \caption{Generated samples from the best run of SN in \cifar{}. 
            The hyper-parameters are: $\alpha_g=\protect\input{figure/sample_quality/cifar10/296/0.0,sn-double-gamma_best_sample/SN_no_gamma/g_lr.txt}$,
            $\alpha_d=\protect\input{figure/sample_quality/cifar10/296/0.0,sn-double-gamma_best_sample/SN_no_gamma/d_lr.txt}$,
            $n_{dis}=\protect\input{figure/sample_quality/cifar10/296/0.0,sn-double-gamma_best_sample/SN_no_gamma/n_dis.txt}$.
            Inception score is $\protect\input{figure/sample_quality/cifar10/296/0.0,sn-double-gamma_best_sample/SN_no_gamma/inception_score_mean.txt}$.
            FID is $\protect\input{figure/sample_quality/cifar10/296/0.0,sn-double-gamma_best_sample/SN_no_gamma/fid.txt}$.
        }
        \label{fig:cifar-SN_no_gamma-generated-samples}
    \end{figure}

    \begin{figure}
        \centering
        \includegraphics[width=0.6\linewidth]{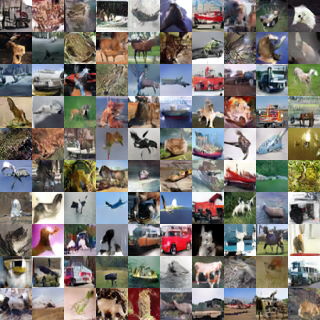}
        \caption{Generated samples from the best run of \nameshort{} in \cifar{}. 
            The hyper-parameters are: $\alpha_g=\protect\input{figure/sample_quality/cifar10/296/0.0,sn-double-gamma_best_sample/BSN/g_lr.txt}$,
            $\alpha_d=\protect\input{figure/sample_quality/cifar10/296/0.0,sn-double-gamma_best_sample/BSN/d_lr.txt}$,
            $n_{dis}=\protect\input{figure/sample_quality/cifar10/296/0.0,sn-double-gamma_best_sample/BSN/n_dis.txt}$.
            Inception score is $\protect\input{figure/sample_quality/cifar10/296/0.0,sn-double-gamma_best_sample/BSN/inception_score_mean.txt}$.
            FID is $\protect\input{figure/sample_quality/cifar10/296/0.0,sn-double-gamma_best_sample/BSN/fid.txt}$.
        }
        \label{fig:cifar-BSN-generated-samples}
    \end{figure}

\clearpage

\section{Experimental Details and Additional Results on \stl{}}
\label{app:result_stl10}

\subsection{Experimental Details}
The network architectures are shown in \cref{tbl:gen-network-cifar10-stl10,tbl:dis-network-cifar10-stl10}.
Batch size is 64. The optimizer is Adam. We use the five hyper-parameter settings listed in \cref{tbl:cifar-stl10-hyper}.
(In \cref{tbl:all} we only show the results from the first hyper-parameter setting.)
We use hinge loss with the popular SN implementation \cite{miyato2018spectral}. 

SN (no $\gamma$) and \nameshort{} under $n_{dis}=1$ settings are run for 800k iterations as we observe that they need longer time to converge. All other experiments are run for 400k iterations.

For \namessnshort{} and \namebssnshort{} in \cref{tbl:all}, we ran following scales: [0.7, 0.8, 0.9, 1.0, 1.2, 1.4, 1.6], and present the results from best one for each metric.

\subsection{FID Plot}
\cref{fig:stl-bar-fid} shows the FID score in \stl{} dataset. We can see that \nameshort{} has the best or competitive performance in most of the hyper-parameter settings. Again, the only exception is $n_{dis}=5$ setting.

\begin{figure}
	\centering
	\includegraphics[width=0.45\linewidth]{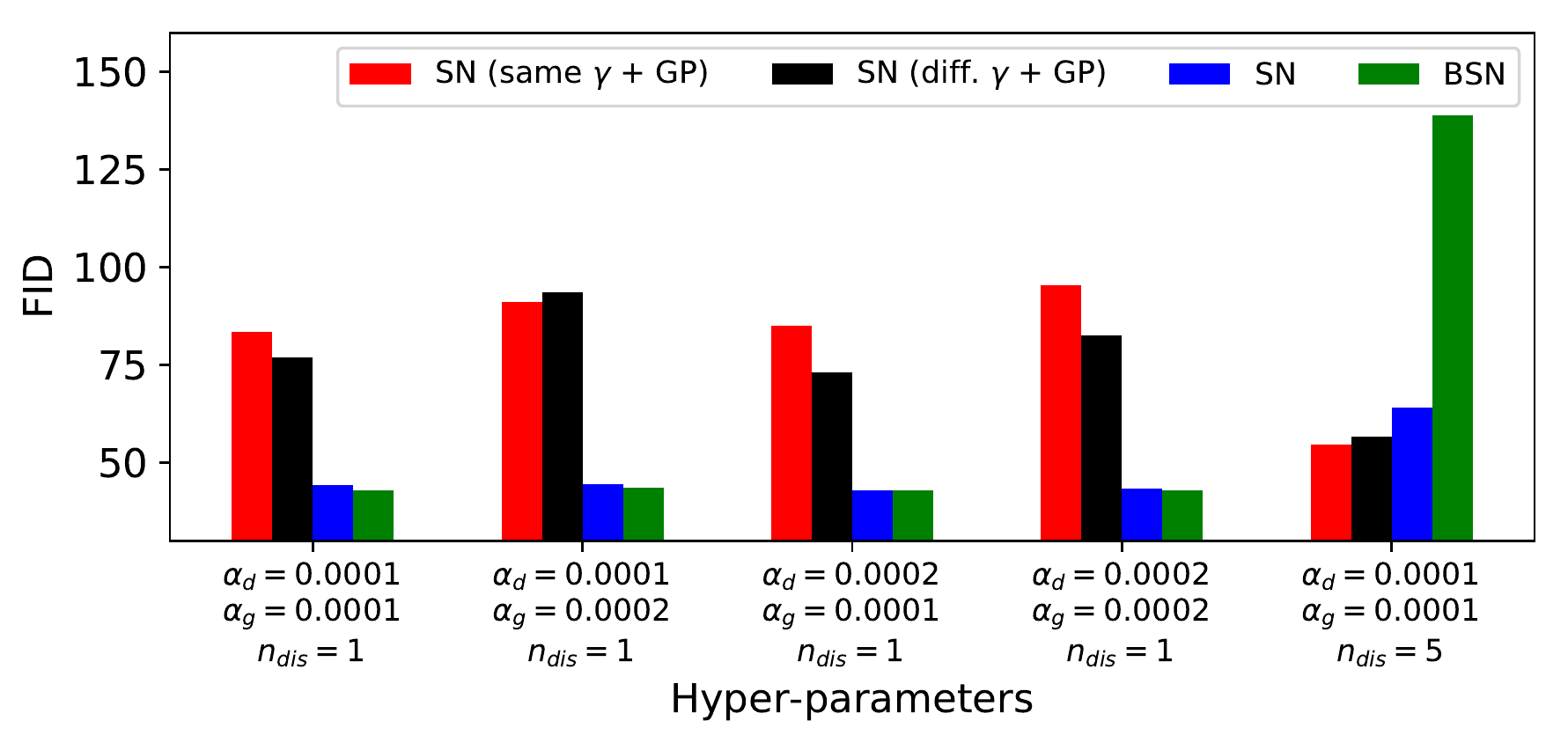}
	\caption{FID in \stl{}. The results are averaged over 5 random seeds.}
	\label{fig:stl-bar-fid}
\end{figure}

\subsection{Training Curves}
From \cref{sec:cifar_stl} we can see that SN (no $\gamma$) and \nameshort{} generally have the best performance. Therefore, in this section, we focus on comparing these two algorithms with the training curves. \cref{fig:stl-g0.0001-d0.0001-ndis1-inception,fig:stl-g0.0001-d0.0001-ndis1-fid,fig:stl-g0.0001-d0.0002-ndis1-inception,fig:stl-g0.0001-d0.0002-ndis1-fid,fig:stl-g0.0002-d0.0001-ndis1-inception,fig:stl-g0.0002-d0.0001-ndis1-fid,fig:stl-g0.0002-d0.0002-ndis1-inception,fig:stl-g0.0002-d0.0002-ndis1-fid,fig:stl-g0.0001-d0.0001-ndis5-inception,fig:stl-g0.0001-d0.0001-ndis5-fid} show the inception score and FID of these two algorithms during training. 
Generally, we see that \nameshort{} converges slower than SN \emph{at the beginning of training}. However, as training proceeds, \nameshort{} finally has better metrics in most cases.
Note that unlike \cifar{}, SN seems to be more stable in \stl{} as its sample quality does not drop in most hyper-parameters.  
But the key conclusion is the same: in most cases, \nameshort{} not only outperforms SN at the end of training, but also outperforms the peak sample quality of SN during training (e.g. \cref{fig:stl-g0.0001-d0.0001-ndis1-inception,fig:stl-g0.0001-d0.0001-ndis1-fid,fig:stl-g0.0001-d0.0002-ndis1-inception,fig:stl-g0.0001-d0.0002-ndis1-fid,fig:stl-g0.0002-d0.0001-ndis1-inception,fig:stl-g0.0002-d0.0001-ndis1-fid,fig:stl-g0.0002-d0.0002-ndis1-inception,fig:stl-g0.0002-d0.0002-ndis1-fid}). 
The only exception is the $n_{dis}=5$ setting, where both SN and \nameshort{} has instability issue: the sample quality first improves and then significantly drops. The problem with \nameshort{} seems to be severer. We discussed about this problem in \cref{sec:cifar_stl}.

\trainingcurve{\stl{}}{stl10/301}{0.0001}{0.0001}{1}{stl}
\trainingcurve{\stl{}}{stl10/301}{0.0001}{0.0002}{1}{stl}
\trainingcurve{\stl{}}{stl10/301}{0.0002}{0.0001}{1}{stl}
\trainingcurve{\stl{}}{stl10/301}{0.0002}{0.0002}{1}{stl}
\trainingcurve{\stl{}}{stl10/301}{0.0001}{0.0001}{5}{stl}

\subsection{Generated Images}

\cref{fig:stl-SN_same_gamma-generated-samples,fig:stl-SN_diff_gamma-generated-samples,fig:stl-SN_no_gamma-generated-samples,fig:stl-BSN-generated-samples} show the generated images from the run with the best inception score for each algorithm.

    \begin{figure}
        \centering
        \includegraphics[width=0.6\linewidth]{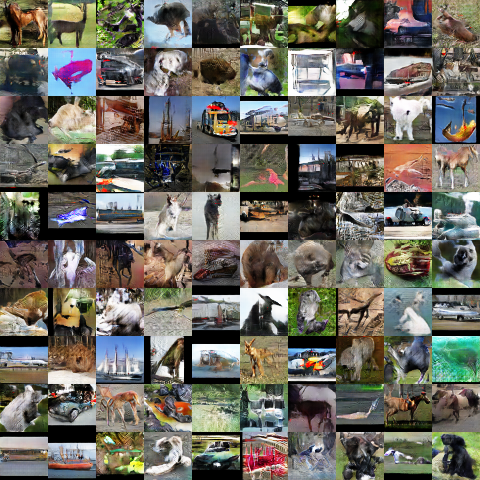}
        \caption{Generated samples from the best run of SN (same $\gamma$) in \stl{}. 
            The hyper-parameters are: $\alpha_g=\protect\input{figure/sample_quality/stl10/301/0.0,sn-double-gamma_best_sample/SN_same_gamma/g_lr.txt}$,
            $\alpha_d=\protect\input{figure/sample_quality/stl10/301/0.0,sn-double-gamma_best_sample/SN_same_gamma/d_lr.txt}$,
            $n_{dis}=\protect\input{figure/sample_quality/stl10/301/0.0,sn-double-gamma_best_sample/SN_same_gamma/n_dis.txt}$.
            Inception score is $\protect\input{figure/sample_quality/stl10/301/0.0,sn-double-gamma_best_sample/SN_same_gamma/inception_score_mean.txt}$.
            FID is $\protect\input{figure/sample_quality/stl10/301/0.0,sn-double-gamma_best_sample/SN_same_gamma/fid.txt}$.
        }
        \label{fig:stl-SN_same_gamma-generated-samples}
    \end{figure}

    \begin{figure}
        \centering
        \includegraphics[width=0.6\linewidth]{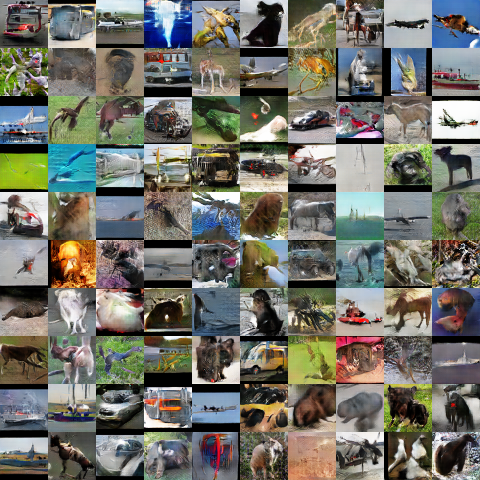}
        \caption{Generated samples from the best run of SN (diff. $\gamma$) in \stl{}. 
            The hyper-parameters are: $\alpha_g=\protect\input{figure/sample_quality/stl10/301/0.0,sn-double-gamma_best_sample/SN_diff_gamma/g_lr.txt}$,
            $\alpha_d=\protect\input{figure/sample_quality/stl10/301/0.0,sn-double-gamma_best_sample/SN_diff_gamma/d_lr.txt}$,
            $n_{dis}=\protect\input{figure/sample_quality/stl10/301/0.0,sn-double-gamma_best_sample/SN_diff_gamma/n_dis.txt}$.
            Inception score is $\protect\input{figure/sample_quality/stl10/301/0.0,sn-double-gamma_best_sample/SN_diff_gamma/inception_score_mean.txt}$.
            FID is $\protect\input{figure/sample_quality/stl10/301/0.0,sn-double-gamma_best_sample/SN_diff_gamma/fid.txt}$.
        }
        \label{fig:stl-SN_diff_gamma-generated-samples}
    \end{figure}

    \begin{figure}
        \centering
        \includegraphics[width=0.6\linewidth]{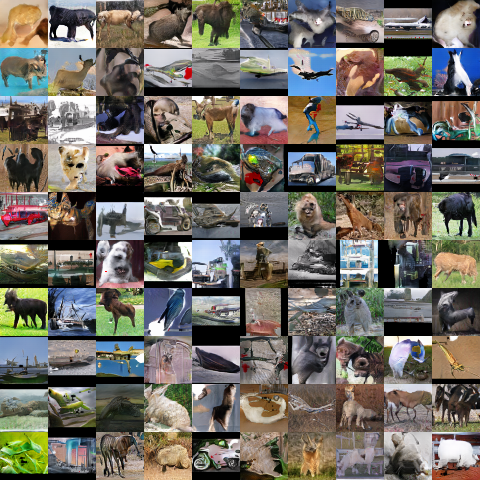}
        \caption{Generated samples from the best run of SN in \stl{}. 
            The hyper-parameters are: $\alpha_g=\protect\input{figure/sample_quality/stl10/301/0.0,sn-double-gamma_best_sample/SN_no_gamma/g_lr.txt}$,
            $\alpha_d=\protect\input{figure/sample_quality/stl10/301/0.0,sn-double-gamma_best_sample/SN_no_gamma/d_lr.txt}$,
            $n_{dis}=\protect\input{figure/sample_quality/stl10/301/0.0,sn-double-gamma_best_sample/SN_no_gamma/n_dis.txt}$.
            Inception score is $\protect\input{figure/sample_quality/stl10/301/0.0,sn-double-gamma_best_sample/SN_no_gamma/inception_score_mean.txt}$.
            FID is $\protect\input{figure/sample_quality/stl10/301/0.0,sn-double-gamma_best_sample/SN_no_gamma/fid.txt}$.
        }
        \label{fig:stl-SN_no_gamma-generated-samples}
    \end{figure}

    \begin{figure}
        \centering
        \includegraphics[width=0.6\linewidth]{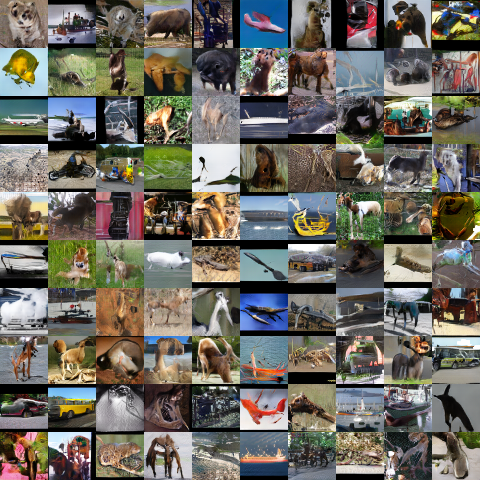}
        \caption{Generated samples from the best run of \nameshort{} in \stl{}. 
            The hyper-parameters are: $\alpha_g=\protect\input{figure/sample_quality/stl10/301/0.0,sn-double-gamma_best_sample/BSN/g_lr.txt}$,
            $\alpha_d=\protect\input{figure/sample_quality/stl10/301/0.0,sn-double-gamma_best_sample/BSN/d_lr.txt}$,
            $n_{dis}=\protect\input{figure/sample_quality/stl10/301/0.0,sn-double-gamma_best_sample/BSN/n_dis.txt}$.
            Inception score is $\protect\input{figure/sample_quality/stl10/301/0.0,sn-double-gamma_best_sample/BSN/inception_score_mean.txt}$.
            FID is $\protect\input{figure/sample_quality/stl10/301/0.0,sn-double-gamma_best_sample/BSN/fid.txt}$.
        }
        \label{fig:stl-BSN-generated-samples}
    \end{figure}

\clearpage
\section{Experimental Details and Additional Results on \celeba{}}
\label{app:result_celeba}

\subsection{Experimental Details}
The network architectures are shown in \cref{tbl:gen-network-cifar10-stl10,tbl:dis-network-cifar10-stl10}.
All experiments are run for 400k iterations. Batch size is 64. The optimizer is Adam.
We use the five hyper-parameter settings listed in \cref{tbl:cifar-stl10-hyper}.
(In \cref{tbl:all} we only show the results from the first hyper-parameter setting.)
We use hinge loss with the popular SN implementation \cite{miyato2018spectral}. 

For \namessnshort{} and \namebssnshort{} in \cref{tbl:all}, we ran following scales: [0.7, 0.8, 0.9, 1.0, 1.2, 1.4, 1.6], and present the results from best one for each metric.

\subsection{FID Plot}
\cref{fig:celeba-bar-fid} shows the FID score in \celeba{} dataset. We can see that \nameshort{} outperforms the standard SN in all 5 hyper-parameter settings.

\begin{figure}
		\centering
		\includegraphics[width=0.45\linewidth]{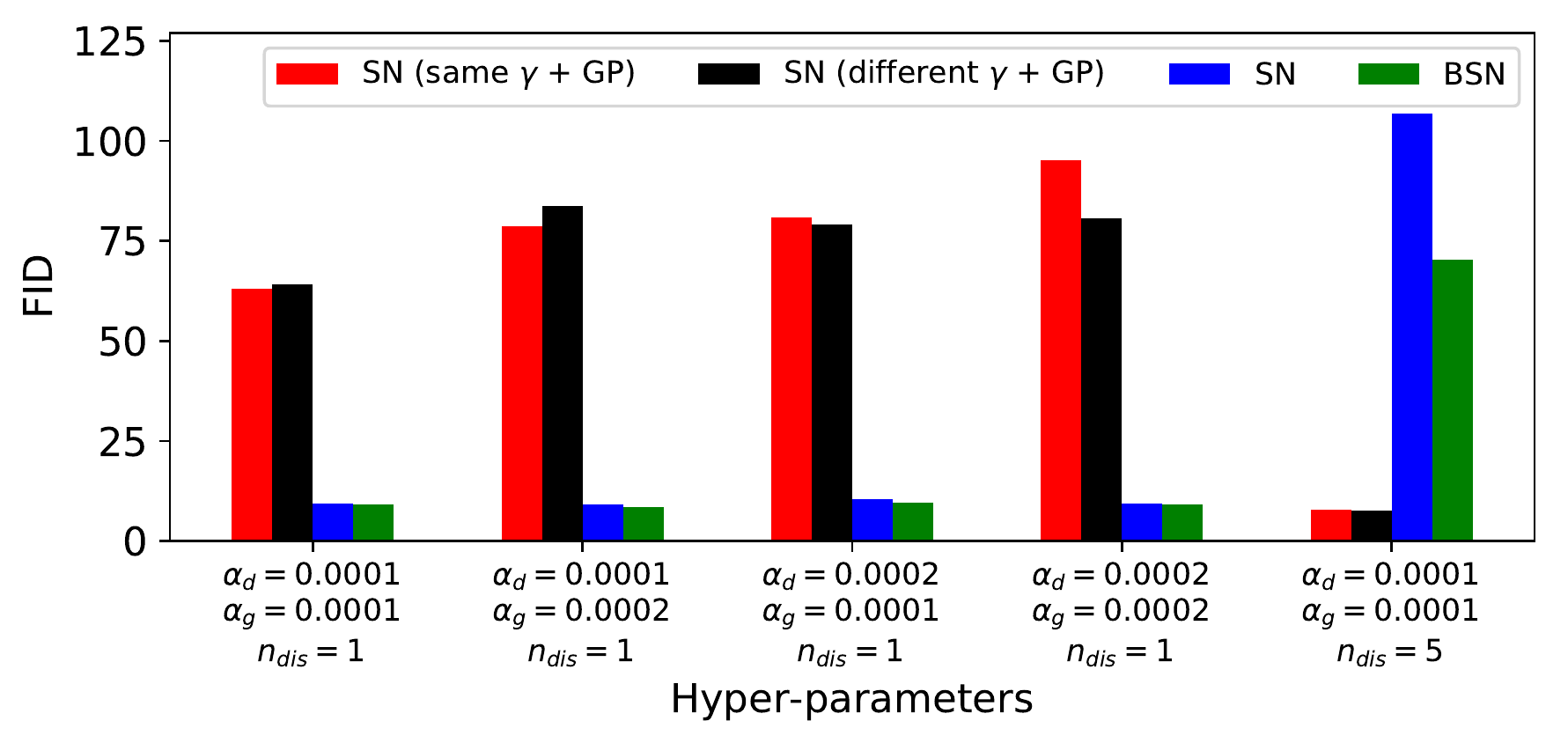}
		\caption{FID in \celeba{}. The results are averaged over 5 random seeds.}
		\label{fig:celeba-bar-fid}
\end{figure}

\subsection{Training Curves}
From \cref{sec:cifar_stl} we can see that SN (no $\gamma$) and \nameshort{} generally have the best performance. Therefore, in this section, we focus on comparing these two algorithms with the training curves. \cref{fig:celeba-g0.0001-d0.0001-ndis1-fid,fig:celeba-g0.0001-d0.0002-ndis1-fid,fig:celeba-g0.0002-d0.0001-ndis1-fid,fig:celeba-g0.0002-d0.0002-ndis1-fid,fig:celeba-g0.0001-d0.0001-ndis5-fid} show the FID of these two algorithms during training. Generally, we see that \nameshort{} converges slower than SN \emph{at the beginning of training}. However, as training proceeds, \nameshort{} finally has better metrics in all cases. 
Note that unlike \cifar{}, SN seems to be more stable in \celeba{} as its sample quality does not drop in most hyper-parameters.  
But the key conclusion is the same: in most cases, \nameshort{} not only outperforms SN at the end of training, but also outperforms the peak sample quality of SN during training (e.g. \cref{fig:celeba-g0.0001-d0.0001-ndis1-fid,fig:celeba-g0.0001-d0.0002-ndis1-fid,fig:celeba-g0.0002-d0.0001-ndis1-fid,fig:celeba-g0.0002-d0.0002-ndis1-fid}). 
The only exception is the $n_{dis}=5$ setting, where both SN and \nameshort{} has instability issue: the sample quality first improves and then significantly drops. But even in this case, \nameshort{} has better final performance than the standard SN.

\trainingcurvefid{\celeba{}}{celeba/314}{0.0001}{0.0001}{1}{celeba}
\trainingcurvefid{\celeba{}}{celeba/314}{0.0001}{0.0002}{1}{celeba}
\trainingcurvefid{\celeba{}}{celeba/314}{0.0002}{0.0001}{1}{celeba}
\trainingcurvefid{\celeba{}}{celeba/314}{0.0002}{0.0002}{1}{celeba}
\trainingcurvefid{\celeba{}}{celeba/314}{0.0001}{0.0001}{5}{celeba}

\subsection{Generated Images}
\cref{fig:celeba-SN_same_gamma-generated-samples,fig:celeba-SN_diff_gamma-generated-samples,fig:celeba-SN_no_gamma-generated-samples,fig:celeba-BSN-generated-samples} show the generated images from the run with the best FID for each algorithm.

    \begin{figure}
        \centering
        \includegraphics[width=0.6\linewidth]{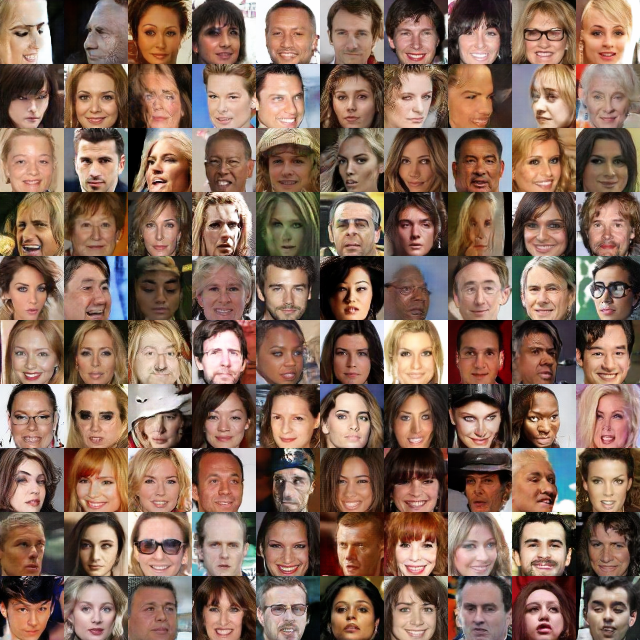}
        \caption{Generated samples from the best run of SN (same $\gamma$) in \celeba{}. 
            The hyper-parameters are: $\alpha_g=\protect\input{figure/sample_quality/celeba/314/0.0,sn-double-gamma_best_sample/SN_same_gamma/g_lr.txt}$,
            $\alpha_d=\protect\input{figure/sample_quality/celeba/314/0.0,sn-double-gamma_best_sample/SN_same_gamma/d_lr.txt}$,
            $n_{dis}=\protect\input{figure/sample_quality/celeba/314/0.0,sn-double-gamma_best_sample/SN_same_gamma/n_dis.txt}$.
            FID is $\protect\input{figure/sample_quality/celeba/314/0.0,sn-double-gamma_best_sample/SN_same_gamma/fid.txt}$.
        }
        \label{fig:celeba-SN_same_gamma-generated-samples}
    \end{figure}

    \begin{figure}
        \centering
        \includegraphics[width=0.6\linewidth]{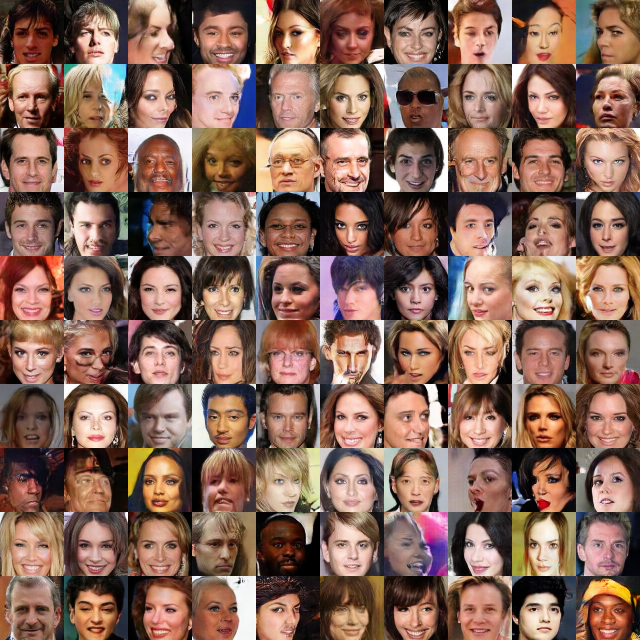}
        \caption{Generated samples from the best run of SN (diff. $\gamma$) in \celeba{}. 
            The hyper-parameters are: $\alpha_g=\protect\input{figure/sample_quality/celeba/314/0.0,sn-double-gamma_best_sample/SN_diff_gamma/g_lr.txt}$,
            $\alpha_d=\protect\input{figure/sample_quality/celeba/314/0.0,sn-double-gamma_best_sample/SN_diff_gamma/d_lr.txt}$,
            $n_{dis}=\protect\input{figure/sample_quality/celeba/314/0.0,sn-double-gamma_best_sample/SN_diff_gamma/n_dis.txt}$.
            FID is $\protect\input{figure/sample_quality/celeba/314/0.0,sn-double-gamma_best_sample/SN_diff_gamma/fid.txt}$.
        }
        \label{fig:celeba-SN_diff_gamma-generated-samples}
    \end{figure}

    \begin{figure}
        \centering
        \includegraphics[width=0.6\linewidth]{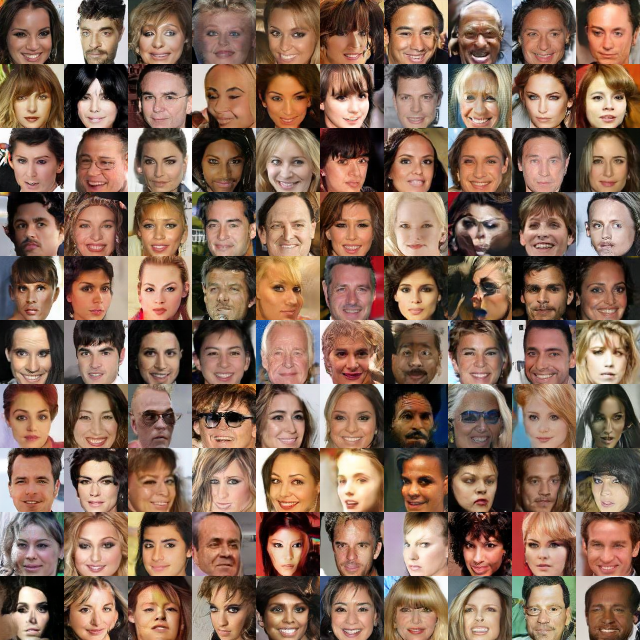}
        \caption{Generated samples from the best run of SN in \celeba{}. 
            The hyper-parameters are: $\alpha_g=\protect\input{figure/sample_quality/celeba/314/0.0,sn-double-gamma_best_sample/SN_no_gamma/g_lr.txt}$,
            $\alpha_d=\protect\input{figure/sample_quality/celeba/314/0.0,sn-double-gamma_best_sample/SN_no_gamma/d_lr.txt}$,
            $n_{dis}=\protect\input{figure/sample_quality/celeba/314/0.0,sn-double-gamma_best_sample/SN_no_gamma/n_dis.txt}$.
            FID is $\protect\input{figure/sample_quality/celeba/314/0.0,sn-double-gamma_best_sample/SN_no_gamma/fid.txt}$.
        }
        \label{fig:celeba-SN_no_gamma-generated-samples}
    \end{figure}

    \begin{figure}
        \centering
        \includegraphics[width=0.6\linewidth]{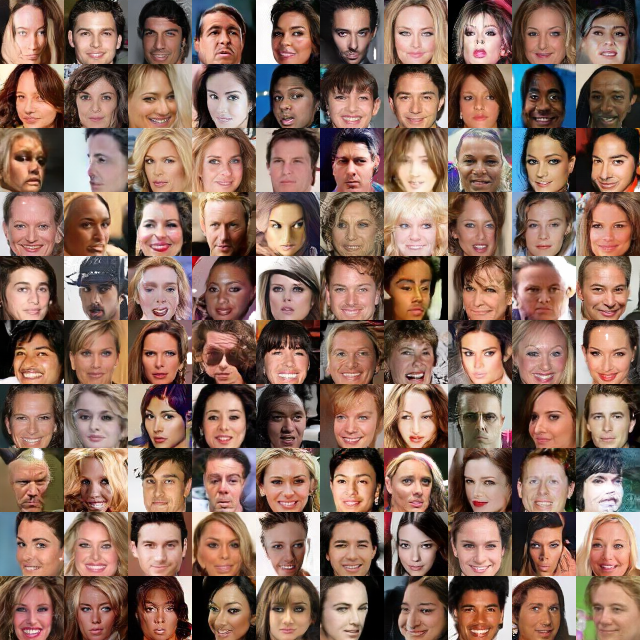}
        \caption{Generated samples from the best run of \nameshort{} in \celeba{}. 
            The hyper-parameters are: $\alpha_g=\protect\input{figure/sample_quality/celeba/314/0.0,sn-double-gamma_best_sample/BSN/g_lr.txt}$,
            $\alpha_d=\protect\input{figure/sample_quality/celeba/314/0.0,sn-double-gamma_best_sample/BSN/d_lr.txt}$,
            $n_{dis}=\protect\input{figure/sample_quality/celeba/314/0.0,sn-double-gamma_best_sample/BSN/n_dis.txt}$.
            FID is $\protect\input{figure/sample_quality/celeba/314/0.0,sn-double-gamma_best_sample/BSN/fid.txt}$.
        }
        \label{fig:celeba-BSN-generated-samples}
    \end{figure}

\clearpage

\section{Experimental Details and Additional Results on \ilsvrc{}}
\label{app:result_ilsvrc}

\subsection{Experimental Details}
The network architectures are shown in \cref{tbl:gen-network-imagenet,tbl:gen-resnet-network-imagenet,tbl:dis-network-imagenet,tbl:dis-resnet-down-network-imagenet,tbl:dis-resnet-first-network-imagenet,tbl:dis-resnet-network-imagenet}.
All experiments are run for 500k iterations. Discriminator batch size is 16. Generator batch size is 32. The optimizer is Adam. $\alpha_g=0.002, \alpha_d=0.002, \beta_1=0.0, \beta_2=0.9, n_{dis}=5$
We use hinge loss with the popular SN implementation \cite{miyato2018spectral}. 

\begin{table}
	\centering
	\begin{tabular}{c}
		\toprule
		$z\in \Rb^{128}\sim \Nc(0, I)$\\\hline
		Fully connected ($4 \times 4\times 1024$).\\\hline
		ResNet-up ($c=1024$).\\\hline
		ResNet-up ($c=512$).\\\hline
		ResNet-up ($c=256$).\\\hline
		ResNet-up ($c=128$).\\\hline
		ResNet-up ($c=64$).\\\hline
		BN. ReLU. Convolution ($c=3, k=3, s=1$). Tanh\\
		\bottomrule
	\end{tabular}
	\caption{Generator network architectures for \ilsvrc{} experiments (from \cite{miyato2018spectral}). BN stands for batch normalization. $c$ stands for number of channels. $k$ stands for kernel size. $s$ stands for stride.}
	\label{tbl:gen-network-imagenet}
\end{table}

\begin{table}
	\centering
	\begin{tabular}{c}
		\toprule
		\textbf{Direct connection}\\\hline
		BN. ReLU. Unpooling(2). Convolution ($k=3, s=1$). \\\hline
		BN. ReLU. Convolution ($k=3, s=1$). \\\hline\hline
		\textbf{Shortcut connection}\\\hline
		Unpooling(2). Convolution ($k=1, s=1$). \\
		\bottomrule
	\end{tabular}
	\caption{ResNet-up network architectures for \ilsvrc{} experiments (from \cite{miyato2018spectral}). BN stands for batch normalization. $k$ stands for kernel size. $s$ stands for stride.}
	\label{tbl:gen-resnet-network-imagenet}
\end{table}

\begin{table}
	\centering
	\begin{tabular}{c}
		\toprule
		$x\in \Rb^{128\times 128\times 3}$\\\hline
		ResNet-first ($c=64$).\\\hline
		ResNet-down ($c=128$).\\\hline
		ResNet-down ($c=256$).\\\hline
		ResNet-down ($c=512$).\\\hline
		ResNet-down ($c=1024$).\\\hline
		ResNet ($c=1024$).\\\hline
		ReLU. Global pooling. Fully connected (1).\\
		\bottomrule
	\end{tabular}
	\caption{Discriminator network architectures for \ilsvrc{} experiments (from \cite{miyato2018spectral}). BN stands for batch normalization. $c$ stands for number of channels. $k$ stands for kernel size. $s$ stands for stride.}
	\label{tbl:dis-network-imagenet}
\end{table}

\begin{table}
	\centering
	\begin{tabular}{c}
		\toprule
		\textbf{Direct connection}\\\hline
		ReLU. Convolution ($k=3, s=1$). \\\hline
		ReLU. Convolution ($k=3, s=1$). Average pooling(2).\\\hline\hline
		\textbf{Shortcut connection}\\\hline
		Convolution ($k=1, s=1$). Average pooling(2).\\
		\bottomrule
	\end{tabular}
	\caption{ResNet-down network architectures for \ilsvrc{} experiments (from \cite{miyato2018spectral}). $k$ stands for kernel size. $s$ stands for stride.}
	\label{tbl:dis-resnet-down-network-imagenet}
\end{table}

\begin{table}
	\centering
	\begin{tabular}{c}
		\toprule
		\textbf{Direct connection}\\\hline
		Convolution ($k=3, s=1$). \\\hline
		ReLU. Convolution ($k=3, s=1$). Average pooling(2).\\\hline\hline
		\textbf{Shortcut connection}\\\hline
		Average pooling(2). Convolution ($k=1, s=1$).\\
		\bottomrule
	\end{tabular}
	\caption{ResNet-first network architectures for \ilsvrc{} experiments (from \cite{miyato2018spectral}). $k$ stands for kernel size. $s$ stands for stride.}
	\label{tbl:dis-resnet-first-network-imagenet}
\end{table}

\begin{table}
	\centering
	\begin{tabular}{c}
		\toprule
		\textbf{Direct connection}\\\hline
		ReLU. Convolution ($k=3, s=1$). \\\hline
		ReLU. Convolution ($k=3, s=1$). \\\hline\hline
		\textbf{Shortcut connection}\\\hline
		Convolution ($k=1, s=1$). \\
		\bottomrule
	\end{tabular}
	\caption{ResNet network architectures for \ilsvrc{} experiments (from \cite{miyato2018spectral}). $k$ stands for kernel size. $s$ stands for stride.}
	\label{tbl:dis-resnet-network-imagenet}
\end{table}

\subsection{Training Curves}
\cref{fig:imagenet-curve-inception,fig:imagenet-curve-fid} show the inception score and FID of SN and \nameshort{} during training. 

For SN, we can see that the runs with scale=1.0/1.2/1.4 have similar performance throughout training. When scale=1.6, the performance is much worse.

For \nameshort{}, the runs with scale=1.2/1.4 perform better than SN runs throughout the training. When scale=1.6, \nameshort{} has similar performance as SN at the early stage of training, and is slightly better at the end. When scale=1.0, the performance is very bad as there is gradient vanishing problem.
\begin{figure}
	\centering
	\includegraphics[width=0.45\linewidth]{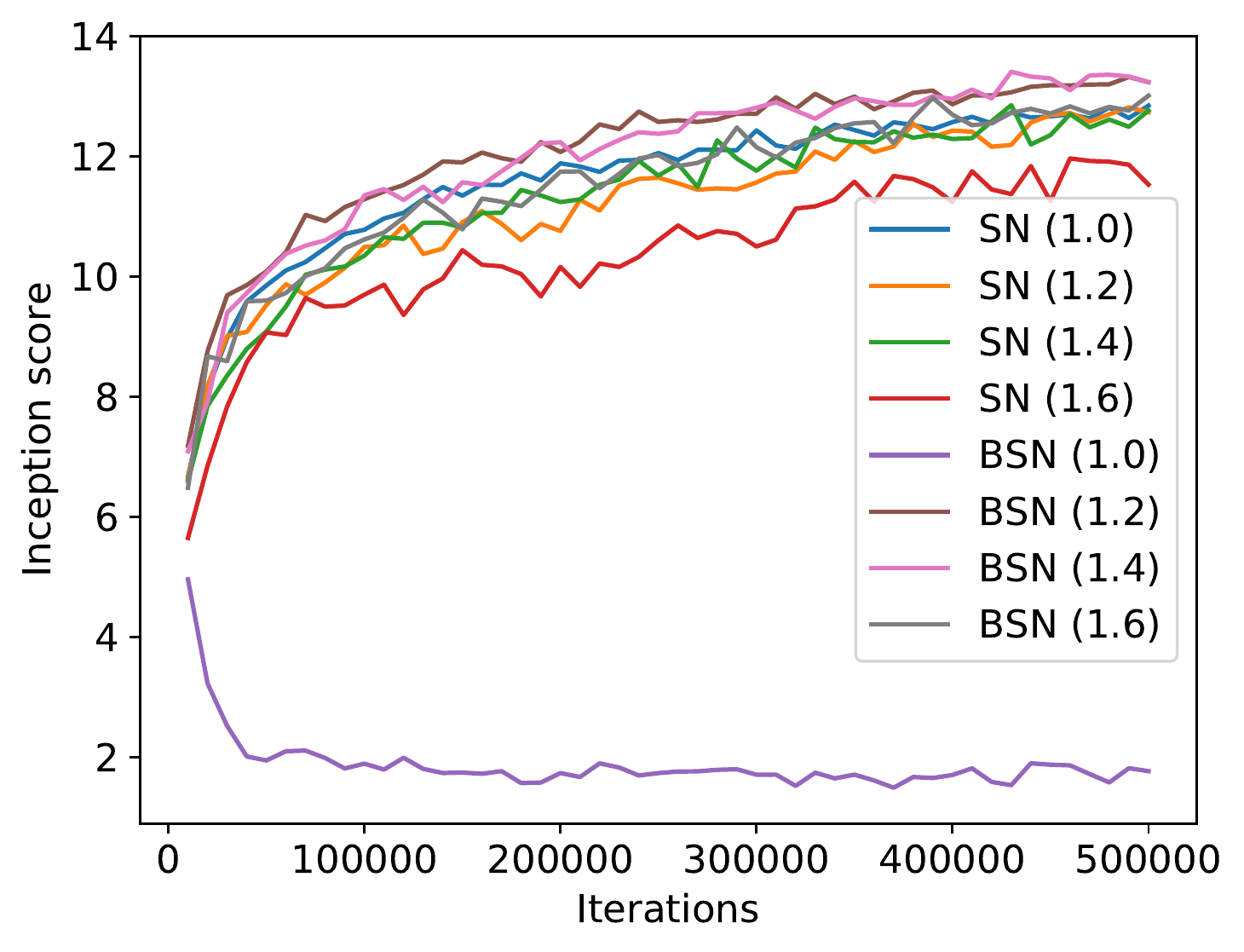}
	\caption{Inception score in \ilsvrc{}. The results are averaged over 5 random seeds.}
	\label{fig:imagenet-curve-inception}
\end{figure}
\begin{figure}
	\centering
	\includegraphics[width=0.45\linewidth]{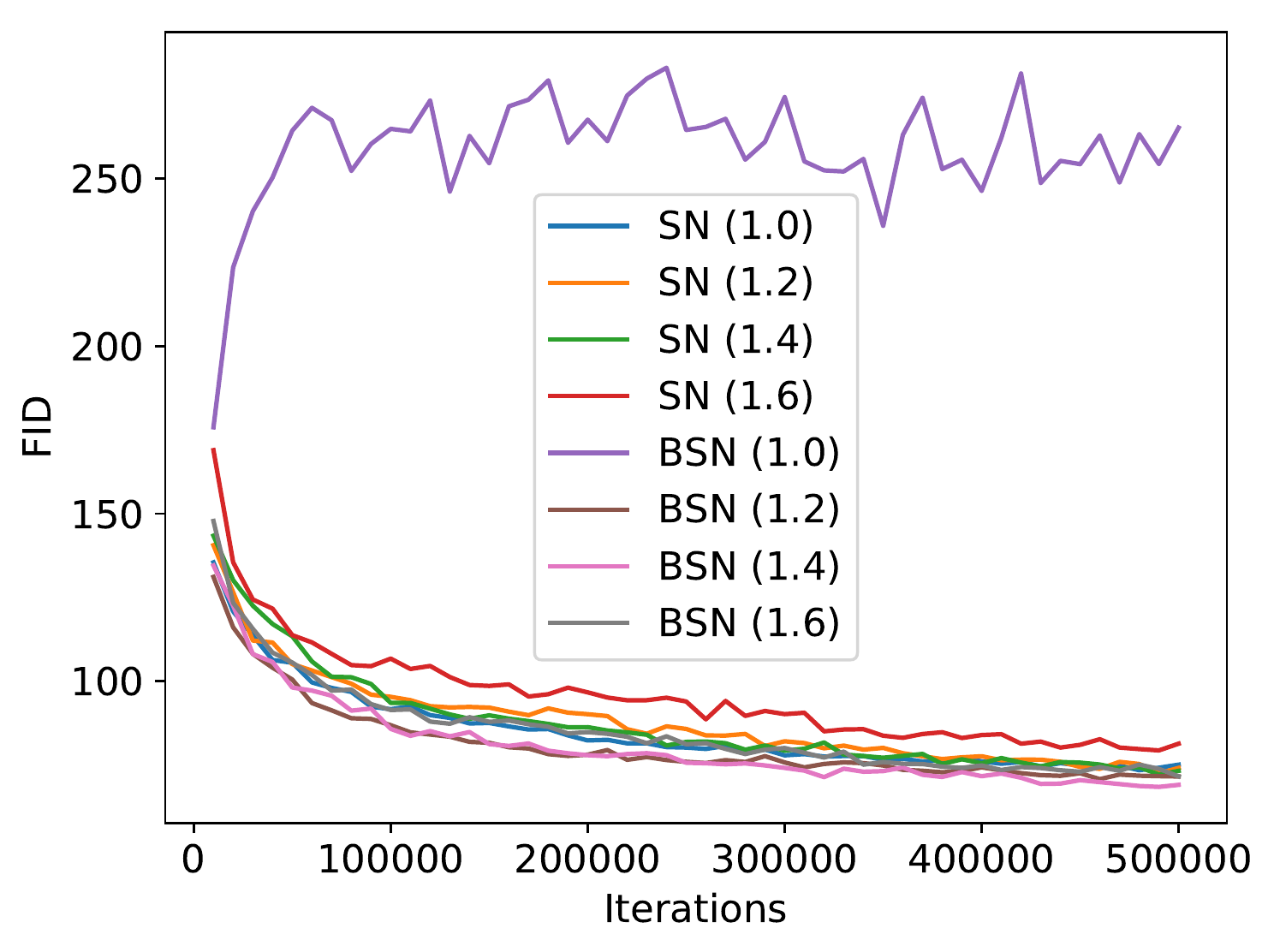}
	\caption{FID in \ilsvrc{}. The results are averaged over 5 random seeds.}
	\label{fig:imagenet-curve-fid}
\end{figure}

\subsection{Generated Images}

\cref{fig:imagenet-SN_1.0-generated-samples,fig:imagenet-SN_1.2-generated-samples,fig:imagenet-SN_1.4-generated-samples,fig:imagenet-SN_1.6-generated-samples,fig:imagenet-BSN_1.0-generated-samples,fig:imagenet-BSN_1.2-generated-samples,fig:imagenet-BSN_1.4-generated-samples,fig:imagenet-BSN_1.6-generated-samples} show the generated images from the run with the best inception score for SN and \nameshort{} with different scale parameters.

    \begin{figure}
        \centering
        \includegraphics[width=0.6\linewidth]{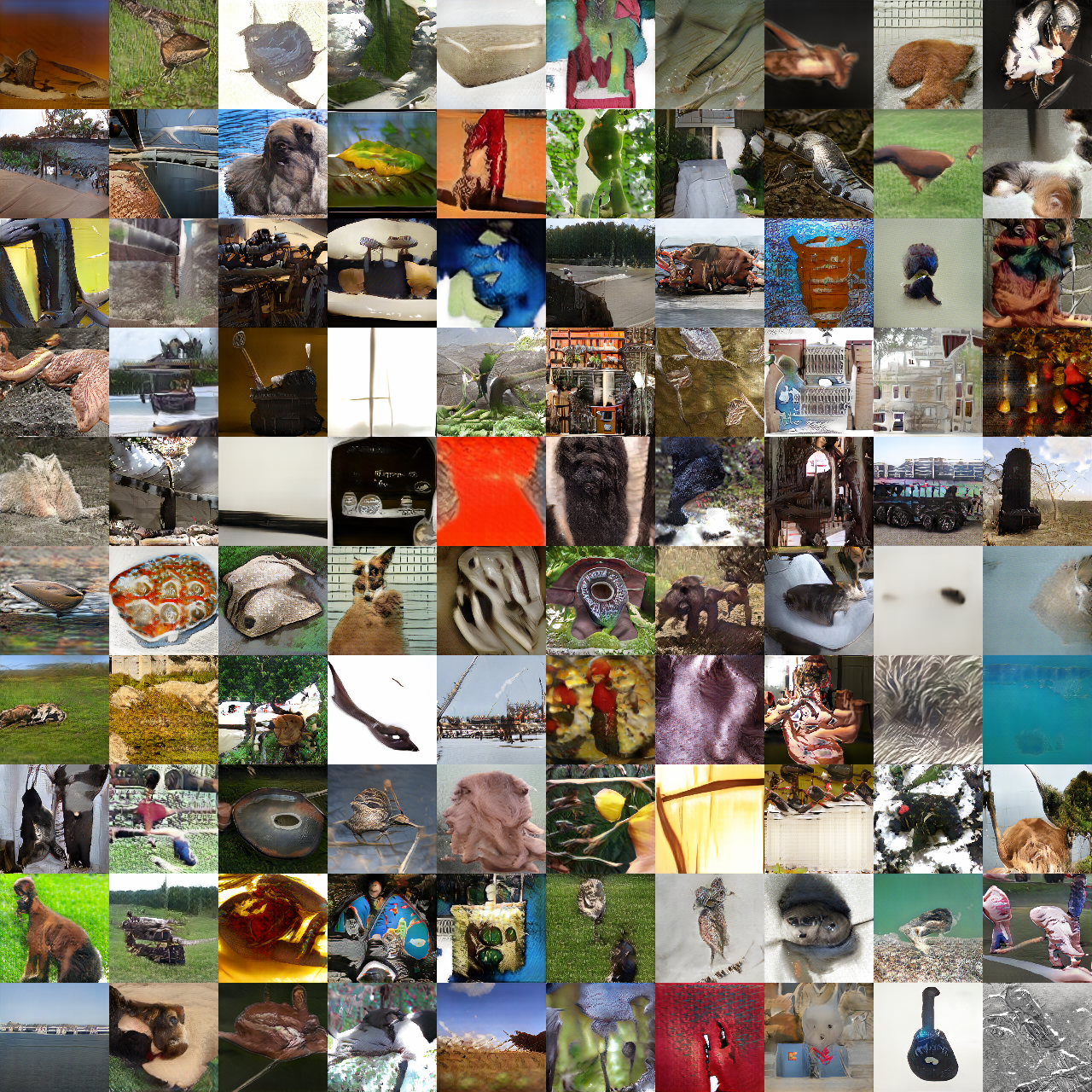}
        \caption{Generated samples from the best run of SN (scale=1.0) in \ilsvrc{}. 
            Inception score is $\protect\input{figure/sample_quality/imagenet/303/0.0,sn-double_best_sample/SN_1.0/inception_score_mean.txt}$.
            FID is $\protect\input{figure/sample_quality/imagenet/303/0.0,sn-double_best_sample/SN_1.0/fid.txt}$.
        }
        \label{fig:imagenet-SN_1.0-generated-samples}
    \end{figure}

    \begin{figure}
        \centering
        \includegraphics[width=0.6\linewidth]{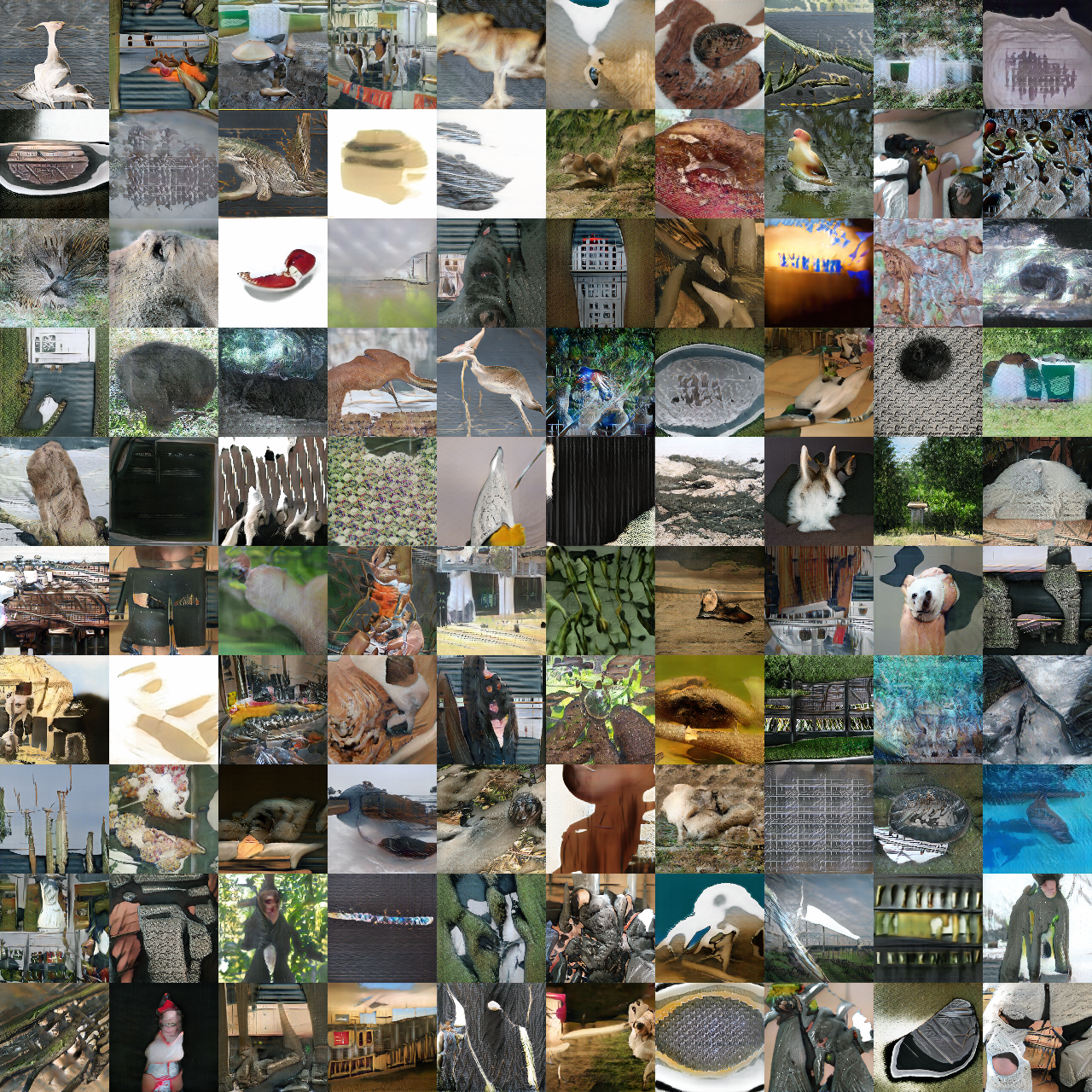}
        \caption{Generated samples from the best run of SN (scale=1.2) in \ilsvrc{}. 
            Inception score is $\protect\input{figure/sample_quality/imagenet/303/0.0,sn-double_best_sample/SN_1.2/inception_score_mean.txt}$.
            FID is $\protect\input{figure/sample_quality/imagenet/303/0.0,sn-double_best_sample/SN_1.2/fid.txt}$.
        }
        \label{fig:imagenet-SN_1.2-generated-samples}
    \end{figure}

    \begin{figure}
        \centering
        \includegraphics[width=0.6\linewidth]{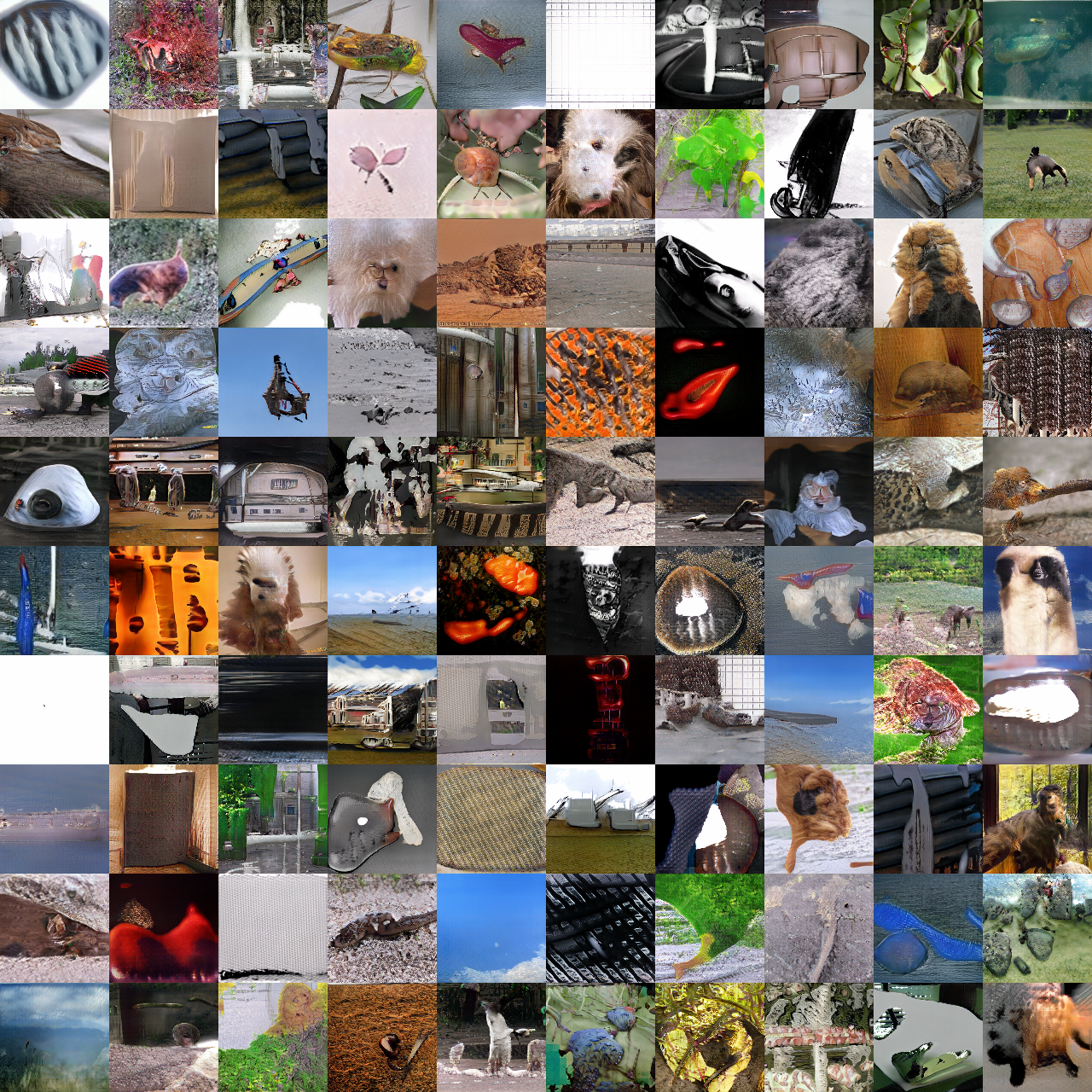}
        \caption{Generated samples from the best run of SN (scale=1.4) in \ilsvrc{}. 
            Inception score is $\protect\input{figure/sample_quality/imagenet/303/0.0,sn-double_best_sample/SN_1.4/inception_score_mean.txt}$.
            FID is $\protect\input{figure/sample_quality/imagenet/303/0.0,sn-double_best_sample/SN_1.4/fid.txt}$.
        }
        \label{fig:imagenet-SN_1.4-generated-samples}
    \end{figure}

    \begin{figure}
        \centering
        \includegraphics[width=0.6\linewidth]{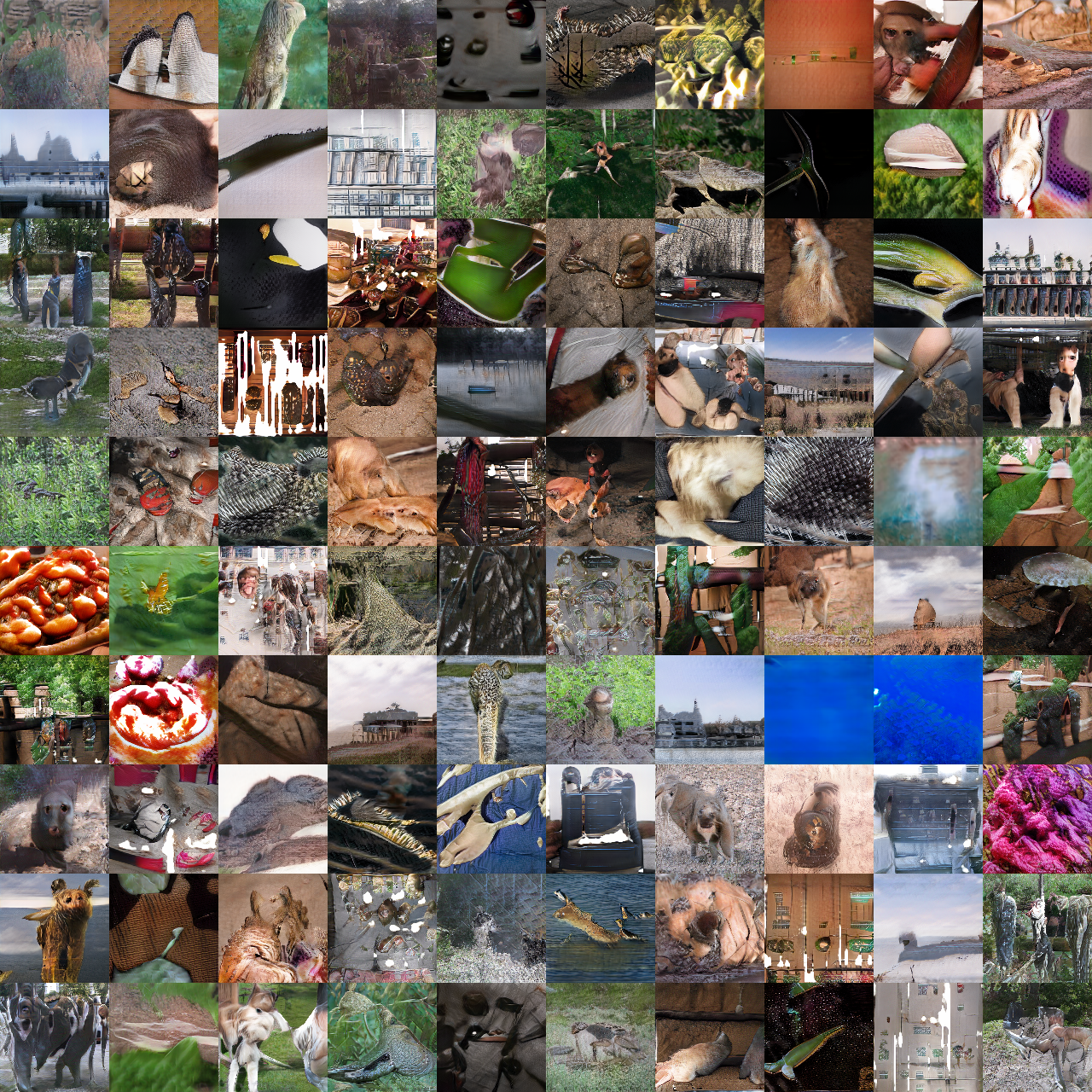}
        \caption{Generated samples from the best run of SN (scale=1.6) in \ilsvrc{}. 
            Inception score is $\protect\input{figure/sample_quality/imagenet/303/0.0,sn-double_best_sample/SN_1.6/inception_score_mean.txt}$.
            FID is $\protect\input{figure/sample_quality/imagenet/303/0.0,sn-double_best_sample/SN_1.6/fid.txt}$.
        }
        \label{fig:imagenet-SN_1.6-generated-samples}
    \end{figure}

    \begin{figure}
        \centering
        \includegraphics[width=0.6\linewidth]{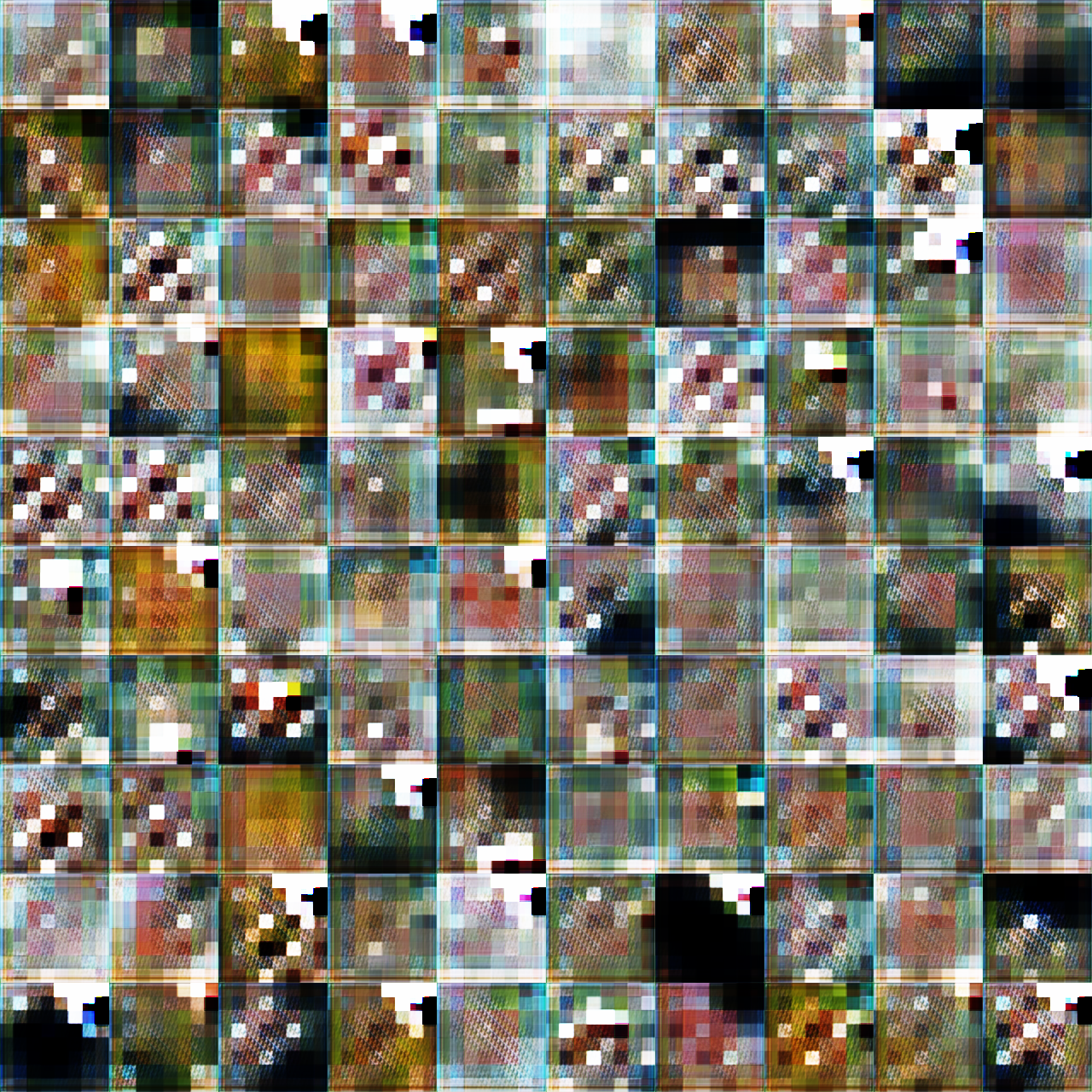}
        \caption{Generated samples from the best run of \nameshort{} (scale=1.0) in \ilsvrc{}. 
            Inception score is $\protect\input{figure/sample_quality/imagenet/303/0.0,sn-double_best_sample/BSN_1.0/inception_score_mean.txt}$.
            FID is $\protect\input{figure/sample_quality/imagenet/303/0.0,sn-double_best_sample/BSN_1.0/fid.txt}$.
        }
        \label{fig:imagenet-BSN_1.0-generated-samples}
    \end{figure}

    \begin{figure}
        \centering
        \includegraphics[width=0.6\linewidth]{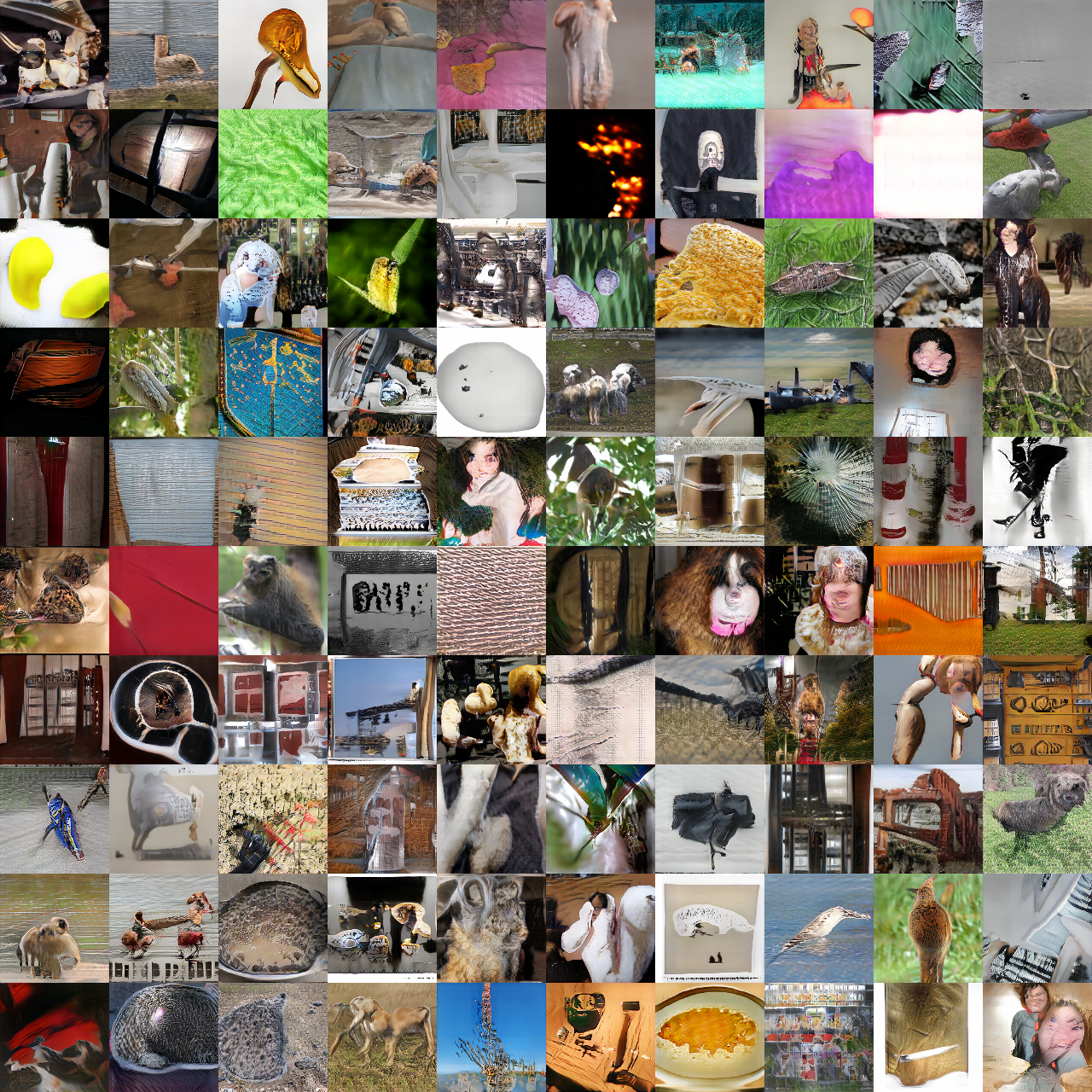}
        \caption{Generated samples from the best run of \nameshort{} (scale=1.2) in \ilsvrc{}. 
            Inception score is $\protect\input{figure/sample_quality/imagenet/303/0.0,sn-double_best_sample/BSN_1.2/inception_score_mean.txt}$.
            FID is $\protect\input{figure/sample_quality/imagenet/303/0.0,sn-double_best_sample/BSN_1.2/fid.txt}$.
        }
        \label{fig:imagenet-BSN_1.2-generated-samples}
    \end{figure}

    \begin{figure}
        \centering
        \includegraphics[width=0.6\linewidth]{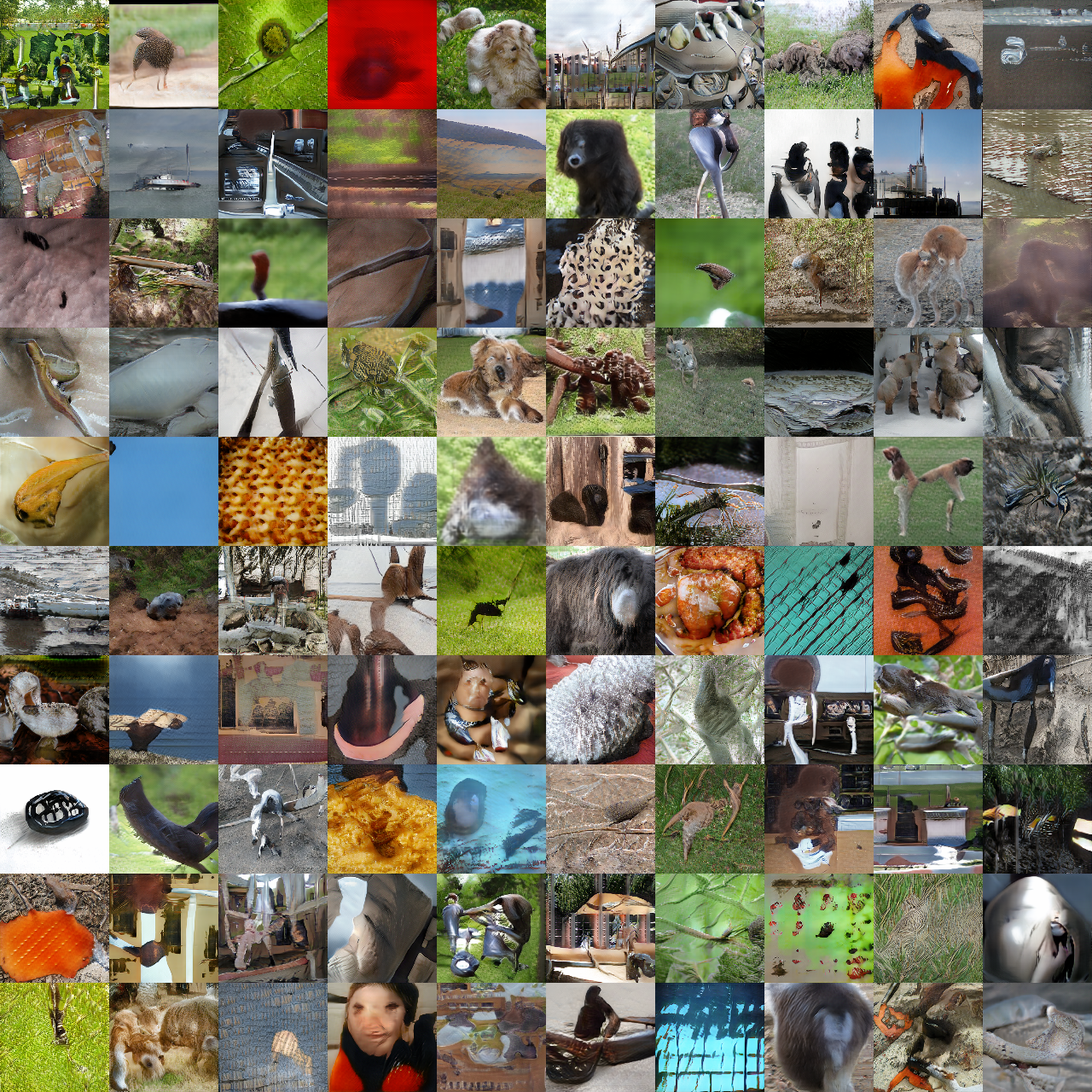}
        \caption{Generated samples from the best run of \nameshort{} (scale=1.4) in \ilsvrc{}. 
            Inception score is $\protect\input{figure/sample_quality/imagenet/303/0.0,sn-double_best_sample/BSN_1.4/inception_score_mean.txt}$.
            FID is $\protect\input{figure/sample_quality/imagenet/303/0.0,sn-double_best_sample/BSN_1.4/fid.txt}$.
        }
        \label{fig:imagenet-BSN_1.4-generated-samples}
    \end{figure}

    \begin{figure}
        \centering
        \includegraphics[width=0.6\linewidth]{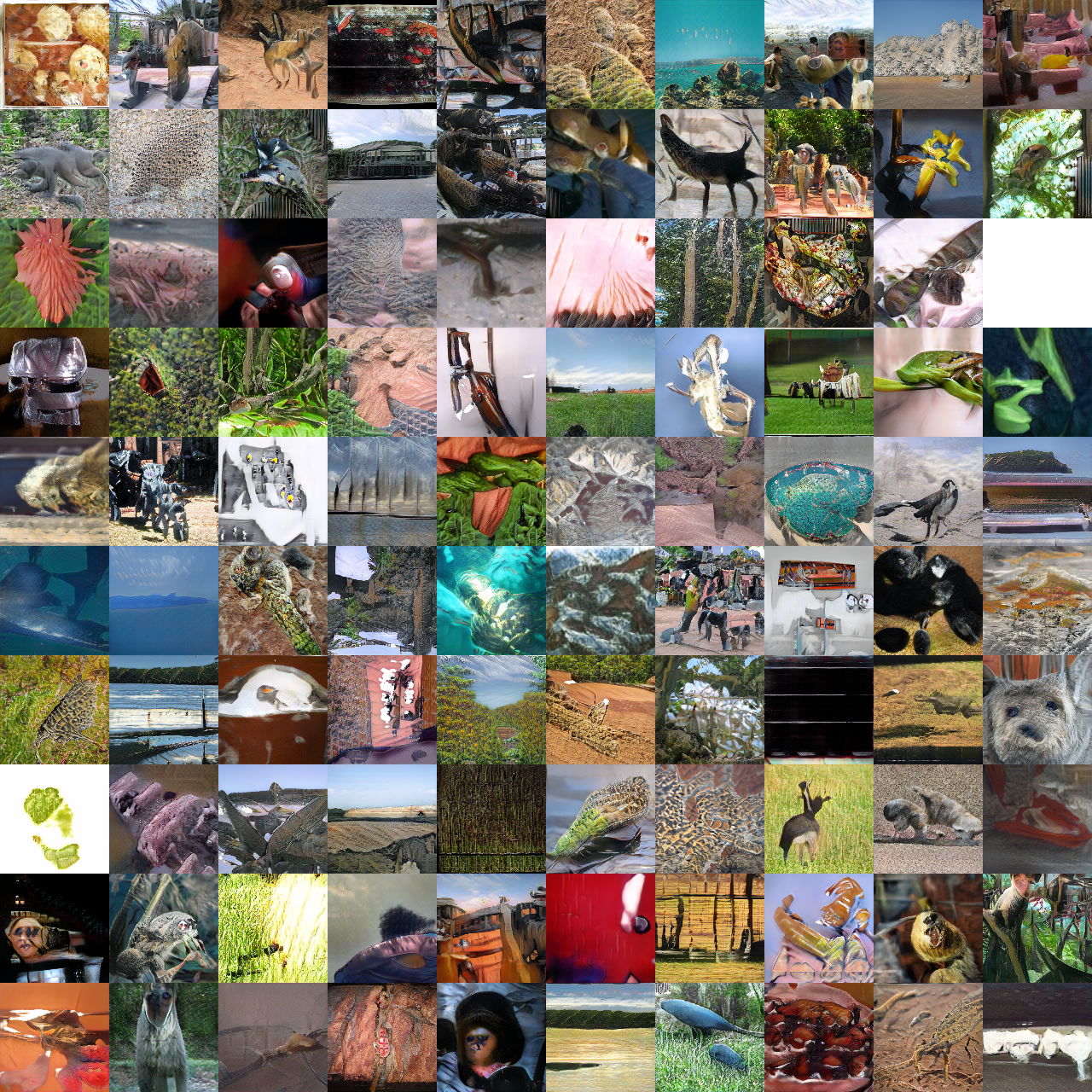}
        \caption{Generated samples from the best run of \nameshort{} (scale=1.6) in \ilsvrc{}. 
            Inception score is $\protect\input{figure/sample_quality/imagenet/303/0.0,sn-double_best_sample/BSN_1.6/inception_score_mean.txt}$.
            FID is $\protect\input{figure/sample_quality/imagenet/303/0.0,sn-double_best_sample/BSN_1.6/fid.txt}$.
        }
        \label{fig:imagenet-BSN_1.6-generated-samples}
    \end{figure}

\clearpage

\end{document}